\documentclass[11pt]{article}
\usepackage{enumitem}
\usepackage[dvipsnames]{xcolor}

\usepackage{amsmath}
\usepackage{amssymb}
\usepackage{mathtools}
\usepackage{amsthm}
\usepackage{color,soul}
\usepackage{aux/macros}

\usepackage{fullpage,times,bm,float,array}
\usepackage{nicefrac}
\usepackage{thmtools}
\usepackage{thm-restate}

\usepackage{microtype}
\usepackage{graphicx}
\usepackage{subfigure}
\usepackage{booktabs} 

\usepackage[dvipsnames]{xcolor}

\newtheoremstyle{italicdefinition} {}                               {}                               {\itshape}                       {}                               {\bfseries}                      {.}                              {.5em}                           {}                               

\theoremstyle{plain}
\newtheorem{theorem}{Theorem}[section]

\newtheorem{lemma}[theorem]{Lemma}
\newtheorem{corollary}[theorem]{Corollary}
\theoremstyle{italicdefinition}
\newtheorem{definition}{Definition}
\newtheorem{assumption}[theorem]{Assumption}
\theoremstyle{remark}

\usepackage[textsize=tiny]{todonotes}

\usepackage[style = alphabetic,citestyle=alphabetic,maxbibnames=99,sorting=nyt,natbib=true,backref=true]{biblatex}
\DefineBibliographyStrings{english}{backrefpage = {Cited on page},backrefpages = {Cited on pages},}
\usepackage[colorlinks,linkcolor = RawSienna, urlcolor  = Green, citecolor = Green, anchorcolor = ForestGreen,bookmarks=False]{hyperref}

\renewcommand{\r}{\right}
\renewcommand{\l}{\left} 

\author{
Neil Mallinar$^*$\\
UC San Diego\\
\texttt{nmallina@ucsd.edu}\\
\and Austin Zane$^*$\\
UC Berkeley\\
\texttt{austin.zane@berkeley.edu}\\
\and Spencer Frei \\
UC Davis\\
\texttt{sfrei@ucdavis.edu}\\
\and Bin Yu \\
UC Berkeley\\
\texttt{binyu@berkeley.edu}
}

\title{\textbf{Minimum-Norm Interpolation Under Covariate Shift}}

\titlespacing*{\section}      {0em}{.75em}{.5em}
\titlespacing*{\subsection}   {0em}{.75em}{.5em}
\titlespacing*{\subsubsection}{0em}{.75em}{.5em}
\titlespacing*{\paragraph}    {0em}{.75em}{.5em}

\date{April 2, 2024}
\begin{document}

\maketitle
\def\thefootnote{*}\footnotetext{Equal contribution.}

\begin{abstract}
Transfer learning is a critical part of real-world machine learning deployments and has been extensively studied in experimental works with overparameterized neural networks.
However, even in the simplest setting of linear regression a notable gap still exists in  the theoretical understanding of transfer learning.
In-distribution research on high-dimensional linear regression has led to the identification of a phenomenon known as \textit{benign overfitting}, in which linear interpolators overfit to noisy training labels and yet still generalize well.
This behavior occurs under specific conditions on the source covariance matrix and input data dimension.
Therefore, it is natural to wonder how such high-dimensional linear models behave under transfer learning.
We prove the first non-asymptotic excess risk bounds for benignly-overfit linear interpolators in the transfer learning setting.
From our analysis, we propose a taxonomy of \textit{beneficial} and \textit{malignant} covariate shifts based on the degree of overparameterization.
We follow our analysis with empirical studies that show these beneficial and malignant covariate shifts for linear interpolators on real image data, and for fully-connected neural networks in settings where the input data dimension is larger than the training sample size. \end{abstract}

\section{Introduction}

Practical deployments of machine learning models are almost always in a transfer learning setting, where models trained on a \textit{source data distribution} with noisy labels are expected to perform well on a different \textit{target data distribution}, referred to as the ``out-of-distribution" (OOD) dataset \citep{Oglic_2022,damour2022underspecification}.
There have been many experimental works on transfer learning with complex models and datasets \citep{recht2019imagenet, Koh2021WILDS,pmlr-v139-miller21b,Hendrycks_2021_ICCV,Wenzel2022AssayingOG,Liang2023AccuracyOT}, but remarkably fewer attempts to study it theoretically, even in the simplest case of linear models which have been of great interest in recent years \citep{Dwivedi2020RevisitingCA,bartlett2020benignpnas,Hastie2022Surprises,tsigler2023benignjmlr}.

There has been an extensive ``in-distribution" (ID) theoretical interest in high-dimensional linear regression and specifically \textit{interpolation}, meaning a model reaches zero training loss \citep{belkin2019reconciling,belkin2019interpolationoptimality}.
Frameworks such as ``benign overfitting", or ``harmless interpolation" \citep{bartlett2020benignpnas,muthukumar2020harmless} emerged as an attempt to explain why interpolating neural networks 
often
do not overfit catastrophically \citep{zhang2017rethinkinggeneralization}.
They found that, in specific cases, overfitting can be ``benign", meaning that a model interpolates noisy training labels and yet has vanishing excess risk.
In linear regression, this occurs if and only if the training (source) covariance matrix satisfies very specific conditions.
Under these conditions, the minimum-norm interpolator (MNI) approximately acts like a ridge regression solution.

This sparked an initial wave of in-distribution theoretical research into benign overfitting in high-dimensional linear models \citep{Chatterji2022ImplicitBias,tsigler2023benignjmlr,  chatterji2023deep}, kernel regression \citep{Rakhlin2018ConsistencyOI, haas2023mind, Barzilai2023GeneralizationIK,belkin2018overfittingperfectfitting}, and even some shallow neural networks \citep{pmlr-v178-frei22a, pmlr-v202-kou23a, kornowski2023tempered, Xu2023ReLU}.  
Although these works were motivated by a desire to understand overfitting in modern deep learning, recent works have shown that in many practical settings of interest, overfitting is not benign~\citep{Mallinar2022, haas2023mind,lai2023generalization}.  
Thus, deeper investigations into the generalization behavior of overfit models are warranted. 

Given the increasing prevalence of overparameterized models, it is natural to ask how such models perform in the transfer learning setting.
There have been some efforts to answer this in the theoretically tractable cases of linear regression and random feature and kernel regression \citep{pmlr-v162-pathak22a,wang2023pseudolabeling}.
However, these works either provide asymptotic bounds that require the training sample size and data dimension to go to infinity at the same rate \citep{tripuraneni2021overparam}, study minimax settings which only considers worst-case risk \citep{lei2021linearregressiondistshift}, or focus on augmented gradient-based training algorithms, like importance weighting \citep{wang2021importance}.

\paragraph{Summary of contributions.}
In this paper, we investigate the generalization behavior of the minimum $\ell_2$-norm linear interpolator (MNI) under distribution shifts when the source distribution satisfies the conditions necessary for benign overfitting.  We summarize our main contributions as follows.
\begin{itemize}
    \item We provide the first non-asymptotic, instance-wise risk bounds for covariate shifts in interpolating linear regression when the source covariance matrix satisfies benign overfitting conditions and commutes with the target covariance matrix.
\item We use our risk bounds to propose a taxonomy of covariate shifts for the MNI. 
    We show how the ratio of target eigenvalues to source eigenvalues and the degree of overparameterization affect whether a shift is \textit{beneficial} or \textit{malignant}, meaning OOD risk is better or worse than ID risk, respectively. 
    The degree of overparameterization is determined by the eigenspectrum's head and tail properties.

    \item We empirically show that our taxonomy of shifts holds: (1) for the MNI on real image data under natural shifts like blur (a beneficial shift) and noise (a malignant shift), underscoring the significance of our findings beyond the idealized source and target covariances for which our theory is applicable; (2) for neural networks in settings where the input data dimension is larger than the training sample size, showing that our findings for the MNI are also reflective of the behavior of more complex models.

\end{itemize} \subsection{Prior Work and Comparisons to this Work}

\textbf{Excess risk analysis under distribution shifts: }
\citet{tripuraneni2021overparam} give an asymptotic analysis of high-dimensional random feature regression in covariate shift.
They require the number of samples, $n$, data dimension, $p$, and random feature dimension to go to $\infty$ at the same rate. In contrast, our non-asymptotic analysis  
considers finite sample cases and differing rates.
This allows us to draw new conclusions about how the \textit{degree of overparameterization} changes the way in which interpolating linear models exhibit out-of-distribution (OOD) generalization.
Additionally, our bounds let us analyze any sequence of eigenvalues for the target feature covariance matrix, which is not possible within their framework.

\citet{lei2021linearregressiondistshift} study linear regression under distribution shifts in the minimax setting.
Their minimax bounds consider the worst-case risk over an $\ell_2$-ball of target models, whereas we compute risk bounds specific to any model instantiation, with no restriction on the target model class.
Furthermore, their experimental results only consider the underparameterized regime.

Several other works study OOD generalization in more distant settings.
\citet{wang2021importance} study linear interpolators for classification, when trained with gradient descent and importance weighting, whereas we consider the closed-form MNI for linear regression.
\citet{pmlr-v202-simchowitz23a} study covariate shifts when the target function class is the sum of two other function classes, and shifts are defined with regard to metric entropy between classes, whereas we focus on well-specified linear models.
\citet{pmlr-v162-pathak22a,MaPathakWainwright2023,feng2023unified} consider covariate shift in kernel regression based on likelihood (``importance") ratios between source and target distributions while we consider source and target eigenvalue ratios which offer granular insights into feature scale changes whereas likelihood ratios capture shifts that affect the global data distribution.
\citet{pmlr-v162-pathak22a,MaPathakWainwright2023} also analyze worst-case, minimax risk for nonparametric function classes.
\citet{kausik2023generalization} work in the proportional asymptotic regime and consider the error in variables setting with noisy features and clean labels, while our work focuses on the linear regression setting with clean features and noisy labels.
Finally, we note that risk bounds in these prior works do not sufficiently account for the behavior of the high-rank covariance tail that benign overfitting requires.

\textbf{Experimental work on distribution shifts:}
\citet{HendrycksDietterich2019} propose the CIFAR-10C dataset as an OOD counterpart to CIFAR-10, featuring test set images corrupted by visual filters like blurs and noises.
\citet{Koh2021WILDS} present benchmarks on more realistic datasets with modern models that can be seen ``in-the-wild".
\citet{pmlr-v139-miller21b} experimentally show a linear relationship between ID accuracy and OOD accuracy for a wide range of modern neural networks and datasets, though their results show ID accuracy is almost always better than OOD accuracy.
On a subset of CIFAR-10C, we find settings in which OOD accuracy is \textit{better} than ID accuracy for linear interpolators.

\textbf{Benign overfitting ``in-distribution":}
\citet{bartlett2020benignpnas} propose benign overfitting, give a non-asymptotic analysis of the MNI, and show specific, necessary conditions under which the MNI achieves zero excess risk in-distribution.
\citet{tsigler2023benignjmlr} extend this work by considering benign overfitting in the case of ridge regression.
Our proof techniques follow most closely to the ideas presented in these two papers for the in-distribution setting.
\citet{pmlr-v178-frei22a} show benign overfitting in shallow non-linear MLPs trained with gradient descent on the logistic loss if the data dimension grows faster than the number of training samples.
\citet{Mallinar2022} experimentally show that interpolating neural networks do not benignly overfit due to the low input data dimension.
Our experiments build on this by looking at settings in which $n<p$ and $n>p$ where $n$ is the training sample size and $p$ is the input data dimension.
Other works study benign overfitting under a variety of conditions 
\citep{pmlr-v202-kou23a,chatterji2023deep,frei2023benign}.
 \section{Preliminaries}
We extend notations in \citet{bartlett2020benignpnas} and \citet{tsigler2023benignjmlr} to the transfer learning setting with OOD generalization risk as our performance metric.
Appendix \ref{apdx:formal_assumptions} formalizes our setting of linear regression under distribution shift, and we provide necessary details here.

\subsection{Linear Models for Source and Target Data}

Let $\dsource$ and $\dtarget$ be source and target distributions over $(x, y) \in \R^p \times \R$. We consider linear regression problems defined as follows.

\begin{definition}[Linear regression]\label{def:lin_reg_main} 
Let the training dataset be comprised of $n$ i.i.d. pairs $(\bx^i,y^i)_{i=1}^n \sim \dsource^n$  concatenated into a data matrix $\xsource\in \R^{n\times p}$ and a response vector $\ysource\in \R^n$, where $n < p$. We define
\begin{enumerate}
    \item the covariance matrix $\covsource=\E_{\dsource}[\bx\bx^\top]$,

    \item (\textit{centered rows}) $\E_{\dsource}[\bx] = 0,$

    \item (well-specified) the optimal parameter vector $\tsource \in \R^p$ such that $$y = \bx^T \tsource + \eps_\src$$ for $(\bx, y) \sim \dsource$, where $\beps_\src$ is a centered random variable with variance $v_{\varepsilon_{s}^2}$ and $\E_{\dsource}[y|\bx] = \bx^T \tsource$.
\end{enumerate}
\end{definition}

We test on $\dtarget$ with $\covtarget$, $\ttarget$, $\eps_\tgt$ defined in the same way. Note that $(\bx,y)$ is used to denote single observation pairs for both source and target data. We will differentiate between the two by explicitly denoting the distribution from which the pair is drawn. 

To facilitate our analysis, we introduce the following assumptions on the covariance matrices and the distribution of the data.

\begin{assumption}
    \label{asm:cov_assum}
    For linear regression problems (Def. \ref{def:lin_reg_main}), with source and target covariance matrices $\covsource$ and $\covtarget \in \R^{p \times p}$, we assume:
    \begin{enumerate}

        \item (simultaneously diagonalizability) $\covsource$ and $\covtarget$ commute; that is, there exists an orthogonal matrix $V\in\R^{p \times p}$ such that $V^\top\covsource V$ and $V^\top\covtarget V$ are both diagonal:
        \begin{align*}
            \covsource &= \E_{\bx \sim \dsource}[\bx\bx^T] = \textrm{diag}(\lsource_1, \lsource_2, ..., \lsource_p),\\
            \covtarget &= \E_{\bx \sim \dtarget}[\bx\bx^T] = \textrm{diag}(\ltarget_1, \ltarget_2, ..., \ltarget_p),
        \end{align*}
        where $\lambda_1 \geq \lambda_2 \cdots \geq \lambda_p$ and $\Tilde{\lambda}_i\lambda_i \geq 0$ for all $i$;

        \item (subgaussianity) the whitened observations $\bz=\bx^T\covsource^{-1/2}$ are centered i.i.d. vectors with independent coordinates and subgaussian norm $\sigma_x$; that is, for all $\gamma\in \R^p$,
        \begin{align*}
            \E[\exp(\gamma^\top \bz)] \leq \exp(\sigma_x^2 ||\gamma||^2/2)
        \end{align*}
    \end{enumerate}
\end{assumption}

Simultaneous diagonalizability is a common assumption in recent studies of high-dimensional linear regression \citep{lei2021linearregressiondistshift, kausik2023generalization, lejeune2024monotonic} and we show in Section \ref{sec:experiments} with experiments that our results hold even when this is violated.
Subgaussianity is also frequently used in statistical learning theory research and encompasses a wide array of distributions of interest \citep{vershynin-hdp}.

\subsection{Min-Norm Interpolator and Target Excess Risk}
Given a source data matrix $\xsource$, the minimum-norm interpolator (MNI) for any vector $\xi\in\R^n$ is defined as 
\begin{align*}
    \mni(\xi) &:= \argmin\left\{\norm{\theta}^2: X\theta = \xi\right\} \\
    &= X^T(XX^T)^{-1}\xi.
\end{align*}
If we consider $\xi=\ysource$, then we recover the MNI for the labels given by the response model, but our analysis will also involve implicit MNIs for different label vectors in $\R^n.$

The quantity that we seek to bound is the excess risk on the target distribution, which we define for an estimator $\theta \in \R^p$ as,
\begin{align}
R(\theta, \dtarget) := \E_{\dtarget}\l[\l(y-\bx^T\theta\r)^2 - \l(y - \bx^T\ttarget\r)^2\r].
\end{align}

We now derive bounds for the target excess risk and its expectation over the source response noise.
The proof of the following can be found in Appendix \ref{apdx:excess_risk_proof}.
\begin{restatable}{theorem}{excessriskdecomp}(Target excess risk decomposition)
    \label{thm:exc_risk_decomp} 
    The excess risk of the MNI trained on the source data, when evaluated on the target distribution, satisfies
\begin{align}
        R(\mni(\ysource), \dtarget) \leq &4 \modelshift + 4 \bias + 2 V_{\epssource},
    \end{align}
    and
\begin{align*}
        \E_{\epssource} R(\mni(\ysource), \dtarget) &= \modelshift + \bias + \E_{\epssource}V_{\epssource} \\
        &\quad + 2(\ttarget - \tsource)^\top\covtarget(\tsource-\mni(\xsource\tsource)),
    \end{align*}
    where we define
    \begin{align*}
        \modelshift &:= \snorm{\tsource - \ttarget}_{\covtarget}^2,\numberthis \label{eqn:model_shift}\\
        \bias &:= \snorm{\tsource - \mni(\xsource \tsource)}_{\covtarget}^2,\numberthis \label{eqn:bias_b2}\\
        V_{\epssource} &:= \snorm{\mni(\epssource)}_{\covtarget}^2 \numberthis \label{eqn:raw_variance},
    \end{align*}
    and $\norm{\bx}_M^2 := \bx^\top M \bx$.
\end{restatable}
We observe that $\modelshift$ is a deterministic model shift term and that no further analysis can improve its dependency on $\tsource$, $\ttarget$, or $\covtarget$.
The cross-term, $(\ttarget - \tsource)^\top \covtarget (\tsource - \mni(\xsource \tsource))$, is dominated by the bias and variance as evidenced by the upper bound. 
Therefore we focus our analysis on $\bias$ and $V_{\epssource}$.
A useful normalized version of $V_{\epssource}$ is defined by
\begin{align}
    \label{eqn:purified_variance}
    V = \E_{\epssource}\l[V_{\epssource}/v_{\epssource}^2\r]. 
\end{align}
Note that $\bias, V$ are reminiscent of the ID bias and variance in prior work \citep{bartlett2020benignpnas,tsigler2023benignjmlr}.

\subsection{Separation of Components and Effective Ranks}
\label{sec:effective_ranks}

For an index $k$, we define the following quantities related to the effective rank of the tail of $\covsource$ \citep{tsigler2023benignjmlr}: 
\begin{align*}
    \rho_k = \frac{\sum_{i > k} \lsource_i}{n\lsource_{k+1}}, \qquad R_k = \frac{(\sum_{i > k} \lsource_i)^2}{\sum_{i > k} \lsource_i^2}.
\end{align*}
 $\rho_k$ measures the ratio of the energy of the source covariance tail to the number of training data observations, after normalizing the tail eigenvalues. 
$R_k$ measures the quantity of noisy features and how evenly distributed their eigenvalues are. 
It is minimized when there is only one nonzero eigenvalue and maximized when there are many equal eigenvalues. 

Benign overfitting occurs if the MNI is overfit to noisy training labels and yet ID excess risk decays to zero.
The central finding of \citet{bartlett2020benignpnas} is that the only way benign overfitting happens for the MNI is if the following occurs: (1) there exists a $k^* = \min\{k \ : \ \rho_k \geq b\}$ for a universal constant $b > 1$, meaning that the last $p-k^*$ components of $\covsource$ have a high effective rank relative to the number of training samples, $n$; (2) the magnitudes of the bottom $p-k^*$ eigenvalues are small relative to the top $k^*$; and (3) $k^* \ll n$.
More formally, consider quantities $p = p(n)$, a sequence of source covariance matrices $\Sigma_n = \diag(\lsource_1, \cdots, \lsource_p)$,  $k^* = k^*(n)$ as defined above,  $R_{k^*} = R_{k*}(\Sigma_n)$, and $\rho_k = \rho_k(\Sigma_n)$. 
A sufficient condition for benign overfitting is,
\begin{align}
    \label{eqn:benign_limits}
    \lim_{n\rightarrow\infty}\rho_0 = \lim_{n \to \infty} k^*/n = \lim_{n \to \infty} n/R_{k^*} = 0.
\end{align}

If this occurs, then the MNI behaves similarly to an estimator with two components. 
One component has variance similar to the ordinary least squares (OLS) estimator in $k^*$ dimensions and bias similar to the ridge regression solution with ridge parameter proportional to $\sum_{i > k} \lsource_i$, a sort of data-induced regularization.
The other component is a high-dimensional component, which has vanishing variance when the data is sufficiently high-dimensional and a bias which is proportional to $\sum_{i > k} \lsource_i (\tsource)_i^2$ \citep{tsigler2023benignjmlr}.
From these conditions, we see that the top $k^*$ components are like ``signal" components of the data and the bottom $p-k^*$ components are ``noise" components. 

\subsection{Spiked Covariance Models}
\label{sec:k_eps}

We will consider a special case of the $(k,\epsilon)$-spike model, a canonical covariance structure that exhibits benign overfitting for the MNI \citep{Chatterji2022ImplicitBias, chatterji2023deep}, to experimentally show properties of interest.
\begin{definition}[$(k, \delta, \epsilon)$-spike model]
    For a source distribution $\dsource$, 
$\delta >0$ and $\epsilon >0$ such that $\delta \gg \epsilon$,
    let 
\begin{align*}
        \E_{\bx \sim \dsource}[\bx\bx^T] &= \diag(\underbrace{\lambda_1, \cdots, \lambda_k}_{= \delta}, \underbrace{\lambda_{k+1}, \cdots, \lambda_p}_{= \epsilon}).
    \end{align*}
In this simplified setting, there are $k$ high-energy ``signal'' directions and $p-k$ low-energy ``noise'' directions.
    For a target distribution $\dtarget$, we use different hyperparameters  $\Tilde{k}, \Tilde{\delta}, \Tilde{\epsilon}$ to similarly characterize a shifted covariance matrix.
\end{definition}

 \section{Main Theorems}
\label{sec:main_theorems}

This section provides upper and lower bounds for the variance and bias terms in Equation \ref{eqn:purified_variance} and Equation \ref{eqn:bias_b2}, respectively.
Appendix \ref{apdx:proof_sketch} gives a high-level overview of our proof techniques. 
Subsequent appendices provide complete proofs. 
The variance bounds are adapted from \citet{bartlett2020benignpnas}, while the bias lower bound is derived from \citet{tsigler2023benignjmlr}. 
Our contributions include a novel bias upper bound and a unique characterization of overparameterization degrees. 
We start with the bounds for the variance term.
Appendix \ref{apdx:variance_bound_proof} contains a proof of the following theorem.
\begin{restatable}{theorem}{varianceublb}(Upper and lower bounds for the variance term)
\label{thm:main_results_variance_bounds}
There exist universal constants $b, c_1 > 1$ given in Lemma \ref{lemma:bartlett_lemma10}, a universal constant $c_2$ given in Lemma \ref{cor:bartlett_cor24} and a constant $c > 1$ that only depends on $\sigma_x, c_1, c_2$, such that for $k \in (0, n/c)$, with probability at least $1 - 10e^{-n/c}$,
\begin{align}
    V &\geq \frac{1}{c n}\sum_{i=1}^p \frac{\ltarget_i}{\lsource_i} \min\left(1, \frac{\lsource_i^2}{\lambda_{k+1}^2(\rho_k + 1)^2}\right) =: \underline{V}.
\end{align}

If in addition $\rho_k \geq b$, with probability at least $1 - 7e^{-n/c}$,
\begin{align}
    V/c &\leq \frac{1}{n}\sum_{i=1}^k \frac{\ltarget_i}{\lsource_i} + n \frac{\sum_{i > k} \ltarget_i \lsource_i}{(\sum_{i > k} \lsource_i)^2} =: \overline{V}/c.
\end{align}
\end{restatable}

We first note that the variance lower bound does not depend on $\rho_k \geq b$ and so it holds for any interpolating linear model, even when benign source conditions are not satisfied.
However, we will see that if $\rho_k \geq b$ for some $k$, then the upper and lower bounds are tight.
In the case where $\covtarget = \covsource$, these bounds reduce to their in-distribution counterparts \citep{bartlett2020benignpnas}. 
Our variance bounds show that the excess risk contribution of each feature is scaled by the ratio of the target and source eigenvalues, $\ltarget_i/\lsource_i$. 
We immediately see that scaling down the target eigenvalues will lessen the overall contribution to variance and that scaling up the target eigenvalues will increase the contribution.  
We investigate these scaling factors and the separation of the first $k$ components and last $p-k$ components in Section \ref{sec:taxonomy_shifts}.

We now state upper and lower bounds for the bias term, $\bias$, given in Equation \ref{eqn:bias_b2}. 
The proof of the following theorem can be found in Appendix \ref{apdx:bias_bound_proof}.
\begin{restatable}{theorem}{biasublb}(Upper and lower bounds for the bias term)
\label{thm:main_results_bias_bounds}
For the lower bound only, assume that random models $\overline{\theta}$ are obtained from the underlying $\tsource$ as $(\overline{\theta})_i = \gamma_i (\tsource)_i$, where each $\gamma_i$ is an independent Rademacher random variable. 
There exists a universal constant $b > 1$, constants $c, C$ that depend only on $b$ and $\sigma_x$, and $k < n/C$ such that if $\rho_k \geq b$, then with probability at least $1-10e^{-n/c},$
\begin{align*}
    \E_{\Bar{\theta}} [\bias] &\geq \frac{1}{c}\l( \sum_{i=1}^k \frac{\ltarget_i}{\lsource_i}\frac{\lsource_i(\tsource)_i^2}{(1 + \frac{\lsource_i}{\lambda_{k+1}\rho_k})^2} + \sum_{i > k}  \ltarget_i (\tsource)_i^2 \r)=:\underline{\bias}.
\end{align*}

If we assume that $p$ is at most exponential in $n$, then with probability at least $1 - 5e^{-n/c}$,
\begin{align*}
    \bias/c \leq \snorm{\tsource}^2\sum_{i=1}^p\frac{\ltarget_i}{\lsource_i}\frac{\lsource_i}{\big(1 + \frac{\lsource_i}{\lsource_{k+1}\rho_k}\big)}=:\overline{\bias}/c.
\end{align*}

\end{restatable}
Note that while the lower bound is in expectation over the random models $\Bar{\theta}$, the resulting expression only depends on the ground-truth $\tsource$.
This Bayesian approach also appears in prior work, i.e. \citet{tsigler2023benignjmlr}.
In studying the bias lower bound, we observe a similar separation of signal and noise components and depence on eigenvalue ratios as in the variance bounds.

To show tightness of our bounds, we assume there exists a $k$ such that $\rho_k \geq b$ for some universal constant $b > 1$.
When this condition is satisfied, the variance bounds are tight up to constant factors.
The bias bounds leave a model-dependent and source covariance-dependent gap, which we discuss in the proof overview in Appendix \ref{apdx:proof_sketch} and in the complete proof found in Appendix \ref{apdx:tightness_proofs}.
\begin{restatable}{theorem}{variancebiastight}(Tightness of variance and bias bounds)
    \label{thm:tightness}
    Let the lower bound and upper bound of $V$ be given by $\underline{V}$ and $\overline{V}$, respectively. There exists a universal constant $b \geq 1$, and constant $c$ as defined in Theorem \ref{thm:main_results_variance_bounds}, and $k \in (0, n/c)$ such that if $\rho_k \geq b$, then 
    \begin{align*}
        \underline{V} / \overline{V} \in \left[b^{-2} (1 + b)^{-2}/c^2, 1\right]. 
    \end{align*}

    Let the lower bound and upper bound of $\bias$ be given by $\underline{\bias}$ and $\overline{\bias}$, respectively, and the assumptions of Theorem \ref{thm:main_results_bias_bounds} be satisfied.
    Then
    \begin{align*}
        \underline{\bias} / \overline{\bias}  &\in \l[ \frac{\min_i \l\{(\tsource)_i^2 \ : \ (\tsource)_i \neq 0\r\}}{\snorm{\tsource}^2 \l( 1 + b^{-1}\frac{\lambda_1}{\lambda_{k+1}}\r)}, 1 \r].
    \end{align*}
\end{restatable}

Note that the gap between our bias upper and lower bounds is independent of the target distribution.

\subsection{A Taxonomy of Shifts}
\label{sec:taxonomy_shifts}
We now present a taxonomy of covariate shifts on the target distribution inspired by our prior analysis.
We first consider OOD generalization and formally categorize shifts as \textit{beneficial} or \textit{malignant}.
\begin{definition}[Beneficial and Malignant shifts]
    For a source distribution, ${\calD}_s$, a target distribution, ${\calD}_t$, excess risk, $R$, and MNI, $\mni$, we say that a shift is 
    \begin{enumerate}
        \item \textbf{beneficial} if $R(\mni, \calD_s) > R(\mni, {\calD}_t)$,
        \item \textbf{malignant} if $R(\mni, \calD_s) < R(\mni, {\calD}_t)$.
    \end{enumerate}
\end{definition}

We define these shifts for excess risk and note in Appendix \ref{apdx:additional_exps_synthetic_mni} that, empirically, the variance is the primary contributor to excess risk and the bias contributions are negligible when $\covsource$ satisfies benign overfitting conditions.
This is in keeping with prior literature that focuses on studying variance in interpolating methods \citep{bartlett2020benignpnas}.
We will thus focus on variance in the following discussion.

Prior work shows that if $n, p \to \infty$ at the same rate, $\tr(\covsource) < \tr(\covtarget)$ results in malignant shifts on excess risk and $\tr(\covsource) > \tr(\covtarget)$ results in beneficial shifts on excess risk \citep{tripuraneni2021overparam}.
In this section we generalize these conditions by considering differing rates of $n, p \to \infty$ and measuring overparameterization by the modified ``effective rank" measure $R_k/n$ rather than $p$.
This leads us to a novel characterization of the role of overparameterization in covariate shifts.
For completeness, we describe their trace conditions in terms of our shifts in Appendix \ref{apdx:trace_conditions_simple_shifts}.

We first consider a simplified example with separate multiplicative shifts on the signal and noise components to facilitate intuition. 
While this is a special case, it provides valuable insights into the dynamics of overparameterization and covariate shift that are relevant to practice. 
Our general results, which allow for arbitrary multiplicative shifts in every direction, are presented in Appendix \ref{apdx:sufficient_arbitrary_shifts}.

Let $\covsource$ be a covariance matrix that satisfes benign source conditions and denote the ID variance term by $V_{id}$. Define $\covtarget$ by
$\ltarget_i = \alpha \lsource_i $ for $ i \leq k$, and 
    $\ltarget_i= \beta \lsource_i $ for $i > k$
with $\alpha, \beta \geq 0$. Let $V_{ood}$ denote the OOD variance term. By Theorem \ref{thm:main_results_variance_bounds}, up to constants, \begin{align}
    V_{id} &\approx k/n + n/R_k, \\
    \label{eqn:var_ood_expression}
    V_{ood} &\approx \alpha (k/n) + \beta (n/R_k),
\end{align}
where $R_k = (\sum_{i > k} \lsource_i)^2 / (\sum_{i > k} \lsource_i^2)$.

It is clear that if $V_{ood} - V_{id} > 0$ then we have a malignant shift on the variance, and if $V_{ood} - V_{id} < 0$ then we have a beneficial shift on the variance.
Observe that,
\begin{align}
    \label{eqn:variance_difference}
    V_{ood} - V_{id} \approx (\alpha - 1) (k/n) + (\beta - 1) (n/R_k).
\end{align}
In this expression, we see that the scales of signal and noise shifts, $\alpha$ and $\beta$, are important, as is the relationship between $k/n$ (the ``classical" rate) and $n/R_k$ (the ``high-dimensional" rate).
The quantity $n/R_k$ can be interpreted as an inverse measure of overparameterization, where smaller values correspond to higher levels of overparameterization. 
The rate of overparameterization relative to the classical rate of $k/n$ determines whether the shift on the first $k$ components ($\alpha$) or the shift on the last $p-k$ components ($\beta$) contributes more to the difference in excess risk. 

\begin{figure}
\centering
\subfigure{\includegraphics[width=.4\textwidth]{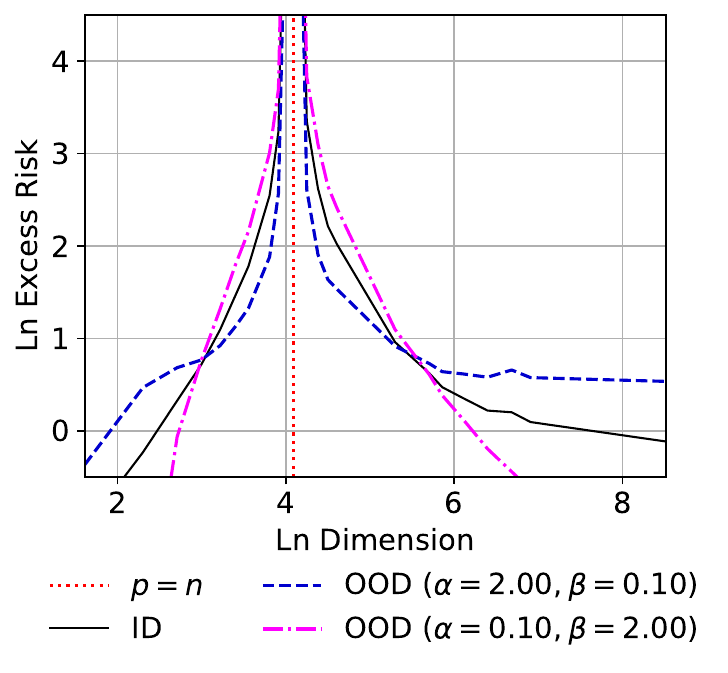}}

\caption{We experiment with the $(k, \delta, \epsilon)$ spiked covariance models and examine conditions for beneficial and malignant shifts as given in Theorem \ref{thm:benficial_malignant_shifts}. We take $n=60, k=10, \delta=1.0, \epsilon=1e^{-6}, \Tilde{\delta}=2.0, \Tilde{\epsilon} = 1e^{-7}$, and vary $p$. 
We see a cross-over from mild to severe overparameterization on the right side of $p=n$ where both OOD shifts swap between beneficial and malignant. For both ID and OOD curves, we observe that excess risk is a decreasing function if input dimension. Curves are averaged over 100 independent runs.}
\label{fig:k_eps_over_under}
\end{figure}

Based on this intuition, we define two regimes of overparameterization: \textit{mild} and \textit{severe}.

\begin{figure*}[htp]
    \centering
    \includegraphics[width=0.8\linewidth]{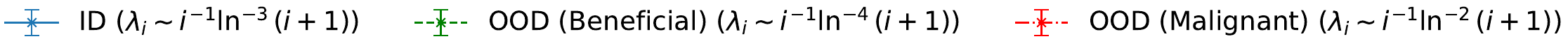}\\
    \subfigure[$n=200, p=20, h=512$]{\includegraphics[width=0.24\linewidth]{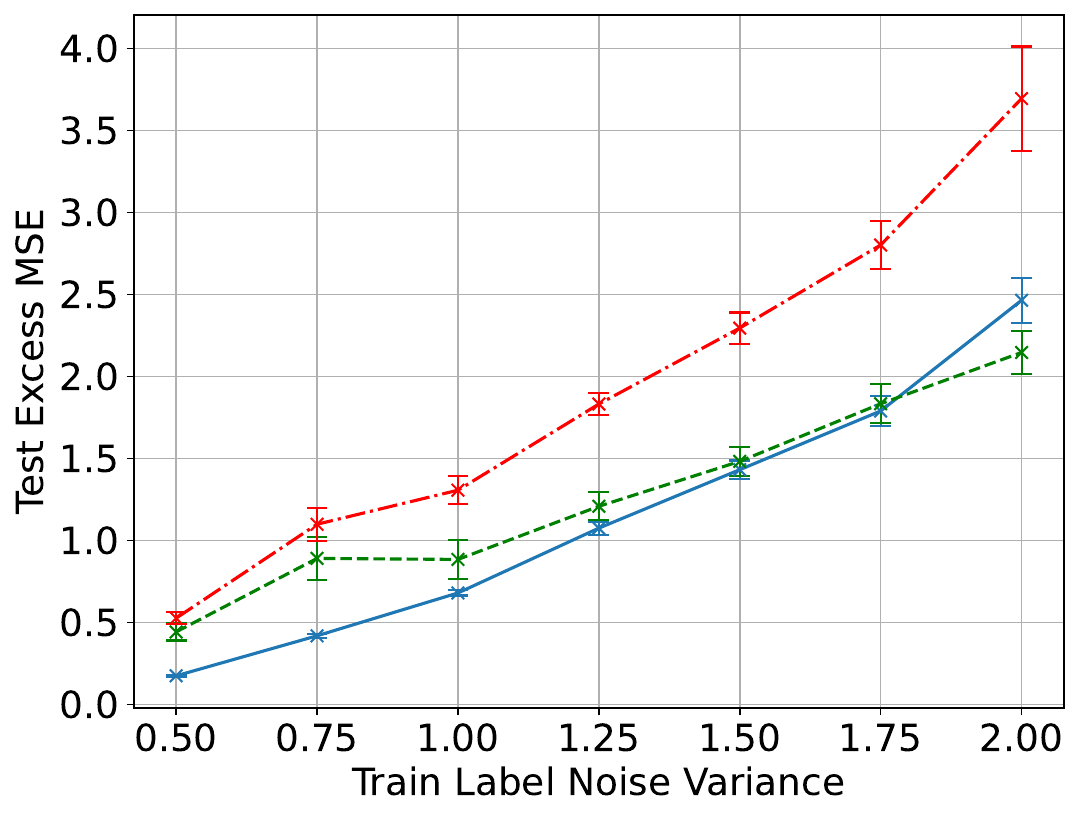}}\hfill
    \subfigure[$n=200, p=20, h=2048$]{\includegraphics[width=0.24\linewidth]{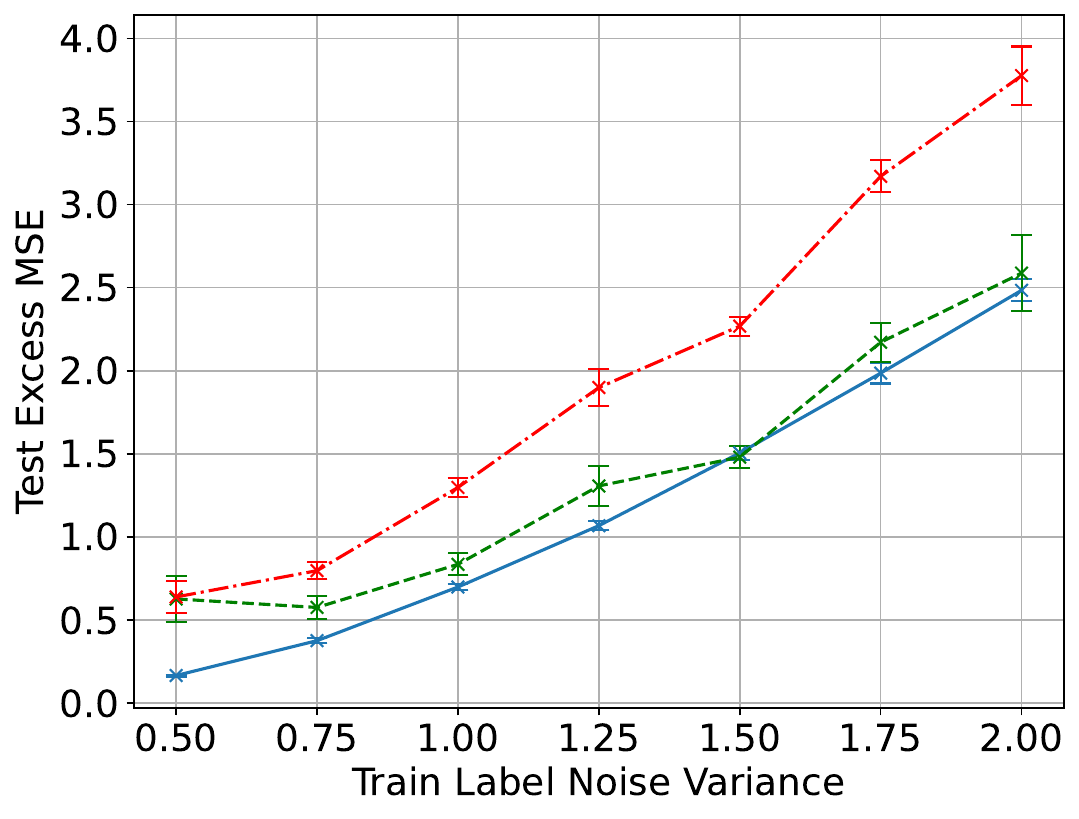}}\hfill
     \subfigure[$n=200, p=2k, h=512$]{\includegraphics[width=0.24\linewidth]{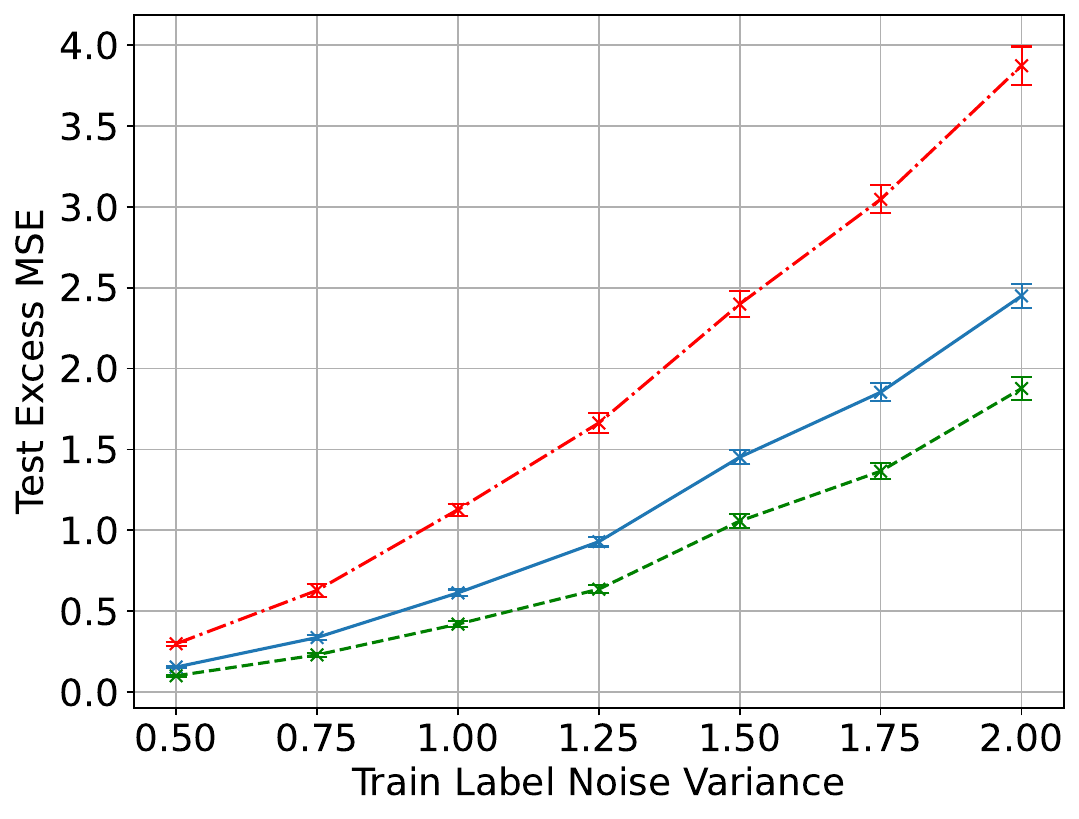}}
     \hfill
     \subfigure[$n=200, p=2k, h=2048$]{\includegraphics[width=0.24\linewidth]{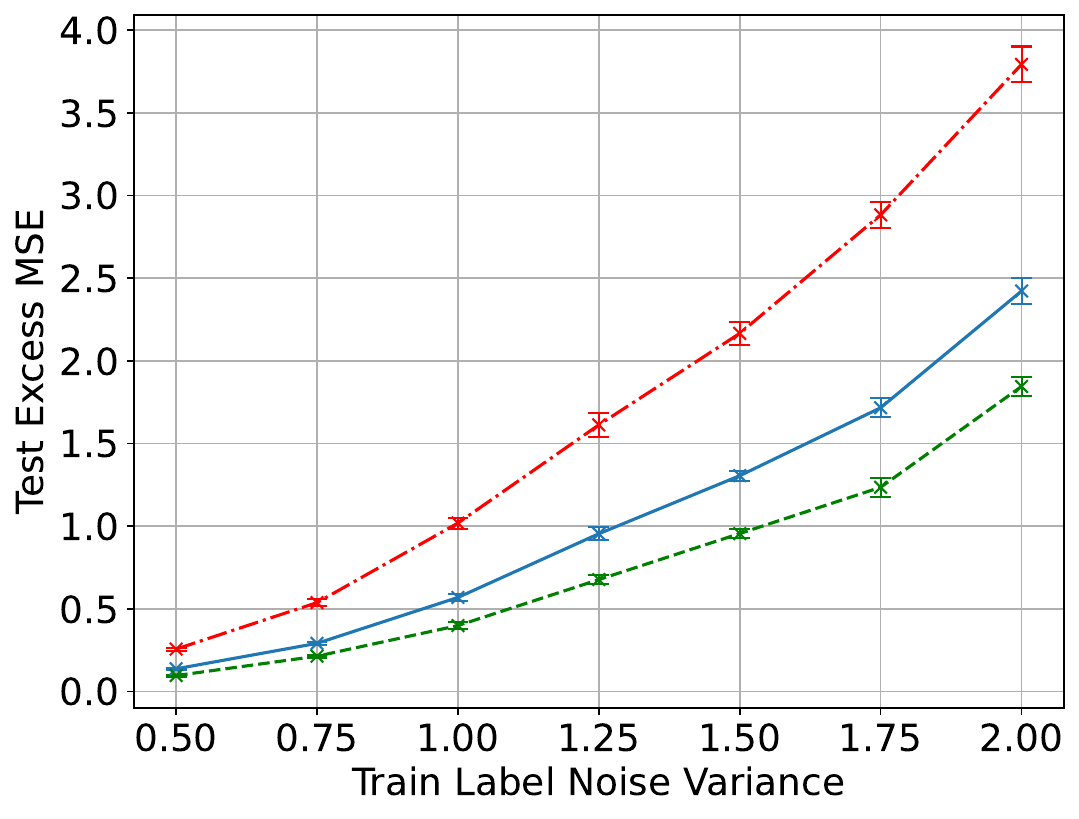}}
    \caption{We train 3 layer ReLU dense neural networks with hidden width, $h$, on $n$ samples from $p$-dimensional Gaussians. ID test data is sampled from the same distribution and OOD test sets are constructed based on beneficial and malignant covariate shifts in our theory. Ground truth models are sampled as $\tsource \sim \calS^{p-1}$, no model shift is invoked. For training data, $\xsource$, train labels are given by $\ysource = \xsource \tsource + \epssource$ with label noise $\epssource \sim \calN(0, \sigma^2)$. All runs reach train loss $< 5e^{-6}$. Points are averaged over 20 independent runs with standard error bars reported.}
    \label{fig:mlp_synth_regr}
\end{figure*}

\newpage
\begin{definition}[Mild and severe overparameterization for multiplicative shifts]
    Let $\covsource$ be a source covariance that satisfies benign source conditions and let $k \leq n$. Define $\covtarget$ as, $\ltarget_i = \alpha \lsource_i$ for $i \leq k$ and $\ltarget_i = \beta \lsource_i$ for $i > k$, with $\alpha, \beta \geq 0$.  
    Let $C_{\alpha\beta}:=\l|\frac{\alpha - 1}{1 - \beta}\r|$.

    We are in the \textbf{mildly overparameterized} regime if 
    \begin{equation*}
        n/R_k =\omega\l(C_{\alpha\beta}\cdot k/n \r).
    \end{equation*}
    We are in the \textbf{severely overparameterized} regime if
    \begin{equation*}      
        n/R_k =o\l(C_{\alpha\beta} \cdot k/n \r).
    \end{equation*}
    For $\beta = 1$ we define $C_{\alpha\beta} = \infty$ and thus are effectively in the severely overparameterized regime with regard to the types of shift we observe.
\end{definition}
It is clear that $k$ is important in defining regimes of overparameterization.
The aforementioned definitions hold for any $k < n$, however we derive our taxonomy of shifts in the case in which $\exists \ k < n$ such that $\rho_k \geq b$ for a universal constant $b > 1$.
We note that even for non-linear models or settings that do not exhibit benign overfitting we can still think about a notion of a ``$k$" akin to the intrinsic dimension of the data.
We empirically show in Figures \ref{fig:under_and_over} and \ref{fig:under_and_over_crossover_vary_n} that our taxonomy of shifts is reflective of shift behavior in realistic settings by heuristically taking $k$ small enough to sufficiently capture the low-dimensional signal in the data.

In Equation \ref{eqn:variance_difference}, we see that the limit of the severe overparameterization regime would take $R_k \to \infty$ first, while holding other problem parameters fixed.
In this case, we are only left with $\alpha$ shifts on the top $k$ components, as any shift on the bottom components is suppressed by the high rank covariance tail. 
This leads to behaviors akin to classical intuitions for an underparameterized linear regression estimator where $k = p < n$.
In this case, $\alpha > 1$ leads to more variance and thus harder learning, whereas $\alpha < 1$ leads to less variance and thus easier learning.
These notions of hard vs. easy learning naturally correspond to $\tr(\covtarget) > \tr(\covsource)$ and $\tr(\covtarget) < \tr(\covsource)$, respectively.
This is shown experimentally in Figs. \ref{fig:k_eps_over_under} and \ref{fig:under_and_over_crossover_vary_n} by looking at the left and right sides of the figures.

On the other hand, in the mildly overparameterized regime covariance tail shifts are not sufficiently suppressed and 
lead to non-negligible interactions with shifts on the signal components.
An increase in energy in the signal components can be counteracted by a decrease in energy in the noise components, effectively increasing the contrast between signal and noise in favor of the signal. Similarly, a decrease in energy in the signal components can be harmfully counteracted by an increase in energy in the noise components.
This is visible in Figs. \ref{fig:k_eps_over_under} and \ref{fig:under_and_over_crossover_vary_n} just above the threshold of interpolation.
Interestingly, in the mildly overparameterized regime we can also define settings in which $\tr(\covtarget) > \tr(\covsource)$ and yet still obtain a beneficial shift, and settings in which $\tr(\covtarget) < \tr(\covsource)$ and yet still obtain malignant shifts.

We formalize these observations in the following theorem, the proof of which can be found in Appendix \ref{apdx:proof_beneficial_malignant_simple_shifts}.
\begin{restatable}{theorem}{beneficialmalignantshifts}(Beneficial and Malignant Multiplicative Shifts on Variance)
    \label{thm:benficial_malignant_shifts}
    Let $\covsource$ be a source covariance that satisfies benign source conditions. That is, $\exists \ k$ such that $\rho_k \geq b$ for a universal constant $b > 1$. Define $\covtarget$ as $\ltarget_i = \alpha \lsource_i$ for $i \leq k$ and $\ltarget_i = \beta \lsource_i$ for $i > k$, with $\alpha, \beta \geq 0$.
    \begin{enumerate}
        \item If $\alpha < 1, \beta \leq 1$ or $\alpha \leq 1, \beta < 1$ then we obtain a beneficial shift in variance.
        \item If $\alpha > 1, \beta \geq 1$ or $\alpha \geq 1, \beta > 1$ then we obtain a malignant shift in variance.
        \item If we are in the mildly overparameterized regime:
        \begin{itemize}
            \item $\alpha > 1$ and $\beta < 1$ leads to beneficial shifts;
            \item $\alpha < 1$ and $\beta > 1$ leads to malignant shifts.
        \end{itemize}
        \item If we are in the severely overparameterized regime:
        \begin{itemize}
            \item $\alpha > 1$ and $\beta < 1$ leads to malignant shifts;
            \item $\alpha < 1$ and $\beta > 1$ leads to beneficial shifts.
        \end{itemize}
    \end{enumerate}
\end{restatable}

Figs. \ref{fig:k_eps_over_under} and \ref{fig:k_eps_extreme_overparam} demonstrate the relationship between the $n/R_k$ and $k/n$ rates in the case of $C_{\alpha\beta} = 1.11, C_{\alpha\beta} = 1$, respectively, for spiked covariance models.
In both, we clearly see a cross-over from beneficial to malignant shifts when we transition from mild to severely overparameterized.

\paragraph{Overparameterization improves OOD robustness}

Focusing on just the target excess risk, let $\alpha = \alpha(n)$ and $\beta = \beta(n, p)$.
We say that the benignly-overfit MNI is robust if its excess risk decays to zero despite the presence of multiplicative covariate shifts. 
In order for the variance upper bound to decay to 0, it is sufficient to have the shifts in the signal and noise components satisfy,
$\alpha = o(n/k),\; \beta = o(R_k/n).$
The condition $\beta = o(R_k/n)$ allows the shift strength to increase at a rate determined by the level of overparameterization, so we conclude that increasing the amount of overparameterization improves robustness to multiplicative distribution shifts.
Note that $\alpha$ has no dependence on $R_k$ and so robustness to shifts on the signal components is independent of the degree of overparameterization. 

\section{Experiments}
\label{sec:experiments}
Our theoretical results have provided insight into distribution shifts in high-dimensional linear regression.
We now present experiments with linear models and neural networks, relaxing many of the assumptions used for theoretical results.
Specifically,
we empirically: (1) observe beneficial and malignant shifts on synthetic and real data for linear models (benignly overfit and otherwise) and even high-dimensional dense neural networks; (2) show the benefit of overparameterization in covariate shift for interpolating linear estimators; (3) validate that our findings hold when the source and target covariance matrices are not simultaneously diagonalizable, as well as under model misspecification; (4) provide experimental insight that high-dimensional neural networks, i.e. when the input data dimension is large relative to the training sample size, act similarly to the MNI whereas low-dimensional neural networks do not, regardless of the level of overparameterization. Details of experimental setup, data, and models are given in Appendix \ref{apdx:experiment_details}.
We now discuss key observations and takeaways from the experiments.

\begin{figure*}[t]
    \centering
    \includegraphics[width=0.8\linewidth]{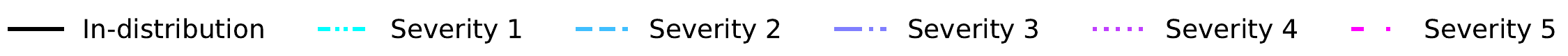}\\
    \subfigure[Noise then Blur Covariance]{\includegraphics[width=0.24\linewidth]{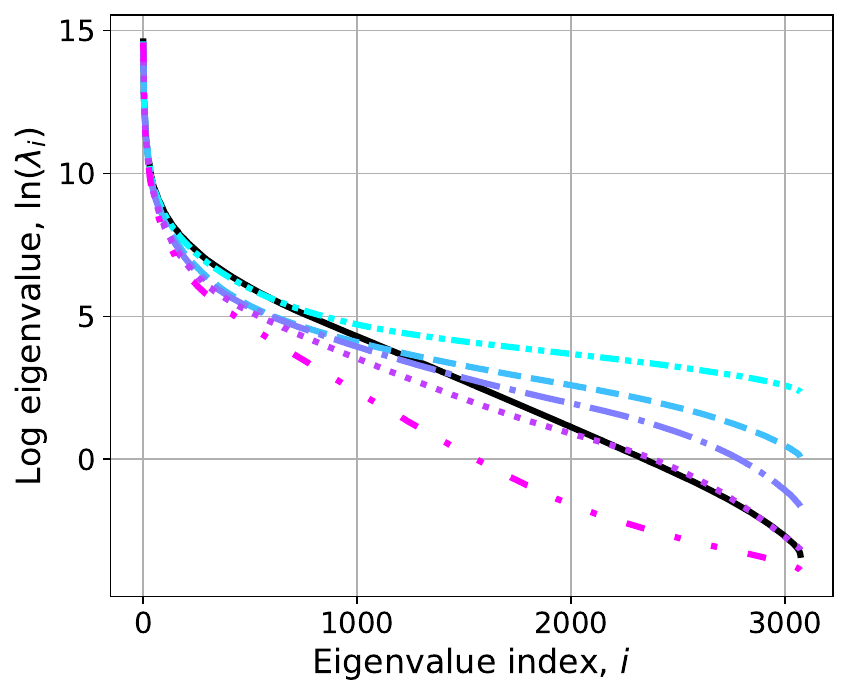}}\qquad
    \subfigure[Noise then Blur $n = 500$]{\includegraphics[width=0.24\linewidth]{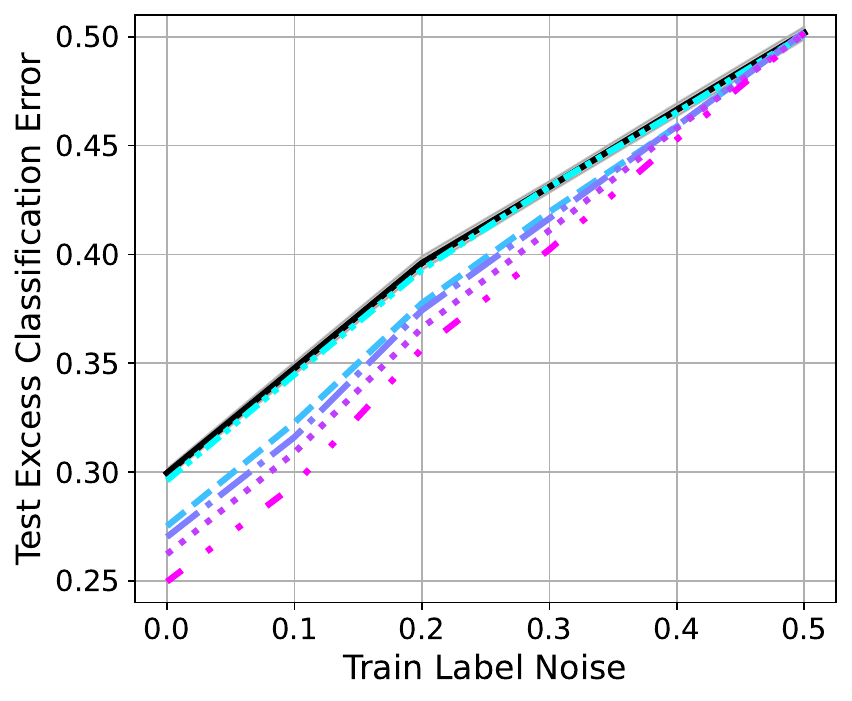}}\qquad
     \subfigure[Noise then Blur $n=2k$]{\includegraphics[width=0.24\linewidth]{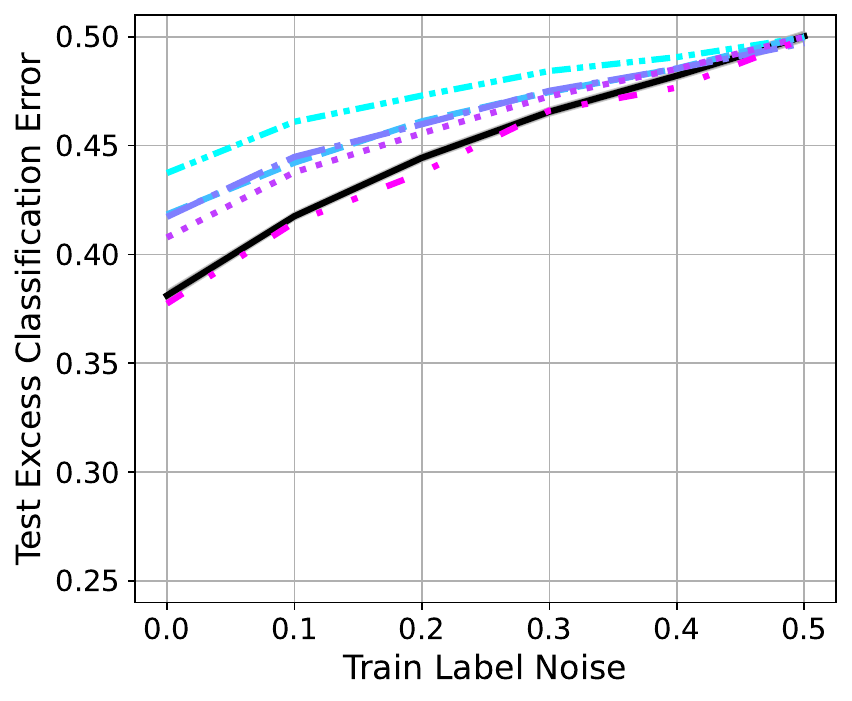}}\\
     \subfigure[Blur then Noise Covariance]{\includegraphics[width=0.24\linewidth]{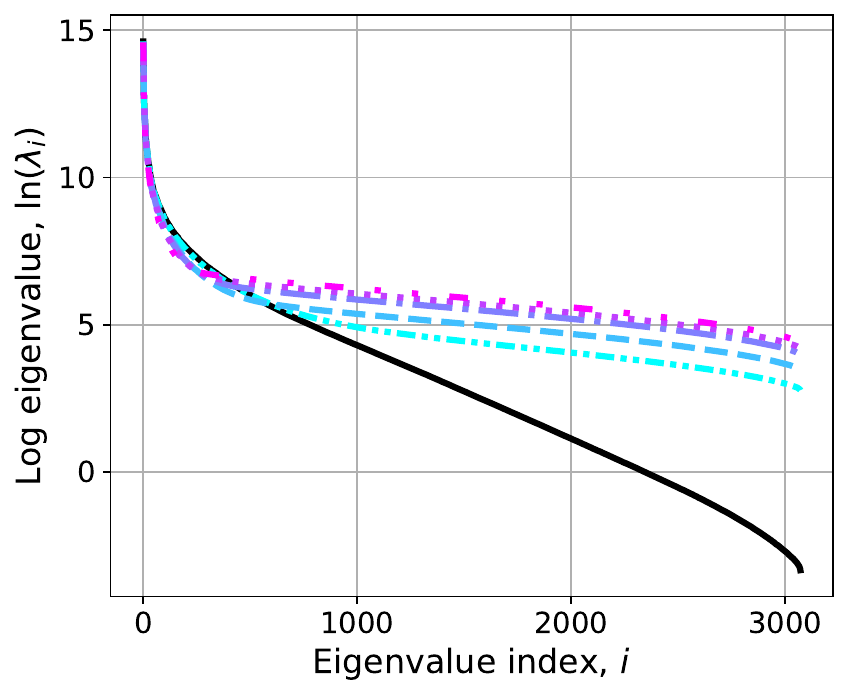}}\qquad
    \subfigure[Blur then Noise $n = 500$]{\includegraphics[width=0.24\linewidth]{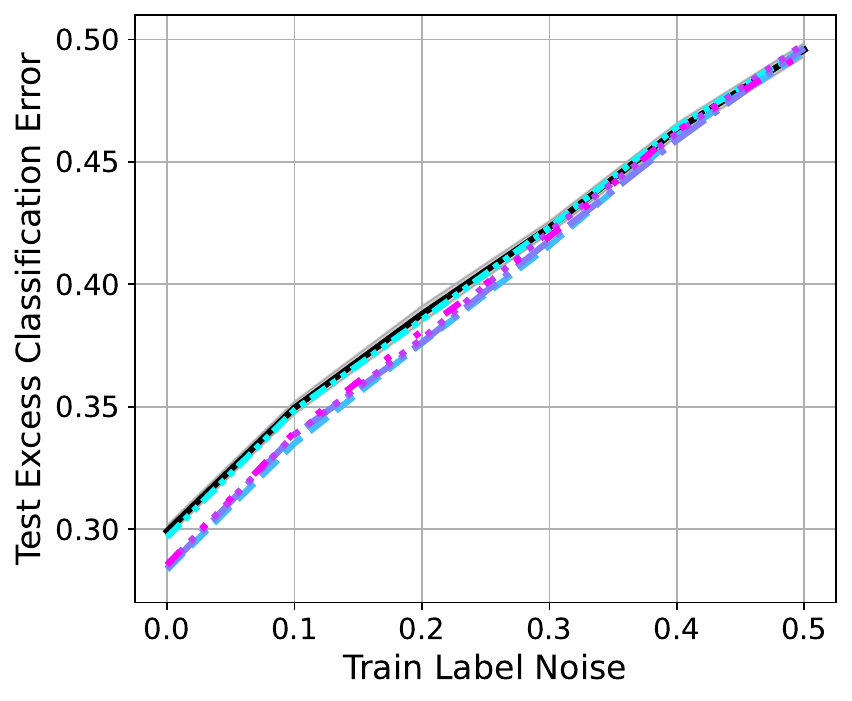}}\qquad
     \subfigure[Blur then Noise $n=2k$]{\includegraphics[width=0.24\linewidth]{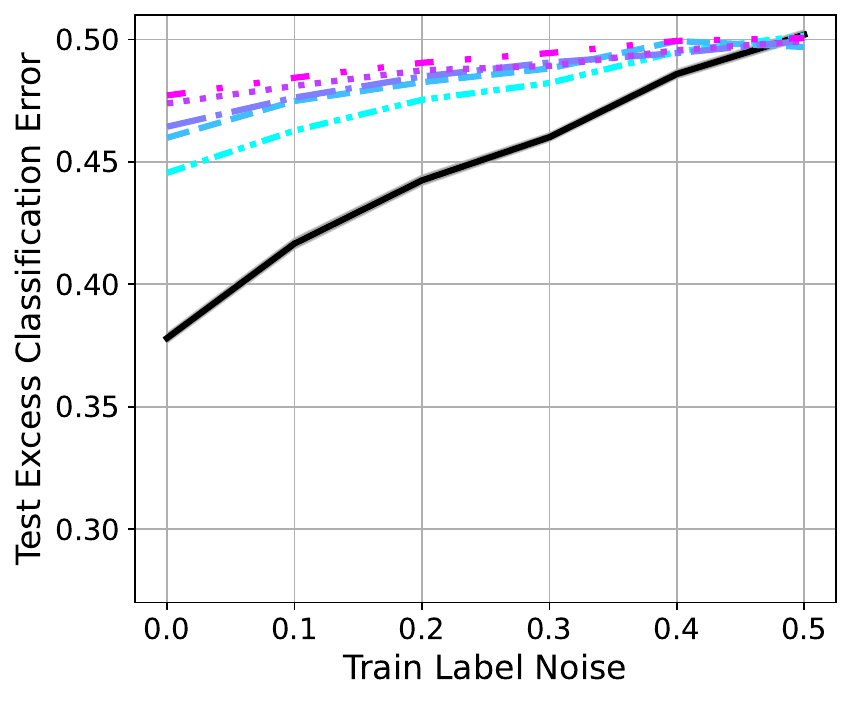}}
    \caption{We experiment with a custom variant of CIFAR-10C in which we apply the blur and noise image filters directly to the test set images of CIFAR-10 at each severity level, e.g. Severity 1 means that we add a small amount of noise \textit{and} a small amount of blurring to the image.
    In the top row we first use the noise filter and then the blur filter. In the bottom row we first use the blur filter and then the noise filter.
In (a) and (d), we observe that the eigenvalue decay of the shifts are non-monotonic and mirror the $\alpha < 1, \beta > 1$ setting in our taxonomy. Indeed, we also see in (b) and (e) that when we are severely overparameterized the noisy tail effects appear to be suppressed and we still obtain beneficial shifts. 
    On the other hand, in (c) and (f) we are in the mildly overparameterized regime and observe that the noisy tail effects hurt generalization, even for severity 4 in the top row which only adds a small amount of noise in the tail. 
    These results are exactly in keeping with our taxonomy for the $\alpha < 1, \beta > 1$ case.
    All curves are averaged over 50 independent runs.
    }
    \label{fig:cf10c_blur_and_noise_mni}
\end{figure*}

\subsection{Synthetic Data Experiments}

Fig. \ref{fig:k_eps_over_under} shows excess risk vs. input dimension for data sampled from the $(k, \delta, \epsilon)$-spike covariance model with $k=10$, $\delta=1.0$, and $\epsilon=1e^{-6}$. 
Beneficial and malignant shifts are seen in the setting of Theorem \ref{thm:benficial_malignant_shifts} with $\alpha=2.0, \beta=0.1$. 
That is, we see \textit{two} cross-over points: one in the underparameterized regime and one in the overparameterized regime (going from mild to severe).
This suggests that non-negligible covariance tail effects are a property of shifting when a model is in a region around the double descent peak.
The further we are from the double descent peak, the more ``classical" our behavior is, in that the top $k$ components are the only ones that influence shift and the bottom $p-k$ components either don't exist or have negligible effects.
Appendix \ref{apdx:additional_exps_synthetic_mni} explores this setup for different values of $\alpha, \beta$ and interpolating linear models for eigendecay rates that lead to harmful interpolation, i.e. \textit{tempered} or \textit{catastrophic} overfitting \citep{Mallinar2022}.

Figs. \ref{fig:mlp_synth_regr} and \ref{fig:mlp_log_plaw_k10_synthetic} show similar results for 3-layer dense ReLU neural networks trained until near-interpolation (train MSE $< 5e^{-6}$) on synthetic data with benign overfitting eigendecay rates \citep{bartlett2020benignpnas}.
For the neural network, we consider $p$ to be the dimension of the input data, rather than the number of network parameters.
From Fig. \ref{fig:mlp_synth_regr} we observe similar trends predicted by our theory for beneficial and malignant shifts when $p > n$, indicating that while our theory is developed for linear models we are able to extrapolate to more complex models.
In both the $p > n$ and $p < n$ experiments, our results are agnostic to the hidden width of the network, further suggesting that overparameterization is qualitatively different from high-dimensionality.
When $p > n$, a neural network appears to act like the interpolating MNI under distribution shifts.
For $p < n$ the interpolating dense net does not exhibit the properties of an interpolating MNI under distribution shift and the ID excess error is better than both ``beneficial" and ``malignant" OOD excess errors.
However, the relative difference between beneficial and malignant shifts is still preserved.
Note that we observe the exact same behavior in Fig. \ref{fig:cf10_resnets} when training ResNets to interpolation on CIFAR-10 and testing on CIFAR-10C blur and noise corruptions \citep{HendrycksDietterich2019}.

\subsection{CIFAR-10 Experiments}
Next, we consider experiments with linear interpolators on a binarized CIFAR-10 and CIFAR-10C with Gaussian noise and blur corruptions at varying levels of corruption severity.
For details, see Appendix \ref{apdx:experiment_details}.
This experiment breaks the assumption of simultaneous diagonalizability, and the well-specified assumption as the labels for CIFAR are not obtained by a ground-truth linear model.

Fig. \ref{fig:cf10c_mni} shows empirical results on the MNI fit to this problem. 
We plot the eigenspectra of the blur and noise covariances from CIFAR-10C compared to the eigenspectra of CIFAR-10 in Figs. \ref{fig:cf10c_mni}a, \ref{fig:cf10c_mni}b.
We identified these two shifts due to their eigenspectra reflecting what we expect to lead to beneficial and malignant shifts based on Theorem \ref{thm:benficial_malignant_shifts}.
We observe a tight relationship between changes in the eigenspectra of the target data and excess classification error when evaluated by the MNI.
We notice that blurs reduce covariance energy with increased blurring, like a ``denoising"-style operation.
Experimentally this leads to improved OOD accuracy.
In contrast, noise corruptions add energy to the covariance tail and lead to worsened OOD accuracy. Fig. \ref{fig:cf10c_mni_overparam} also shows that further overparameterization in this setting leads to improved behavior of the MNI on both corruptions.

Fig. \ref{fig:cf10c_blur_and_noise_mni} extends these results to a more realistic setting with custom variants of CIFAR-10C that involves applying both noise and blur filters on test set images.
Using both blur and noise filters together lead to covariate shifts that feature non-monotonic behaviors when comparing source to target eigenvalues.
This experiment highlights the $\alpha < 1, \beta > 1$ case and we see that the OOD accuracy matches the predictions of our taxonomy based on whether we are in the mildly overparameterized regime ($n=500$) or the severely overparameterized regime ($n=2k)$.
In Fig. \ref{fig:cf10c_artificial_blur_and_noise} we show plots for an artificially constructed version of this same experiment in which Gaussian noise is injected into the high variance directions of the blur test set, simulating the equivalent of $\alpha > 1, \beta < 1$ in our taxonomy. 
We similarly find the OOD accuracy to match the predictions of our taxonomy.

 \section{Conclusion and Future Work}
\label{sec:conclusion}
Our work provides the first finite-sample, instance-wise analysis of the MNI under transfer learning with high-dimensional linear models.
We show a taxonomy of beneficial and malignant covariate shifts depending on whether we are in a \textit{mild} or \textit{severely} overparameterized regime.
In the mildly overparameterized regime, variance contributions on the top $k$ components interact with that of the bottom $p-k$ components in non-negligible ways, leading to non-standard shifts. 
In the severely overparameterized regime, the high-rank covariance tail suppresses variance contributions in the bottom $p-k$ components and so OOD generalization acts more ``classical", akin to underparameterized linear regression where $k=p < n$.

Benign overfitting literature commonly claims to be motivated by ``overparameterized" neural networks, referring to the number of parameters in the network rather than the dimension of the data.
However recent works have challenged this, suggesting the role of the ambient dimension and source covariance is more important than parameter count in determining whether overfitting is benign or catastrophic in neural networks \citep{frei2023random,kornowski2023tempered}.
Prior work has also shown that gradient descent on 2-layer neural networks has an implicit bias towards linear decision boundaries when $p \gg n$, independent of the degree of overparameterization \citep{frei2022implicit}.

Our experiments further support the view that high-dimensional neural networks behave similarly to high-dimensional linear models, whereas low-dimensional neural networks do not.
They provide a new and important perspective on the difference between high-dimensionality and overparameterization in neural networks in the case of distribution shift, which has yet to be appreciated in the literature.
While dimensionality and degree of overparameterization are inextricably linked in linear regression, practical deep learning tends to operate in the overparameterized setting, not the high-dimensional one.

An important future direction is to investigate the extent to which our results hold for distribution shifts on more complex high-dimensional datasets. 
It is also of interest to extend our finite-sample theoretical analysis to shallow ReLU neural networks, other nonlinear models, and learning algorithms that overfit in a \textit{tempered} manner \citep{Mallinar2022}.
Finally, future work might seek to extend our understanding of neural networks by carefully studying the interplay between the data dimension, number of network parameters, number of training samples, and the optimization algorithm and loss function, and how this interplay can affect ID \& OOD generalization.

Another important future direction is to relax key assumptions in this work.
These results rely on a simultaneous diagonalizability assumption on the source and target covariance matrices which frequently appears in related works on high-dimensional linear regression \citep{lei2021linearregressiondistshift, kausik2023generalization, lejeune2024monotonic}. 
This enables us to highlight the different effects of covariate shifts in the ``signal" components vs. in the ``noise" components. 
In contrast, prior works (including those which can tolerate target covariances which do not satisfy simultaneous-diagonalizability \citep{tripuraneni2021overparam, lei2021linearregressiondistshift}) all rely on averages or traces over matrices involving the target covariance. 
Exploring techniques that can effectively handle non-simultaneously diagonalizable covariance matrices while preserving the insights gained from separating the impact of covariate shifts in signal and noise components is a promising avenue for future work.

\section*{Acknowledgements}

The authors thank Alexander Tsigler, Peter Bartlett, Emmanuel Abbe and Libin Zhu for useful discussion and feedback on the manuscript.

We gratefully acknowledge partial support from NSF grants DMS-2209975, 2015341, NSF grant 2023505 on Collaborative Research: Foundations of Data Science Institute (FODSI), the NSF and the Simons Foundation for the Collaboration on the Theoretical Foundations of Deep Learning through awards DMS-2031883 and 814639, and NSF grant MC2378 to the Institute for Artificial CyberThreat Intelligence and OperatioN (ACTION).
NM gratefully acknowledges funding and support for this research from the Eric and Wendy Schmidt Center at The Broad Institute of MIT \& Harvard. 
AZ additionally acknowledges support from NSF RTG Grant \#1745640.

This work used Delta GPU compute nodes at NCSA and HPE and Expanse GPU compute nodes at Dell and SDSC through allocation CIS220009 from the Advanced Cyberinfrastructure Coordination Ecosystem: Services \& Support (ACCESS) program, which is supported by National Science Foundation grants \#2138259, \#2138286, \#2138307, \#2137603, and \#2138296 \citep{access_grant}.

\printbibliography

\newpage
\appendix
\onecolumn

\section{Formal assumptions}
\label{apdx:formal_assumptions}

\begin{definition}[Linear regression under distribution shift]\label{def:lin_reg} 
We consider a training dataset comprised of $n$ i.i.d. pairs $(\bx^i,y^i)_{i=1}^n \sim \dsource^n$  concatenated into a data matrix $\xsource\in \R^{n\times p}$ and a response vector $\ysource\in \R^n$. The setting is overparameterized, meaning we have more input features than training samples, or $n < p$ .

We define
\begin{enumerate}
    \item the covariance matrix $\covsource=\E_{\dsource}[\bx\bx^\top]$,
    \item the optimal parameter vector $\theta_{src}^*\in \R^p$, satisfying 
    \[\E_{\dsource}\big[(y-x^\top\theta^*_{src})^2\big] = \textrm{min}_\theta\E_{\dsource}\big[(y-x^\top\theta)^2\big].\]
\end{enumerate}
We test on the distribution $\dtarget$ with $\covtarget$ and $\ttarget$ defined in the same way. We assume

\begin{enumerate}
    \item (\textit{centered rows}) $\E_{\dsource}[\bx] = 0;$
    \item (\textit{well-specified - source}) For $(X, \boldsymbol{y}) \subseteq \dsource$, $\boldsymbol{y} = X \tsource + \beps_\src$. We assume that the components of the source noise vector $\beps_\src$ are i.i.d. centered random variables with positive variance $v_{\varepsilon_{s}^2}$ and that $\E_{\dsource}[y|x] = x^T \tsource$;
    \item (\textit{well-specified - target}) For $(X, \boldsymbol{y}) \subseteq \dtarget$, $\boldsymbol{y} = X \ttarget + \beps_\tgt$. We assume that the components of the target noise vector $\beps_\tgt$ are i.i.d. centered random variables with noise variance, $v_{\varepsilon_{t}^2}$, and that $\E_{\dtarget}[y|x] = x^T \ttarget$;
    \item (\textit{simultaneously diagonalizability}) $\covsource$ and $\covtarget$ commute; that is, there exists an orthogonal matrix $V\in\R^p$ such that $V^\top\covsource V$ and $V^\top\covtarget V$ are both diagonal. This allows us to fix an orthonormal basis in which we can express the covariance matrices as
    \begin{align*}
        \covsource &= \E_{x \sim \dsource}[xx^T] = \textrm{diag}(\lsource_1, \lsource_2, ..., \lsource_p),\\
        \covtarget &= \E_{x \sim \dtarget}[xx^T] = \textrm{diag}(\ltarget_1, \ltarget_2, ..., \ltarget_p),
    \end{align*}
    where the source eigenvalues are a non-increasing sequence, $\lambda_1 \geq \lambda_2 \cdots \geq \lambda_p$. Note that we do not require the target eigenvalues to be a non-increasing sequence, however we require that $\Tilde{\lambda_i}\lambda_i \geq 0$ for all $i$;
    \item (\textit{subgaussianity}) the whitened data matrix, denoted $Z=\xsource\covsource^{-1/2}$, has centered i.i.d. row vectors with independent coordinates. We assume that the rows are subgaussian with subgaussian norm $\sigma_x$; that is, for all $\gamma\in \R^p,$
    \[\E[\exp(\gamma^\top z)] \leq \exp(\sigma_x^2 ||\gamma||^2/2). \]
\end{enumerate}
\end{definition} \section{Key results from prior work and technical lemmas}

For ease on the reader, we replicate some key lemma statements from \citet{bartlett2020benignpnas} and \citet{tsigler2023benignjmlr} and provide new lemmas and corollaries that we use in our work. 

Recall that $\rho_k = \frac{1}{n\lsource_{k+1}}\sum_{i > k} \lsource_i$, $\xsource \in \R^{n \times p}$, and $\covsource \in \R^{p \times p} = \text{diag}(\lsource_1, \cdots, \lsource_p)$. Let $X_{0:k} \in \R^ {n\times k}$ denote the matrix comprised of the first $k$ feature columns. 
Similarly, $X_{k:p} \in \R^{n\times (p-k)}$ denote the matrix of the last $p-k$ feature columns. The Gram matrix of the data, denoted here by 
\begin{align*}
    A = XX^T,
\end{align*} plays a central role in the investigation of high-dimensional linear regression.
Analogous to the above, we express $A_{0:k} = X_{0:k} X_{0:k}^T \in \R^{n \times n}$ and similarly for $A_{k:p} \in \R^{n \times n}$.

Letting $Z = X\covsource^{-1/2} \in \R^{n \times p}$ and denoting the independent column vectors of $Z$ by $z_i \in \R^n$, we have the following expressions:
\begin{align*}
    A = \sum_i \lsource_i z_i z_i^T, && A_{-i} = \sum_{j \neq i} \lsource_j z_j z_j^T, && A_k = \sum_{i > k} \lsource_i z_i z_i^T.
\end{align*}

The following lemma from \citet{bartlett2020benignpnas} is key in controlling the largest and smallest eigenvalues of the data Gram matrix and its variants $A_{-i}$ and $A_k$.
Importantly, it also shows that if the energy in the bottom $p-k$ components of the covariance matrix is sufficiently large ($\rho_k$ is lower bounded by a constant), then the largest and smallest eigenvalues of $A_k$ are equal up to constants.
\begin{lemma}[Lemma 5 from \citet{bartlett2020benignpnas}]
    \label{lemma:bartlett_lemma10}
    There are constants $b, c \geq 1$ such that for any $k \geq 0$, with probability at least $1 - 2e^{-n/c}$,

    \begin{enumerate}
        \item for all $i \geq 1$,
        \[\mu_{k+1}(A_{-i}) \leq \mu_{k+1}(A) \leq \mu_1(A_k) \leq c \Bigg(\sum_{j>k}\lambda_j + \lambda_{k+1}n\Bigg),\]
        \item for all $1 \leq i \leq k$,
        \[\mu_n(A) \geq \mu_n(A_{-i}) \geq \mu_n(A_k) \geq \frac{1}{c}\sum_{j>k}\lsource_j - c\lsource_{k+1}n,\]
        \item if $\rho_k \geq b$, then 
        \[\frac{1}{c}\lsource_{k+1}\rho_kn\leq \mu_n(A_k) \leq \mu_1(A_k) \leq c \lambda_{k+1}\rho_kn.\]
    \end{enumerate}
\end{lemma}

A consequence of the prior eigenvalue bounds is that when $\rho_k$ is lower bounded by a constant, the condition number of $A_k$ is upper bounded by a constant.
Therefore even as problem parameters such as training sample size and input dimension grow to $\infty$, $A_k$ is still well-conditioned.
This is important as non-benign overfitting occurs when the condition number bound on $A_k$ grows with problem parameters.
This would happen if the lower bound on the smallest eigenvalue of $A_k$ decays to zero too quickly which would cause the condition number of $A_k$ to diverge.
If this occurs then the excess risk of the MNI would be lower bounded.
This is shown for the in-distribution case in \citet{bartlett2020benignpnas}.
\begin{corollary}\label{cor:condnum}
     Following from Lemma \ref{lemma:bartlett_lemma10}, there are constants $b, c \geq 1$ such that for any $k \geq 0$, with probability at least $1 - 2e^{-n/c}$, if $\rho_k \geq b$ then
     \begin{equation}
         \frac{\mu_1(A_k)}{\mu_n(A_k)} \leq c^2
     \end{equation}
     which is the equivalent of the assumption \textit{CondNum}$(k, 2e^{-n/c}, c^2)$ as defined in \citet{tsigler2023benignjmlr}.
\end{corollary}

The following definition and lemma omit all references to \textit{NonCritReg} and the ridge parameter in \citet{tsigler2023benignjmlr}.

\begin{definition}[StableLowerEig$(k, \delta, L)$ from \citet{tsigler2023benignjmlr}]
    \label{def:stablowereig}
    Assume that for any $j \in \{1, 2, \cdots, p\}$ with probability (separate for every $j$) at least $1 - \delta$,
    \begin{equation}
        \mu_n(A_{-j}) \geq \mu_n(\E A_k) / L = (\sum_{i > k} \lsource_i) / L.
    \end{equation}
\end{definition}

We now state key assumptions that are necessary in order to obtain an explicit bias lower bound.
Exchangeable coordinates (\textit{ExchCoord}) is a weaker assumption than independent components of the data vector.
It is used in \citet{tsigler2023benignjmlr} instead of independent components.
We assume that components of $Z$ are independent and so we immediately satisfy the \textit{ExchCoord}, which we define here.
\begin{definition}[ExchCoord]
    \label{def:exchcoord}
    Assume the sequence of coordinates of $\covsource^{-1/2}x$, for any $x \in \xsource$, is exchangable (any deterministic permutation of the coordinates of whitened data vectors doesn't change their distribution).
\end{definition}

The \textit{PriorSigns} assumption is necessary to obtain lower bounds on the bias term. 
It allows us to use bounds on the expectation of a quadratic form, $\E_v [v^T M v]$, in order to separately analyze the contributions of $v$ and $M$.
As the bias takes the form $\tsource^T (I - X^T A^{-1}X) \covtarget (I - X^T A^{-1}X) \tsource$ we see that such a bound would separate the contributions of the model from that of data-dependent matrix expressions.
\begin{definition}[PriorSigns]\label{def:priorsigns} Assume that $\theta^*$ is sampled from a prior distribution in the following way: one starts with vector $\overline{\theta}$ and flips signs of all its coordinates with probability 0.5 independently.
\end{definition}
Under \textit{PriorSigns}, the random model vector is obtained by flipping signs on the components of the ground-truth model vector.
This does not affect our bounds as we see in Theorem \ref{thm:main_results_bias_bounds} that our bias lower bound only relies on squared components of the random model vector which are equivalent to the squared components of the ground truth model.

An important consequence of having a bounded condition number and independent coordinates is that with high probability the smallest eigenvalue of $A_{-i}$ for all $i \geq 1$ is lower bounded by $n \lambda_{k+1} \rho_k$ up to constants. These assumptions allow \citet{bartlett2020benignpnas} to prove Lemma \ref{lemma:bartlett_lemma10}, which in turn allows us to derive the \textit{StableLowerEig} condition. This is a simple consequence of B.1 and we provide details here for completeness.

\begin{corollary}[Our variant of StableLowerEig from \citet{tsigler2023benignjmlr}]
\label{lem:our_stab_lower_eig}
    For all $i \geq 1$, with probability at least $1-2e^{-n/c_2}$
\[\mu_n(A_{-i}) \geq \frac{1}{c_2}\mu_n( \E A_k) = \frac{1}{c_2}\sum_{j>k}\lambda_j.\]
\end{corollary}
\begin{proof}
    By Lemma \ref{lemma:bartlett_lemma10}, for some absolute constant $c_1\geq 1$ with probability at least $1-2e^{-n/c_1}$
    \[\mu_n(A_k)\geq \frac{1}{c_1}\sum_{i>k}\lambda_i - c_1\lambda_{k+1}n.\]
The assumption $\rho_k \geq b$ for some $b\geq 1$ gives us
\begin{align*}
     \frac{1}{c_1}\sum_{i>k}\lambda_i - c_1\lambda_{k+1}n &= \frac{1}{c_1}\lambda_{k+1}n\rho_k - c_1\lambda_{k+1}n\\
     & \geq \bigg(\frac{1}{c_1} - \frac{c_1}{b}\bigg)\lambda_{k+1}n\rho_k \\
     & = \bigg(\frac{1}{c_1} - \frac{c_1}{b}\bigg)\sum_{i>k}\lambda_i.
\end{align*}
Choosing $b>c_1^2$ and $c_2 = \max\{c_1, (1/c_1 - c_1/b)^{-1}\}$, we get that with probability at least $1-2e^{-n/c_2}$
\[\mu_n(A_k) \geq \frac{1}{c_2}\sum_{i>k}\lambda_i.\]
The next step is to extend this result to $A_{-i}$ for all $i$. 

For $i\leq k$, observe that $A_{-i}\succeq A_k$ gives us $\mu_n(A_{-i}) \geq \mu_n(A_k).$ For the case of $i>k,$ we have
\begin{align*}
    A_{-i} & = \sum_{j\neq i}\lsource_j z_j z_j^\top \\
    &= \sum_{j\leq k}\lsource_j z_j z_j^\top + \sum_{j> k, j\neq i}\lsource_j z_j z_j^\top\\
    &\succeq \lsource_1 z_1 z_1^\top + \sum_{j> k, j\neq i}\lsource_j z_j z_j^\top \\
    &\succeq \lsource_i z_1 z_1^\top + \sum_{j> k, j\neq i}\lsource_j z_j z_j^\top.
\end{align*}
We assume that the features are independent and $z_i$ is centered and whitened, so $\lsource_i z_1 z_1^\top + \sum_{j> k, j\neq i}\lsource_j z_j z_j^\top$ has the same distribution as $A_k = \sum_{j> k}\lsource_j z_j z_j^\top$. Therefore,
\begin{align*}
    & \P\Bigg(\mu_n(A_{-i})\geq \frac{1}{c_2}\sum_{i>k}\lambda_i\Bigg) \\
    \geq & \P\Bigg(\mu_n\Bigg(\lsource_i z_1 z_1^\top + \sum_{j> k, j\neq i}\lsource_j z_j z_j^\top\Bigg)\geq \frac{1}{c_2}\sum_{i>k}\lambda_i\Bigg) \\
    = & \P\Bigg(\mu_n(A_k)\geq \frac{1}{c_2}\sum_{i>k}\lambda_i\Bigg) \\
    \geq & 1-2e^{-n/c}.
\end{align*}
\end{proof}

The following corollaries provide high-probability bounds on random subgaussian vectors with independent coordinates. 
\begin{corollary}[Corollary 1 from \citet{bartlett2020benignpnas}]
    \label{cor:bartlett_cor24}
    There is a universal constant $c$ such that for any centered random vector $z \in \R^n$ with independent $\sigma^2$-subgaussian coordinates with unit variances, any random subspace $\mathscr{L}$ of $\R^n$ of codimension $k$ that is independent of $z$, and any $t > 0$, with probability at least $1 - 3e^{-t}$,
    \begin{align*}
        \snorm{z}^2 &\leq n + c\sigma^2(t + \sqrt{nt}),\\
        \snorm{\Pi_{\mathscr{L}} z}^2 &\geq n - c\sigma^2(k + t + \sqrt{nt}),
    \end{align*}
    where $\Pi_{\mathscr{L}}$ is the orthogonal projection on $\mathscr{L}$.
\end{corollary}
In our proofs, we will need to control the norm of $z_i$ for all $i\leq p$ on the same high-probability event.
In these cases we need to apply a union bound over the events in the summation.
The following corollary shows how to invoke a union bound over $\ell$ of these events in such a way that the probability over the union of all such events holds with high probability that depends $n$.
\begin{corollary}
    \label{cor:z_projz_corollary}
    There is a universal constant $c$ as defined in Corollary \ref{cor:bartlett_cor24}.
    Let $z \in \R^n$ be a centered random vector with $\sigma^2$-subgaussian coordinates and unit variances.
    Let $\mathscr{L}$ be a random subspace of $\R^n$ of codimension $k$ that is independent of $z$.

    For $0 < t < n/c_0$ and $k \in (0, n/c_1)$ for $c_1 > c_0$ with $c_0$ sufficiently large, with probability $1 - 3e^{-t}$,
    \begin{align*}
        \snorm{z}^2 &\leq c_2 n \\
        \snorm{\Pi_{\mathscr{L}} z}^2 &\geq n/c_3
    \end{align*}
    where $c_2, c_3$ only depends on $c, c_0, \sigma$.\\

    We obtain a union bound over the intersection of $\ell$ of these events so long as $\ln(\ell) \leq n/c_0 \Rightarrow \ell \leq e^{n/c_0}$.
    Then for $k \in (0, n/c_1)$ for $c_1 > c_0$ with $c_0$ sufficiently large, if $\ell \leq e^{n/c_0}$, with probability at least $1 - 3e^{-n/c_0}$, $\ell$ of the above events independently hold.
\end{corollary}

\begin{proof}
    Let Corollary \ref{cor:bartlett_cor24} hold with universal constant $c$.
    Then, with probability $1 - 3e^{-t}$ for $t > 0$,
    \begin{align*}
    \snorm{z}^2 &\leq n + c\sigma^2(t + \sqrt{nt}) \\
        \snorm{\Pi_{\mathscr{L}} z}^2 &\geq n - c\sigma^2(k + t + \sqrt{nt}).
    \end{align*}
    Let $t \leq \frac{n}{c_0}$. 
    Then we have that,
    \begin{align*}
        -\frac{n}{c_0} \leq -t \qquad -\frac{n}{\sqrt{c_0}} \leq -\sqrt{nt}.
    \end{align*}
    Plugging in for $\snorm{z}^2$,
    \begin{align*}
        \snorm{z}^2 &\leq n + c\sigma^2(t + \sqrt{nt}) \\
        &\leq n + c\sigma^2(\frac{n}{c_0} + \frac{n}{\sqrt{c_0}}) \\
        &= n(1 + c\sigma^2(c_0^{-1} + c_0^{-1/2})) \\
        &= c_1 n
    \end{align*}
    for $c_1$ only dependent on $c, c_0, \sigma$.
    Now, plugging in for $ \snorm{\Pi_{\mathscr{L}} z}^2$,
    \begin{align*}
         \snorm{\Pi_{\mathscr{L}} z}^2 &\geq n - c\sigma^2(k + t + \sqrt{nt}) \\
         &\geq n - c\sigma^2(k + \frac{n}{c_0} + \frac{n}{\sqrt{c_0}}) \\
        &= n(1 - c\sigma^2(\frac{k}{n} + c_0^{-1} + c_0^{-1/2})).
    \end{align*}
    Let $k < \frac{n}{c_2}$ for $c_2 > c_0$. Then it is clear that $-\frac{k}{n} > -\frac{1}{c_2}$ and,
    \begin{align*}
        n(1 - c\sigma^2(\frac{k}{n} + c_0^{-1} + c_0^{-1/2})) &\geq n(1 - c\sigma^2(c_2^{-1} + c_0^{-1} + c_0^{-1/2})) \\
        &= n/c_3
    \end{align*}
    for constant $c_3$ that only depends on $c, \sigma^2, c_0$.
    We finally require that $1 - c\sigma^2(c_1^{-1} + c_0^{-1} + c_0^{-1/2}) > 0$ which we can achieve by taking $c_0$ sufficiently large.

    We now proceed to bound the union of $\ell$ of the complement events, in order to obtain a bound over the intersection of $\ell$ of these events.

    For multiple $z_i$'s, define by $A_i$ the events shown above, that $\snorm{z_i}^2 \leq c_2 n$ and $\snorm{\Pi_{\mathscr{L}_i} z_i}^2 \geq n/c_3$ where $z_i$ and $\mathscr{L}_i$ are defined analogous to $z, \mathscr{L}$ above.
    Then
    \begin{align*}
        P(\cup_{i=1}^\ell (A_i)^{c}) &\leq \sum_{i=1}^\ell P((A_i)^c) \\
        &\leq \sum_{i=1}^\ell 3e^{-t} \\
        &= 3\ell e^{-t}.
    \end{align*}
    Then $P(\cap_{i=1}^\ell A_i) \geq 1 - 3\ell e^{-t}$.
    Observing that $3\ell e^{-t} = 3e^{\ln(\ell)}e^{-t} = 3e^{-t + \ln(\ell)} = 3e^{-(t - \ln(\ell))}$ we can set the per-event $t$ accordingly and obtain the necessary bound.
    We want $0 < t - \ln(\ell) \leq n/c_0$ to complete the bound.
    Therefore, we need that, per-event, $\ln(\ell) < t \leq n/c_0 + \ln(\ell)$.
    If $\ln(\ell) \leq n/c_0$ then this reduces to needing $\ln(\ell) < t \leq 2n/c_0$.
    Since each event is defined for $t \in (0, n/c_0]$ the union bound proof is complete by taking $t = n/c_0$ and requiring that $\ln(\ell) \leq n/c_0$.

\end{proof}

The following lemma is necessary in order to extend a summation over random variables, each lower bounded by a real number with equal probability, to a unified lower bound over the entire summation.
\begin{lemma}[Lemma 9 from \citet{bartlett2020benignpnas}]
    \label{lem:bartlett_lemma15}
    Suppose $n \leq \infty$ and $\{\eta_i\}_{i=1}^n$ is a sequence of non-negative random variables, $\{t_i\}_{i=1}^n$ is a sequence of non-negative real numbers (at least one of which is strictly positive) such that for some $\delta \in (0, 1)$ and any $i \leq n$, $P(\eta_i > t_i) \geq 1-\delta$. Then,
    \begin{align*}
        P\left(\sum_{i=1}^n \eta_i \geq \frac{1}{2}\sum_{i=1}^n t_i\right) \geq 1 - 2\delta.
    \end{align*}
\end{lemma}

We now provide a minor generalization of Corollary S.6 in \citet{bartlett2020benignpnas} that comes from replacing $a_1$ in a non-increasing sequence of non-negative numbers $\{a_i\}_{i=1}^p$ with $\max_i a_i$ and only requiring that $\{a_i\}_{i=1}^p$ is a sequence of non-negative numbers.
\begin{corollary}
    \label{cor:bernstein_any_seq}
    There is a universal constant $c$ such that for any sequence $\{a_i\}_{i=1}^p$ of non-negative numbers such that $\sum_{i=1}^p a_i < \infty$, and any independent, centered, $\sigma$-subexponential random variables $\{\xi_i\}_{i=1}^p$, and any $x > 0$, with probability at least $1 - 2e^{-cx}$,
    \begin{align*}
        |\sum_{i} a_i \xi_i| \leq \sigma \max \l( x \max_i a_i, \sqrt{x\sum_{i=1}^p a_i^2} \r).
    \end{align*}
\end{corollary}

Lastly, the following identity will allow us to use the \textit{PriorSigns} assumption to derive a new form for the bias term, which will be used for the proof of the lower bound.

\begin{lemma}[Identity for expectation of a quadratic form]
\label{lem:kend_quad_form}
    Assume $M\in \R^{p \times p}$ is a symmetric matrix. For a random vector $x\in\R^p$ with mean $\E[x]$ and covariance $\mathrm{Cov}(x)$, 
    \begin{align*}
        \E_x[x^T M x] = \E[x]^T M \E[x] + \tr(M \textrm{Cov}(x)).
    \end{align*}
\end{lemma}
\begin{proof}
    \begin{align*}
        \E_x[x^T M x] &=  \E[\tr(x^T M x)] \\
        &= \E[\tr(Mxx^T)] \\
        &= \tr(M\E[xx^T]) \\
        &= \tr(M\textrm{Cov}(x) + M\E[x]\E[x]^T ) \\
        &= \tr(M\textrm{Cov}(x)) + \tr(\E[x]M\E[x]^T ) \\
        &= \tr(M\textrm{Cov}(x)) + \E[x]M\E[x]^T.
    \end{align*}

\end{proof}
 \section{Proof of excess risk bound}
\label{apdx:excess_risk_proof}

We start by restating Theorem \ref{thm:exc_risk_decomp}.
\excessriskdecomp*

\begin{proof}
Let us begin by noting that the excess risk of any $\theta$ is given by,
\begin{align*}
R(\theta) &= \extarget {(x,y)} \l[ \l( y - x^\top \theta\r)^2 \r] - \extarget {(x,y)} \l[ \l(y - x^\top \ttarget\r)^2\r] \\
&= \extarget {(x,y)} \l[ \l( y - x^\top \ttarget + x^\top \ttarget - x^\top \theta \r)^2\r] - \extarget {(x,y)} \l[ \l(y - x^\top \ttarget\r)^2\r] \\
&= \extarget {(x,y)} \l[ \l( x^\top \ttarget - x^\top \theta \r)^2\r] +  2\extarget {(x,y)} \l[ \l(y - x^\top \ttarget\r) \l( x^\top \ttarget - x^\top \theta\r) \r] \\
&\overset{(i)}= \extarget {x} \l[ \l( x^\top \ttarget - x^\top \theta \r)^2\r].\numberthis \label{eq:excess.risk.target.modelshift}
\end{align*}
Equality $(i)$ uses that, conditional on $x$, $y-x^\top \ttarget|x$ is mean-zero, which is given in Assumption 3 (\textit{well-specified - target}).
So that
\[ \extarget {(x,y)} \l[ \l(y - x^\top \ttarget\r) \l( x^\top \ttarget - x^\top \theta\r) \r] = \E\l[  \l( x^\top \ttarget - x^\top \theta\r)\E \l[ \l(y - x^\top \ttarget\r) \big|x \r] \r]=0.\]

We now note that the source-data MNI can be decomposed as follows,
\begin{align*}
    \mnisource(\ysource) &= \xsource^\top (\xxtsource)^{-1} \ysource \\
    &= \xsource^\top (\xxtsource)^{-1} (\xsource\tsource + \epssource) \\
    &= \mnisource(\xsource\tsource) + \mnisource(\epssource)
\end{align*}

We can thus continue from~\eqref{eq:excess.risk.target.modelshift} to characterize the excess risk of the source-data MNI as
\begin{align*}
    R(\mnisource(\ysource)) &= \extarget x \l[ \l( x^\top \ttarget - x^\top \mnisource(\ysource) \r)^2 \r] \\
    &= \extarget x \l[ \l( x^\top \ttarget - x^\top (\mnisource(\xsource\tsource) + \mnisource(\epssource)) \r)^2 \r] \\
    &= \extarget x \l[ \l( x^\top (\ttarget -  \mnisource(\xsource\tsource)) - x^\top \mnisource(\epssource) \r)^2 \r] \\
    &\overset{(i)}\leq \extarget x \l[ 2\l( x^\top (\ttarget -  \mnisource(\xsource\tsource)) \r)^2 + 2\l( x^\top \mnisource(\epssource) \r)^2 \r] \\
    &= \extarget x \l[ 2\l( x^\top (\ttarget - \tsource + \tsource - \mnisource(\xsource\tsource)) \r)^2 + 2\l( x^\top \mnisource(\epssource) \r)^2 \r] \\
    &\overset{(ii)}\leq \extarget x \l[ 4\l( x^\top (\ttarget - \tsource) \r)^2 + 4\l( x^\top(\tsource - \mnisource(\xsource\tsource)) \r)^2 + 2\l( x^\top \mnisource(\epssource) \r)^2 \r].\numberthis 
\label{eq:excess.risk.ub.modelshift}
\end{align*}
In inequalities $(i)$ and $(ii)$, we have used Young's inequality, which implies $(a-b)^2 \leq 2(a-c)^2 + 2(b-c)^2$ for any $a,b,c\in \R$. 
Recalling that
\[ \snorm{x}_M^2 := x^\top M x,\] it is apparent that the first term is just the weighted distance between the source and target vectors,
\begin{align*} \label{eq:distance.target.source.modelshift}
  \extarget x \l[ \l( x^\top (\ttarget - \tsource) \r)^2\r] = (\ttarget - \tsource)^\top \extarget x \l[ x x^\top \r](\ttarget - \tsource) = \snorm{\tsource - \ttarget}_{\covtarget}^2.\numberthis 
\end{align*}

The second term looks quite similar to the bias term, $B$, in \citet{bartlett2020benignpnas} and \citet{tsigler2023benignjmlr}. 
\begin{align*}
&\extarget x \l[ \l( x^\top (\tsource - \mnisource(\xsource\tsource)) \r)^2 \r]\\
&\qquad = \l(\tsource - \mnisource\l(\xsource \tsource\r)  \r) \extarget x[xx^\top ] \l(\tsource -  \mnisource \l(\xsource \tsource\r) \r)\\
&\qquad= \snorm{\tsource - \mnisource\l(\xsource \tsource \r) }_{\covtarget}^2.\label{eq:bias.term.modelshift} \numberthis 
\end{align*}
The key difference with the standard supervised setting is that now the quantitiy in the middle is $\covtarget$, not $\covsource$.  Equivalently, the norm on $\tsource - \mnisource(\xsource \tsource)$ is induced by $\covtarget$ rather than $\covsource$. 

And finally, the third term is similar to the variance term, $C$, in ~\citet{bartlett2020benignpnas}: 
\begin{align*}
\extarget x \l[ \l( x^\top \mnisource(\epssource) \r)^2 \r] &= \mnisource(\epssource)^\top \extarget x [xx^\top] \mnisource(\epssource) \\
&= \mnisource(\epssource)^\top \covtarget \mnisource(\epssource) \\
&= \snorm{\mnisource(\epssource)}_{\covtarget}^2.\numberthis\label{eq:variance.term.modelshift}
\end{align*}
As in the bias term, the only difference is that the middle term is $\covtarget$ rather than $\covsource$.  Equivalently, the norm on $\mnisource(\epssource)$ is induced by $\covtarget$ rather than $\covsource$. 

Putting it all together, we get the following upper bound for the excess risk of the minimum-norm interpolator on the training data,
\begin{align*}
R(\mnisource(\ysource)) &\leq 4 \snorm{\tsource - \ttarget}_{\covtarget}^2 + 4 \snorm{\tsource - \mnisource(\xsource \tsource)}_{\covtarget}^2 + 2 \snorm{\mnisource(\epssource)}_{\covtarget}^2.
\end{align*}
This completes the upper bound for the risk.

For the lower bound, we have
\begin{align*}
    \E_{\epssource} R(\mnisource(\ysource)) &= \E_{\epssource, x \sim \dtarget} \l[ \l( x^\top \ttarget - x^\top \mnisource(\ysource) \r)^2 \r] \\
    &= \E_{\epssource, x \sim \dtarget} \l[ \l( x^\top (\ttarget -  \mnisource(\xsource\tsource)) - x^\top \mnisource(\epssource) \r)^2 \r] \\
    &= \E_{\epssource, x \sim \dtarget} \bigg[ \l( x^\top (\ttarget -  \mnisource(\xsource\tsource)) \r)^2 - 2\cdot x^\top (\ttarget -  \mnisource(\xsource\tsource)) \cdot x^\top \mnisource(\epssource) \\
    &\qquad + \l( x^\top \mnisource(\epssource) \r)^2 \bigg] \\
    &\overset{(i)}= \extarget x \l[ \l( x^\top (\ttarget -  \mnisource(\xsource\tsource)) \r)^2 \r] + \E_{\epssource, x \sim \dtarget} \l[\l( x^\top \mnisource(\epssource) \r)^2 \r]
\end{align*}
The equality $(i)$ uses that, conditional on $\xsource$, $\epssource$ is zero-mean.  
Note that the second term above is just $\E_{\epssource} \snorm{\mnisource(\epssource)}_{\covtarget}^2$, so we need only deal with the first term.  
Adding and subtracting $\tsource$ inside the square and expanding, we have
\begin{align*}
& \E_{x\sim \dtarget} \l[ \l( x^\top (\ttarget -  \mnisource(\xsource\tsource)) \r)^2 \r]\\
&\quad = \E_{x\sim \dtarget} \l[ \l( x^\top (\ttarget - \tsource) \r)^2\r] + \E_{x\sim \dtarget} \l[ \l(x^\top(\tsource - \mnisource(\xsource\tsource)) \r)^2 \r]\\
&\qquad + 2 \E_{x\sim \dtarget} \l[ (\ttarget - \tsource)^\top x x^\top (\tsource - \mnisource(\xsource\tsource)) \r]\\
&\quad =  \snorm{\ttarget-\tsource}_{\covtarget}^2 + \snorm{\tsource - \mnisource(\xsource \tsource)}_{\covtarget}^2 + 2(\ttarget - \tsource)^\top \covtarget (\tsource - \mnisource(\xsource \tsource)). 
\end{align*}

\end{proof}

 \section{Overview of variance and bias proof techniques}
\label{apdx:proof_sketch}

The central pillar of both proofs is controlling the eigenvalues of $A_k$, which in turn provides certain bounds on the eigenvalues of $A$ and $A_{-i}$. A key finding of \citet{bartlett2020benignpnas} is that once $\rho_k$ is large enough, all eigenvalues of $A_k$ are identical up to a constant factor. Specifically,
\begin{align*}
    z^T A z \approx n^2\lsource_{k+1}\rho_k, && z^T A^{-1} z \approx n(n\lsource_{k+1}\rho_k)^{-1}.
\end{align*}
\subsection{Variance}
Due to independence between the components of $\epssource$, the variance term from Eqn. \ref{eqn:purified_variance} can be expressed as 
\begin{align*}
    V &= \E_{\varepsilon_\src} [V_{\varepsilon_\src} / v_\varepsilon^2] \\
    &= \tr(A^{-1}\xsource \Tilde{\Sigma} \xsource^\top A^{-1}) \\
    &= \sum_{i=1}^p \Tilde{\lambda_i} \lambda_i z_i^T A^{-2} z_i.
\end{align*}
Now that we are dealing with a sum of quadratic forms, we consider the first $k^*$ signal and last $p-k^*$ noise components separately. Using the Sherman-Morrison formula the former can be written as
\begin{align*}
    \sum_{i\leq k^*} \Tilde{\lambda_i} \lambda_i z_i^T A^{-2} z_i &= \sum_{i\leq k^*} \frac{\ltarget_i}{\lsource_i} \frac{\lsource_i^2 z_i^T A_{-i}^{-2}z_i}{(1 + \lambda_i z_i^T A_{-i}^{-1}z_i)^2} \\
    &\approx \sum_{i\leq k^*} \frac{\ltarget_i}{\lsource_i} \frac{\lsource_i^2 n (n\lsource_{k+1}\rho_k)^{-2}}{\lsource_i^2n^2(n\lsource_{k+1}\rho_k)^{-2}} \\
    & =  \sum_{i\leq k^*} \frac{\ltarget_i}{\lsource_i} \frac{1}{n},
\end{align*}
where $\lambda_i z_i^T A_{-i}^{-1}z_i$ dominates $1$ for $i\leq k^*$. For the sum over the noise components the $1$ in the denominator dominates the other term and so we directly analyze the tail contributions as,
\begin{align*}
    \sum_{i>k^*}\frac{\ltarget_i}{\lsource_i} \lsource_i^2 z_i^T A^{-2} z_i &\approx \sum_{i>k^*}\frac{\ltarget_i}{\lsource_i} \lsource_i^2 n(n\lsource_{k+1}\rho_k)^{-2}. 
\end{align*}

The result is that the variance term is upper and lower bounded by
\begin{align*}
    \frac{1}{n}\sum_{i=1}^k \frac{\ltarget_i}{\lsource_i} + \sum_{i > k} \frac{\ltarget_i}{\lsource_i} \l( \frac{\lsource_i^2}{n\lambda_{k+1}^2 \rho_k^2} \r)
\end{align*}
times constant factors.
\subsection{Bias} 
As in Eqn. \ref{eqn:bias_b2}, the bias term is given by
\begin{align*}
    \bias &= \snorm{\tsource - X^TA^{-1}X\tsource}_{\covtarget}^2 \\
    &= \tr(\tsource^T(I-X^TA^{-1}X)\covtarget(I-X^TA^{-1}X)\tsource) \\
    &\leq \tr(\tsource\tsource^T)\cdot \tr((I-X^TA^{-1}X)\covtarget(I-X^TA^{-1}X)) \\
    &= \snorm{\tsource}^2\sum_{i=1}^p\frac{\ltarget_i}{\lsource_i}\sum_{j=1}^p\Big(e_i[j]-\sqrt{\lambda_i\lambda_j}z_i^\top A^{-1}z_j\Big)^2,
\end{align*}
where we use the Cauchy-Scharwz inequality to separate the parameter vector from the quadratic form. A quick application of the Sherman-Morrison formula allows us to write
\begin{align*}
    \bias &\leq \snorm{\tsource}^2\sum_{i=1}^p\Tilde{\lambda}_i \frac{1}{1 + \lsource_i z_i^TA_{-i}^{-1}z_i}.
\end{align*}
From here, we once again exert control over the eigenvalues of $A_{-i}$ to get
\begin{align*}
     \frac{1}{1 + \lsource_i z_i^TA_{-i}^{-1}z_i} &\approx  \frac{1}{1 +\frac{\lsource_i}{\lsource_{k+1}\rho_k}},
\end{align*}
which completes the upper bound proof sketch.

Note that the looseness of the bias bounds largely stems from the application of the Cauchy-Schwarz inequality. The only situations in which the bound becomes an equality are when 
\begin{align*}
    c\tsource = (I-X^TA^{-1}X)\covtarget^{1/2}
\end{align*}
for some scalar $c\in\R$ or when $\tsource$ is the zero vector.

Between the upper and lower bounds, the latter is likely tighter due to the use of the \textit{PriorSigns} assumption. As detailed in Appendix \ref{apdx:bias_lower}, it allows us to write

\begin{align*}
    B \geq \tsource^T(I-\textrm{diag}(X^TA^{-1}X))\covtarget(I-\textrm{diag}(X^TA^{-1}X))\tsource,
\end{align*}
where for a matrix $Q\in\R^{m\times m}$, we use $\textrm{diag}(Q)\in\R^{m\times m}$ to denote zeroed off-diagonal entries. The contribution of the off-diagonal entries is non-negative and dominated by the diagonals, so they can be dropped in the lower bound while preserving tightness under the $\textit{PriorSigns}$ assumption. In general, non-negative terms cannot be discarded in the proof of an upper bound, so we resort to the Cauchy-Schwarz inequality in order to avoid addressing the off-diagonals directly. However, decoupling the model vector $\tsource$ from the matrix $(I-X^TA^{-1}X)\covtarget^{1/2}$ introduces another degree of looseness, contributing to the gap between our bounds. Improving our upper bound will require controlling the off-diagonals of this matrix product with a technique more appropriate than Cauchy-Schwarz. \section{Proof of variance bounds}
\label{apdx:variance_bound_proof}

\varianceublb*

\begin{proof}
    We derive the variance terms necessary here and finish the proof of the upper bound in Appendix \ref{apdx:variance_upper} and the lower bound in Appendix \ref{apdx:variance_lower}.

    We follow the proof techniques in \citet{bartlett2020benignpnas,tsigler2023benignjmlr}.
    Observe that we can express the variance term as follows,
    \begin{align*}
        V &= \E_{\varepsilon_\src} [V_{\varepsilon_\src} / v_\varepsilon^2] \\
        &= \E_{\varepsilon_\src} [\snorm{\xsource^\top(\xsource \xsource^\top)^{-1}\varepsilon_\src}_{\covtarget}^2 / v_\varepsilon^2].
    \end{align*}
    Defining $A = \xsource \xsource^\top$,
    \begin{align*}
        V &= \E_{\varepsilon_\src} [\snorm{\xsource^\top A^{-1}\varepsilon_\src}_{\covtarget}^2 / v_\varepsilon^2] \\
        &= \E_{\varepsilon_\src} [(\varepsilon_\src^\top A^{-1}\xsource \covtarget \xsource^\top A^{-1}\varepsilon_\src) / v_{\varepsilon}^2] \\
        &= \E_{\varepsilon_\src} [\tr(\varepsilon_\src^\top A^{-1}\xsource \covtarget \xsource^\top A^{-1}\varepsilon_\src) / v_{\varepsilon}^2].
    \end{align*}
    Using the trace trick,
    \begin{align*}
        V &= \tr(A^{-1}\xsource \covtarget \xsource^\top A^{-1}\E_{\varepsilon_\src}[\varepsilon_\src \varepsilon_\src^\top])/v_\epsilon^2 \\
         &= \tr(A^{-1}\xsource \covtarget \xsource^\top A^{-1} v_\epsilon^2 I_n)/v_\epsilon^2 \\
         &= \tr(A^{-1}\xsource \covtarget \xsource^\top A^{-1}) \\
         &= \tr(\xsource \covtarget \xsource^\top A^{-2}) \\
         &= \tr((\sum_{i=1}^p \Tilde{\lambda_i} x^i (x^i)^\top) A^{-2}) \\
         &= \tr((\sum_{i=1}^p \Tilde{\lambda_i} \lambda_i z_i z_i^\top) A^{-2})
    \end{align*}
    where $x^i \in \R^n$ and $\frac{x^{i}}{\sqrt{\lambda_i}} = z_i \in \R^n$ are columns of $\xsource \in \R^{n \times p}$ and $\xsource \covsource^{-1/2} \in \R^{n \times p}$, respectively. 
    Continuing the calculation, we have that
    \begin{align*}
        V &= \sum_{i=1}^p \Tilde{\lambda_i} \lambda_i \tr(z_i^T A^{-2} z_i) \\
        &= \sum_{i=1}^p \Tilde{\lambda_i} \lambda_i \tr(z_i^T (A_{-i} + z_i z_i^T \lambda_i)^{-2} z_i) \\
        &= \sum_{i=1}^p \Tilde{\lambda_i} \lambda_i \frac{z_i^T A_{-i}^{-2}z_i}{(1 + \lambda_i z_i^T A_{-i}^{-1}z_i)^2}
    \end{align*}
    where $A_{-i} = XX^T - \lambda_i z_i z_i^T = \sum_{j \neq i} \lambda_j z_j z_j^T$.
    This expression will serve as the starting point for the variance term, which we will now proceed to upper and lower bound.

\end{proof}

\subsection{Upper bound}
\label{apdx:variance_upper}

After isolating the contribution of $\frac{\Tilde{\lambda_i}}{\lambda_i}$, most of the components of this proof are as given in the proof of Lemma 6 in \citet{bartlett2020benignpnas}.
For completeness, we replicate them here and refer the reader to their paper for further details and intuitions.

We start by separating the variance term into the top $k$ components and the bottom $p - k$ components as follows,
\begin{align*}
    V = \sum_{i=1}^k \frac{\Tilde{\lambda_i}}{\lambda_i} \frac{\lambda_i^2 z_i^T A_{-i}^{-2}z_i}{(1 + \lambda_i z_i^T A_{-i}^{-1}z_i)^2} + \sum_{i > k} \frac{\Tilde{\lambda_i}}{\lambda_i} (\lambda_i^2 z_i^T A^{-2} z_i).
\end{align*}

Fix constants $b, c_1 \geq 1$ as defined in Lemma \ref{lemma:bartlett_lemma10}. 
Then, with probability $1 - 2e^{-n/c_1}$, if $\rho_k \geq b$ then for all $z \in \R^n$ and $i \in [1, k]$,
\begin{align*}
    z_i^T A_{-i}^{-2}z_i &\leq \mu_1(A_{-i}^{-2}) \snorm{z_i}^2 \\
    &\leq \mu_n(A_{-i})^{-2} \snorm{z_i}^2 \\
    &\leq \frac{c_1^2\snorm{z_i}^2}{(n \lambda_{k+1}\rho_k)^2}
\end{align*}
and on the same event,
\begin{align*}
    z_i^T A_{-i}^{-1}z_i &\geq (\Pi_{\calL_i} z_i)^T A_{-i}^{-1} (\Pi_{\calL_i} z_i) \\
    &\geq \mu_n(A_{-i}^{-1}) \snorm{\Pi_{\calL_i} z_i}^2 \\
    &\geq \mu_{k+1}(A_{-i})^{-1} \snorm{\Pi_{\calL_i} z_i}^2 \\
    &\geq \frac{\snorm{\Pi_{\calL_i} z_i}^2}{nc_1 \lambda_{k+1}\rho_k}
\end{align*}
where $\Pi_{\calL_i}$ is the orthogonal projection onto the span of the bottom $n-k$ eigenvectors of $A_{-i}$.
It is important to use the projection onto the bottom eigenvectors of $A_{-i}$ in lower bounding the denominator term because we have to use the fact that $\mu_n(A_{-i}^{-1}) \geq \mu_1(A_{-i})^{-1}$.
When we don't do the projection, then $z_i$ is affected by all of $A_{-i}$ and so the largest eigenvalue that affects this expression is $\mu_1(A_{-i}) = \lambda_1$.
After doing this projection, we no longer have contributions from the top $k$ eigenvectors / eigenvalues in the summation of $z_i^T A_{-i}^{-1} z_i$. 
Therefore, the largest eigenvalue that affects this summation is now $\lambda_{k+1}$ instead of $\lambda_1$, and so we can use this in our lower bound instead, as desired.

Putting it together, for $i \leq k$,
\begin{align*}
    \frac{\lambda_i^2 z_i^T A_{-i}^{-2} z_i}{(1 + \lambda_i z_i^T A_{-i}^{-1}z_i)^2} &\leq \frac{z_i^T A_{-i}^{-2}z_i}{(z_i^T A_{-i}^{-1}z_i)^2} \\
    &\leq c_1^4 \frac{\snorm{z_i}^2}{\snorm{\Pi_{\calL_i} z_i}^4}.
\end{align*}

We now invoke Corollary \ref{cor:z_projz_corollary} with a union bound over $k$ events.
Let $t < n/c_0$ and $k \in (0, n/c)$ for $c > c_0$ and $c_0$ sufficiently large.
Since $k < n/c$ we also satisfy the union bound condition that $\ln(k) < n/c$.
Then, with probability at least $1 - 3e^{-t}$,
\begin{align*}
    \snorm{z_i}^2 &\leq c_2 n \\
    \snorm{\Pi_{\calL_i} z_i}^2 &\geq n/c_3
\end{align*}
for constants $c_2, c_3$ that only depend on $\sigma_x, c_0$, and a universal constant $c$ as defined in Corollary \ref{cor:bartlett_cor24}.

Altogether, with probability $1 - 5e^{-n/c_0}$ for $c_0$ sufficiently large,
\begin{align*}
    \sum_{i=1}^k \frac{\Tilde{\lambda_i}}{\lambda_i} \Bigg( \frac{\lambda_i^2 z_i^T A_{-i}^{-2} z_i}{(1 + \lambda_i z_i^T A_{-i}^{-1}z_i)^2} \Bigg) &\leq \sum_{i=1}^k \frac{\Tilde{\lambda_i}}{\lambda_i} c_1^4 \frac{\snorm{z_i}^2}{\snorm{\Pi_{\calL_i} z_i}^4} \\
    &\leq \sum_{i=1}^k \frac{\Tilde{\lambda_i}}{\lambda_i} c_1^4 \frac{c_2 c_3^2}{n} \\
    &= c_4 \sum_{i=1}^k \frac{\Tilde{\lambda_i}}{\lambda_i} \frac{1}{n}.
\end{align*}

On the same event we use to bound $\mu_{k+1}(A_{-i})$ via Lemma \ref{lemma:bartlett_lemma10}, we also have that $\mu_1(A^{-2}) \leq \mu_n(A)^{-2}$.
As such,
\begin{align*}
    \sum_{i > k} \frac{\Tilde{\lambda_i}}{\lambda_i} (\lambda_i^2 z_i^T A^{-2} z_i) \leq \frac{c_1^2 \sum_{i > k} \frac{\Tilde{\lambda_i}}{\lambda_i} \lambda_i^2 \snorm{z_i}^2}{(n\lambda_{k+1}\rho_k)^2}.
\end{align*}

Then by Corollary \ref{cor:bernstein_any_seq}, there is a universal constant $a$ such that with probability at least $1 - 2e^{-t}$ for $t < n/c_0$ and $c_0 > a^{-1}$,
\begin{align*}
    \sum_{i > k} \frac{\Tilde{\lambda}_i}{\lambda_i} \lambda_i^2 \snorm{z_i}^2 &\leq \sigma_x^2 \max(t \max_{i > k}(\Tilde{\lambda}_{i}\lambda_{i}), \sqrt{t\sum_{i > k} (\Tilde{\lambda_i}\lambda_i)^2}) \\
    &\leq n\sum_{i > k} \Tilde{\lambda}_i \lambda_i + \sigma_x^2 \max(t \max_{i > k}(\Tilde{\lambda}_{i}\lambda_{i}), \sqrt{tn\sum_{i > k} (\Tilde{\lambda_i}\lambda_i)^2}) \\
    &\leq n\sum_{i > k} \Tilde{\lambda}_i \lambda_i + \sigma_x^2 \max(t\sum_{i > k} \Tilde{\lambda}_i \lambda_i, \sqrt{tn}\sum_{i > k} \Tilde{\lambda}_i \lambda_i) \\
    &\leq c_5 n \sum_{i > k} \Tilde{\lambda}_i \lambda_i \\
    &= c_5 n \sum_{i > k} \frac{\Tilde{\lambda}_i}{\lambda_i} \lambda_i^2.
\end{align*}

Altogether,
\begin{align*}
    \sum_{i > k} \frac{\Tilde{\lambda_i}}{\lambda_i} (\lambda_i^2 z_i^T A^{-2} z_i) &\leq \frac{c_1^2 \sum_{i > k} \frac{\Tilde{\lambda_i}}{\lambda_i} \lambda_i^2 \snorm{z_i}^2}{(n\lambda_{k+1}\rho_k)^2} \\
    &\leq \frac{c_1^2c_5 n}{(n\lambda_{k+1}\rho_k)^2} \sum_{i > k} \frac{\Tilde{\lambda}_i}{\lambda_i} \lambda_i^2 \\
    &= c_6 \sum_{i > k} \frac{\Tilde{\lambda}_i}{\lambda_i} \l( \frac{\lambda_i^2}{n\lambda_{k+1}^2 \rho_k^2} \r).
\end{align*}

By taking $c > \max(c_0, c_4, c_6)$ we have that with probability $1 - 7e^{-n/c}$,
\begin{align*}
    V &\leq c\l( \sum_{i=1}^k \frac{\Tilde{\lambda}_i}{\lambda_i} \frac{1}{n} + \sum_{i > k} \frac{\Tilde{\lambda}_i}{\lambda_i} \l( \frac{\lambda_i^2}{n\lambda_{k+1}^2 \rho_k^2} \r) \r) \\
    &=  \frac{1}{n}\sum_{i=1}^k \frac{\ltarget_i}{\lsource_i} + n \frac{\sum_{i > k} \ltarget_i \lsource_i}{(\sum_{i > k} \lsource_i)^2}.
\end{align*}

\subsection{Lower bound}
\label{apdx:variance_lower}

Recall that the variance takes the form,
\begin{align*}
    V = \sum_{i=1}^p \Tilde{\lambda_i} \lambda_i \frac{z_i^T A_{-i}^{-2}z_i}{(1 + \lambda_i z_i^T A_{-i}^{-1}z_i)^2}.
\end{align*}

\noindent By Cauchy-Schwartz,
\begin{align*}
    (z_i^T A_{-i}^{-1} z_i)^2 = |\langle z_i, A_{-i}^{-1}z_i \rangle|^2 \leq \snorm{z_i}^2 \cdot (z_i^T A_{-i}^{-2}z_i).
\end{align*}
We plug this identity into our lower bound, and further multiply by $\frac{\lambda_i}{\lambda_i}$, resulting in 
\begin{align*}
    V &= \sum_{i=1}^p \Tilde{\lambda_i} \lambda_i \frac{z_i^T A_{-i}^{-2}z_i}{(1 + \lambda_i z_i^T A_{-i}^{-1}z_i)^2} \\
    &= \sum_{i=1}^p (\frac{\Tilde{\lambda_i}}{\lambda_i}) \frac{\lambda_i^2 z_i^T A_{-i}^{-2}z_i}{(1 + \lambda_i z_i^T A_{-i}^{-1}z_i)^2} \\
    &\geq \sum_{i=1}^p (\frac{\Tilde{\lambda_i}}{\lambda_i}) \frac{\lambda_i^2 (z_i^T A_{-i}^{-1}z_i)^2}{||z_i||^2(1 + \lambda_i z_i^T A_{-i}^{-1}z_i)^2} \\
    &= \sum_{i=1}^p (\frac{\Tilde{\lambda_i}}{\lambda_i}) \frac{1}{||z_i||^2(1 + \lambda_i z_i^T A_{-i}^{-1}z_i)^2(\lambda_i z_i^T A_{-i}^{-1}z_i)^{-2}} \\
    &= \sum_{i=1}^p (\frac{\Tilde{\lambda_i}}{\lambda_i}) \frac{1}{||z_i||^2(1 + (\lambda_i z_i^T A_{-i}^{-1}z_i)^{-1})^2}.
\end{align*}

Then, let $k \in (0, n)$ and $\mathscr{L}_i$ be the span of the bottom $n-k$ eigenvectors of $A_{-i}$ and $\Pi_{\mathscr{L}_i}$ be the projection onto the orthogonal complement of $\mathscr{L}_i$. 
We have that
\begin{align*}
    z_i^T A_{-i}^{-1}z_i &\geq (\Pi_{\mathscr{L}_i} z_i)^T A_{-i}^{-1} (\Pi_{\mathscr{L}_i} z_i) \\
    &\geq \snorm{\Pi_{\mathscr{L}_i} z_i}^2 \mu_{k+1}(A_{-i})^{-1}.
\end{align*}

From Lemma \ref{lemma:bartlett_lemma10}, there is a constant $c_1 \geq 1$, such that for any $k \geq 0$, with probability $1 - 2e^{-n/c_1}$, $\mu_{k+1}(A_{-i}) \leq c_1(\sum_{j > k} \lambda_j + \lambda_{k+1}n)$.
Additionally, by Corollary \ref{cor:z_projz_corollary}, let $t < n/c_3$ and $k \in (0, n/c)$ for $c > c_3$ and $c_3$ sufficiently large.
Then, with probability at least $1 - 3e^{-t}$
\begin{align*}
    \snorm{\Pi_{\mathscr{L}_i} z_i}^2 \geq n/c_4
\end{align*}
where $c_4$ only depends on $c_3, \sigma_x$ and the universal constant given in Corollary \ref{cor:bartlett_cor24}.

Then, for $c \geq \max\{c_1, c_3\}$, with probability $1 - 5e^{-n/c}$,
\begin{align*}
    z_i^T A_{-i}^{-1}z_i &\geq \snorm{\Pi_{\mathscr{L}_i} z_i}^2 \mu_{k+1}(A_{-i})^{-1} \\
    &\geq \frac{n}{c_4(\sum_{j > k} \lambda_j + \lambda_{k+1}n)}.
\end{align*}

By again applying Corollary \ref{cor:z_projz_corollary} on the same event we have
\begin{align*}
    \snorm{z_i}^2 &\leq c_5 n.
\end{align*}
where $c_5$ has the same dependencies as $c_4$.

Altogether, we have for each $i$, with probability $1 - 5e^{-n/c}$,
\begin{align*}
    \frac{1}{||z_i||^2(1 + (\lambda_i z_i^T A_{-i}^{-1}z_i)^{-1})^2} &\geq \frac{1}{c_5 n (1 + (\frac{c_4(\sum_{j > k} \lambda_j + \lambda_{k+1}n)}{\lambda_i n}))^2} \\
    &= \frac{1}{c_5 n(1 + \frac{c_4\lambda_{k+1}}{\lambda_i}(\frac{\sum_{j > k} \lambda_j}{\lambda_{k+1}n} + 1))^2} \\
    &= \frac{1}{c_5 c_4^2 n(1/c_4 + \frac{\lambda_{k+1}}{\lambda_i}(\rho_k + 1))^2} \\
    &\geq \frac{1}{c_6 n(1 + \frac{\lambda_{k+1}}{\lambda_i}(\rho_k + 1))^2}
\end{align*}
where $c_6 = c_5 c_4^2$ and $c > \max\{c_1, c_3\}$ as defined above.

Finally, we invoke Lemma \ref{lem:bartlett_lemma15} and that $1 / (a + b)^2 \geq \min(a^{-2}, b^{-2})/4$ to get that,
with probability $1 - 10e^{-n/c}$,
\begin{align*}
    V \geq \frac{1}{8 c_6 n}\sum_{i=1}^p \frac{\Tilde{\lambda}_i}{\lambda_i} \min (1, \frac{\lambda_i^2}{\lambda_{k+1}^2(\rho_k + 1)^2}).
\end{align*}

For $c_7 \geq \max \{ 8c_6,  c \}$ we have that with probability $1 - 10e^{-n/c_7}$,
\begin{align*}
    V \geq \frac{1}{c_7 n}\sum_{i=1}^p \frac{\Tilde{\lambda}_i}{\lambda_i} \min (1, \frac{\lambda_i^2}{\lambda_{k+1}^2(\rho_k + 1)^2}).
\end{align*} \section{Proof of bias bounds}
\label{apdx:bias_bound_proof}

\biasublb*

\subsection{Upper bound}
\label{apdx:bias_ub_proof}

\begin{proof}
    As defined in Eqn. \ref{eqn:bias_b2}, 
    \begin{align*}
        \bias &= \snorm{\tsource - \mnisource(\xsource \tsource)}_{\covtarget}^2 \\
        &= \snorm{\tsource - X^TA^{-1}X\tsource}_{\covtarget}^2\\
        &= \tsource^T(I - X^TA^{-1}X)\covtarget(I - X^TA^{-1}X)\tsource. \numberthis \label{eqn:bias_expansion}
    \end{align*}

The $i^{\text{th}}$ row of $I_p - X^TA^{-1}X$ is given by $e_i - \sqrt{\lambda_i} z_i^T A^{-1}X$. It follows that 
\begin{align*}
    (\theta^*)^T M \theta^* & = \begin{pmatrix}
        \vdots \\
        \sum_{j=1}^p\theta_j(e_i[j]-\sqrt{\lambda_i\lambda_j}z_i^\top A^{-1}z_j) \\
        \vdots
    \end{pmatrix}^\top \covtarget \begin{pmatrix}
        \vdots \\
        \sum_{j=1}^p\theta_j(e_i[j]-\sqrt{\lambda_i\lambda_j}z_i^\top A^{-1}z_j) \\
        \vdots
    \end{pmatrix} \;\;\textit{$i^{\text{th}}$ row shown} \\
    &= \sum_{i=1}^p\Tilde{\lambda}_i \Big(\sum_{j=1}^p\theta_j(e_i[j]-\sqrt{\lambda_i\lambda_j}z_i^\top A^{-1}z_j)\Big)^2 \\
    &\leq \sum_{i=1}^p\Tilde{\lambda}_i \Big(\sum_{j=1}^p\theta^2_j \Big)\sum_{j=1}^p\Big(e_i[j]-\sqrt{\lambda_i\lambda_j}z_i^\top A^{-1}z_j\Big)^2 \\
    &= \snorm{\theta^*}^2\sum_{i=1}^p\Tilde{\lambda}_i \sum_{j=1}^p\Big(e_i[j]-\sqrt{\lambda_i\lambda_j}z_i^\top A^{-1}z_j\Big)^2.
\end{align*}

Next we look at $i^{\text{th}}$ term in the outer sum.
\begin{align*}
    \Tilde{\lambda}_i \sum_{j=1}^p(e_i[j]-\sqrt{\lambda_i\lambda_j}z_i^\top A^{-1}z_j)^2 &= \Tilde{\lambda}_i (1-\lambda_i z_i^\top A^{-1}z_i)^2 + \Tilde{\lambda}_i \sum_{j\neq i}\lambda_i\lambda_j(z_i^\top A^{-1}z_j)^2 \\
    &= \ltarget_i (1-2\lsource_iz_i^TA^{-1}z_i + \lsource_i^2 (z_i^TA^{-1}z_i)^2 + \sum_{j\neq i}\lambda_i\lambda_j(z_i^\top A^{-1}z_j)^2) \\
    &= \ltarget_i (1-2\lsource_iz_i^TA^{-1}z_i  + \sum_{i=1}^p\lambda_i\lambda_j(z_i^\top A^{-1}z_j)^2) \\
    &= \ltarget_i (1-2\lsource_iz_i^TA^{-1}z_i  + \lambda_i z_i^\top A^{-1}\Big(\sum_{i=1}^p\lambda_jz_jz_j^T\Big)A^{-1}z_i) \\
    &= \ltarget_i \l(1-2\lsource_iz_i^TA^{-1}z_i  + \lambda_i z_i^\top A^{-1}AA^{-1}z_i\r) \\
    &= \ltarget_i \l(1-2\lsource_iz_i^TA^{-1}z_i  + \lambda_i z_i^\top A^{-1}z_i\r) \\
    &= \ltarget_i \l(1-\lsource_iz_i^TA^{-1}z_i\r). 
\end{align*}

Using the Sherman-Morrison formula, we get that
\begin{align*}
    1-\lsource_i z_i^T A^{-1} z_i & = 1 - \lsource_i z_i^T \big(A_{-i} + \lsource_i z_i z_i^T\big)^{-1}z_i\\
    & = 1 - \lsource_i z_i^T \Big(A^{-1}_{-i} - \lsource_i A^{-1}_{-i} z_i(1+\lsource_i z_i^T A^{-1}_{-i} z_i)^{-1} z_i^TA^{-1}_{-i}\Big)z_i \\
    & = 1 - \lsource_i z_i^T A^{-1}_{-i}z_i + \frac{(\lsource_i z_i^TA^{-1}_{-i} z_i)^2}{1+\lsource_i z_i^T A^{-1}_{-i} z_i}  \\
    &= \frac{1}{1 + \lsource_i z_i^TA_{-i}^{-1}z_i}.
\end{align*}

    We now provide an upper bound for the remaining term. Let $\Pi_{\calL_i}$ be the orthogonal projection onto the bottom $n-k$ eigenvectors of $A_{-i}.$ By Lemma \ref{lemma:bartlett_lemma10}, there exist constants $b,c_0\geq 1$ such that if $\rho_k \geq b$, then with probability at least $1-2e^{-n/c_0},$
\begin{align*}
    \mu_{k+1}(A_{-i}) \leq c_0 \lambda_{k+1}\rho_kn,
\end{align*}
so we get
\begin{align*}
    1 + \lsource_i z_i^TA_{-i}^{-1}z_i & \geq 1 + \lsource_i (\Pi_{\calL_i}z_i)^TA_{-i}^{-1}(\Pi_{\calL_i}z_i) \\
    & \geq 1 + \frac{\lsource_i\norm{\Pi_{\calL_i}z_i}^2}{c_0\lsource_{k+1}n\rho_k}.
\end{align*}

By Corollary \ref{cor:z_projz_corollary}, there exist constants $c_1$ and $c_2$ with $c_2>c_1$ and $c_1$ sufficiently large such that for $0<k<n/c_2$, we have with probability at least $1-3e^{-n/c_1}$,
\begin{align*}
    \norm{\Pi_{\calL_i}z_i}^2 \geq n/c_3,
\end{align*}
where $c_3$ depends only on $c_1$ and $\sigma$.

Plugging these in gives us with probability at least $1-5e^{-n/c_4}$,
\begin{align*}
    \ltarget_i \l(1-\lsource_iz_i^TA^{-1}z_i\r) & \leq \frac{\ltarget_i}{\big(1 + \frac{c_5^2\lsource_i}{\lsource_{k+1}\rho_k}\big)} \\
    & = \frac{\ltarget_i}{\lsource_i}\frac{\lsource_i}{\big(1 + \frac{c_5^2\lsource_i}{\lsource_{k+1}\rho_k}\big)},
\end{align*}
where $c_4 = \max(c_0,c_1)$ and $c_5 = \min(c_0, c_3).$

Therefore by union bound over the application of Corollary \ref{cor:z_projz_corollary},
\begin{align*}
    B &\leq \snorm{\theta^*}^2\sum_{i=1}^p\frac{\ltarget_i}{\lsource_i}\frac{\lsource_i}{\big(1 + \frac{c_5^2\lsource_i}{\lsource_{k+1}\rho_k}\big)} \\
    &\leq \frac{1}{c_6}\snorm{\theta^*}^2\sum_{i=1}^p\frac{\ltarget_i}{\lsource_i}\frac{\lsource_i}{\big(1 + \frac{\lsource_i}{\lsource_{k+1}\rho_k}\big)},
\end{align*}
where $c_6 = \min(c_5^2, 1)$. Taking $c = \max(c_6^{-1}, c_4)$ gives us the result.

\end{proof}

\subsection{Lower bound}
\label{apdx:bias_lower}
After isolating the contribution of $\frac{\ltarget_i}{\lsource_i}$, many of the components of this proof are as given in \citet{tsigler2023benignjmlr}.
For completeness, we replicate them here.

\begin{proof}

    Assume that the vector $\tsource$ is randomly distributed according to the PriorSigns($\overline{\theta}_{\src}$) assumption. Using Lemma \ref{lem:kend_quad_form}, the bias term can be rewritten as
\begin{align*}
    B &= \E_{\tsource} [B_\tsource] \\
    & = \E_{\tsource} [ \normtarget{\tsource - \mnisource(\xsource\tsource)}^2] \\
    & = \E_{\tsource} [(\tsource)^T(I_p - \xsource^T(\xsource\xsource^T)^{-1}\xsource)\covtarget(I_p - \xsource^T(\xsource\xsource^T)^{-1}\xsource)\tsource] \\
    & = \E_{\tsource} [(\tsource)^T M\tsource] \\
    & = \E_{\tsource} [\tsource]^T M \E_{\tsource} [\tsource] + \tr(M\mathrm{Cov}(\tsource)) \\
    & = \tr(M\mathrm{Cov}(\tsource)).
\end{align*}
where $M = (I_p - \xsource^T(\xsource\xsource^T)^{-1}\xsource)\covtarget(I_p - \xsource^T(\xsource\xsource^T)^{-1}\xsource).$ The last equality follows from the assumption $\E_{\tsource} [(\tsource)]=0.$ The diagonal elements of Cov($\tsource$) are the component-wise variances of $\tsource$, which are given by $(\tsource)_i^2 = (\overline{\theta}_{\src})_i^2$. The off-diagonal elements are 0 since the components of $\tsource$ are independent. As such, we need only consider the diagonal elements of $M$.

Note that the $i^{th}$ row of $I_p - \xsource^T(\xsource\xsource^T)^{-1}\xsource$ is equal to $e_i - \sqrt{\lsource_i}z_i^T(\xsource\xsource^T)^{-1}\xsource$, where $e_i$ is the $i^{th}$ vector of the standard orthonormal basis. It follows that the $i^{th}$ diagonal element of $M$ is given by
\begin{align*}
    M_{ii} &= \sum_{j=1}^p \ltarget_j (e_i[j] - \sqrt{\lsource_i\lsource_j}z_i^T A^{-1}z_j)^2 \\
    &= \ltarget_i(1 - \lsource_i z_i^T A^{-1} z_i)^2 + \sum_{j \neq i} \ltarget_j \lsource_i \lsource_j (z_i^T A^{-1} z_j)^2.
\end{align*}
Hence, we can express the bias term as
\begin{align*}
    B &= \sum_{i=1}^p (\overline{\theta}_{\src})_i^2\big[\ltarget_i(1 - \lsource_i z_i^T A^{-1} z_i)^2 + \sum_{j \neq i} \ltarget_j \lsource_i \lsource_j (z_i^T A^{-1} z_j)^2\big] \\
        & \geq \sum_{i=1}^p\frac{\ltarget_i}{\lsource_i} \lsource_i(\overline{\theta}_{\src})_i^2(1 - \lsource_i z_i^T A^{-1} z_i)^2.
    \end{align*}
We are able to eliminate the second term because it is non-negative. 
Substituting $A = A_{-i} + \lsource_i z_i z_i^T$ and using the Sherman-Morrison identity, we have that $1-\lsource_i z_i^T A^{-1}z_i = \frac{1}{1 + \lsource_i z_i^T A_{-i}^{-1}z_i}$ (see proof of bias upper bound in Appendix \ref{apdx:bias_ub_proof}).
Then,
\begin{align*}
    B &\geq  \sum_{i=1}^p\frac{\ltarget_i}{\lsource_i} \frac{\lsource_i (\overline{\theta}_{\src})_i^2}{(1 + \lsource_iz_i^T A_{-i}^{-1}z_i)^2}
\end{align*}

Let's bound each term in that sum from below with high probability. By Corollary \ref{lem:our_stab_lower_eig}, there exist constants $b,c_0\geq 1$ such that for any $i\geq 0$ with probability at least $1-2e^{-n/c_0}$, if $\rho_k \geq b$, then
\begin{align*}
    \mu_n(A_{-i}) \geq \frac{1}{c_0} n\lsource_{k+1}\rho_k.
\end{align*}

Next,
\begin{align*}
\frac{\lsource_i}{(1 + \lsource_i z_{i}^T A_{-i}^{-1} z_{i})^2} \geq \frac{\lsource_i}{(1 + \lsource_i \mu_n(A_{-i})^{-1} \norm{z_i}^2)^2}.
\end{align*}
By Corollary \ref{cor:z_projz_corollary}, for constants $c_1, c_2$ such that $k < n/c_2$ with $c_2 > c_1$ for sufficiently large $c_1$ with probability at least $ 1 - 3e^{-n/c_1}$ we have $\norm{z_i}^2 \leq c_3 n$, where $c_3$ depends only on $c_1$ and $\sigma.$

We obtain that w.p. at least $1 - 5e^{-n/c_4}$,
\begin{align*}
\frac{\lsource_i \bar{\theta}_i^2}{(1 + \lsource_i z_{i}^T A_{-i}^{-1} z_{i})^2} \geq \frac{\lsource_i \bar{\theta}_i^2}{\left(1 + \frac{c_4^2\lsource_i}{\lambda_{k+1}\rho_k}\right)^2},
\end{align*}
where $c_4=\max(c_0, c_1, c_3).$ All the terms are non-negative so Lemma \ref{lem:bartlett_lemma15} provides a lower bound on their sum.
With probability at least $1-10e^{-n/c_4},$
\begin{align*}
    B &\geq \frac{1}{2}\sum_{i=1}^p \frac{\ltarget_i}{\lsource_i}\frac{\lsource_i\Bar{\theta}_i^2}{(1 + \frac{c_4^2\lsource_i}{\lambda_{k+1}\rho_k})^2} \\
    &\geq \frac{1}{c_5}\sum_{i=1}^p \frac{\ltarget_i}{\lsource_i}\frac{\lsource_i\Bar{\theta}_i^2}{(1 + \frac{\lsource_i}{\lambda_{k+1}\rho_k})^2},
\end{align*}
where $c_5 = 2\max(c_4^2, 1).$

Finally, we notice that on $i > k$ we have $\rho_k \geq b > 1$ and $\lambda_i \leq \lambda_{k+1}$ giving us,
\begin{align*}
    \sum_{i>k} \frac{\ltarget_i}{\lsource_i} \frac{\lsource_i \Bar{\theta}_i^2}{(1 + \frac{\lsource_i}{\lsource_{k+1}\rho_k})^2} &\geq \sum_{i>k} \frac{\ltarget_i}{\lsource_i} \frac{\lsource_i \Bar{\theta}_i^2}{(1 + \frac{\lsource_i}{\lsource{k+1}})^2} \\
    &\geq \sum_{i>k} \frac{\ltarget_i}{\lsource_i} \frac{\lsource_i \Bar{\theta}_i^2}{4} \\
    &= \frac{1}{4}\sum_{i > k} \ltarget_i \Bar{\theta}_i^2.
\end{align*}

Letting $c = 4\max(c_4, c_5)$ gives us the result.

\end{proof} \section{Proof of tightness of bounds}
\label{apdx:tightness_proofs}

\variancebiastight*

\begin{proof}
    We split the proof into the variance proof in Appendix \ref{apdx:tightness_variance} and the bias proof in Appendix \ref{apdx:tightness_bias}.
\end{proof}

\subsection{Variance Proof}
\label{apdx:tightness_variance}

\begin{proof}
Recall that
\begin{align*}
    \underline{V} &= \frac{1}{8 c_6 n} \sum_{i=1}^p \frac{\Tilde{\lambda}_i}{\lambda_i} \min \l( 1, \frac{\lambda_i^2}{\lambda_{k+1}^2 (\rho_k + 1)^2} \r) \\
    \overline{V} &= c\l(\sum_{i=1}^k \frac{\Tilde{\lambda}_i}{\lambda_i} \frac{1}{n} + \sum_{i > k} \frac{\Tilde{\lambda}_i}{\lambda_i} \l( \frac{\lambda_i^2}{n \lambda_{k+1}^2 \rho_k^2} \r)\r).
\end{align*}

Since $k$ is the smallest $\ell$ such that $\rho_\ell \geq b$, it is clear by definition that $\rho_{k-1} < b$.
Then we observe that 
\begin{align*}
    \rho_{k-1} = \frac{1}{n \lambda_k}\sum_{j > k-1} \lambda_j = \frac{\lambda_k + \sum_{j > k} \lambda_j}{n \lambda_{k}} = \frac{\lambda_k + n\lambda_{k+1}\rho_k}{n\lambda_k} < b \\
    \therefore \lambda_{k} + n\lambda_{k+1}\rho_k < nb\lambda_k \Rightarrow \lambda_k > \frac{\lambda_k + n\lambda_{k+1}\rho_k}{nb} > \frac{n\lambda_{k+1}\rho_k}{nb} = \frac{\lambda_{k+1}\rho_k}{b}.
\end{align*}

On $i \leq k$,
\begin{align*}
    \underline{V} : \overline{V} &= \frac{1}{8c_6 n}\sum_{i=1}^k \frac{\Tilde{\lambda_i}}{\lambda_i} \min\Bigg(1, \frac{\lambda_i^2}{\lambda_{k+1}^2(\rho_{k}+1)^2}\Bigg) : \frac{\Tilde{\lambda_i}}{\lambda_i}\frac{c}{n} \\
    &\geq \frac{1}{8c_6 c}\sum_{i=1}^k \min\Bigg(1, \frac{\lambda_i^2}{\lambda_{k+1}^2(\rho_{k}+1)^2}\Bigg) : 1.
\end{align*}
If the min is $1$ then we are okay otherwise, using the identity above and that fact that $\lambda_i \geq \lambda_k$, we have that
\begin{align*}
    \frac{\lambda_i^2}{\lambda_{k+1}^2(\rho_k + 1)^2} &> \frac{(\lambda_{k+1}\rho_k)^2}{b^2\lambda_{k+1}^2(\rho_k+1)^2} \\
    &= \frac{\rho_k^2}{b^2(\rho_k+1)^2}.
\end{align*}

Examining the $\rho_k$ terms:
\begin{align*}
    \frac{\rho_k^2}{(\rho_k + 1)^2} &= \frac{1}{\rho_k^{-2}(\rho_k + 1)^2} \\
    &= \frac{1}{(1 + \rho_k^{-1})^2}.
\end{align*}
As $\rho_k \geq b$ we have that $\rho_k^{-1} \leq b \Rightarrow 1 + \rho_k^{-1} \leq 1 + b \Rightarrow (1 + \rho_k^{-1})^{-2} \geq (1 + b)^{-2}$.

Putting it together we get that
\begin{align*}
    \frac{1}{8c_6 c} \sum_{i=1}^k \frac{\lambda_i^2}{\lambda_{k+1}^2(\rho_k + 1)^2} &\geq \frac{1}{8c_6 c} \sum_{i=1}^k \frac{1}{b^2(1 + b)^2} \\
    &= \frac{k}{8 c_6 c \cdot b^2(1 + b)^2} \\
    &\geq \frac{1}{8 c_6 c \cdot b^2(1 + b)^2}.
\end{align*}

On $i > k$, it is clear that the $\min$ is always given by the second term, as $\lambda_i \leq \lambda_{k+1}$, so we get
\begin{align*}
    \underline{V} : \overline{V} &= \frac{1}{8 c_6 n}\sum_{i> k} \frac{\Tilde{\lambda_i}}{\lambda_i} \min\Bigg(1, \frac{\lambda_i^2}{\lambda_{k+1}^2(\rho_{k}+1)^2}\Bigg) : c\frac{\Tilde{\lambda_i}}{\lambda_i} \frac{\lambda_i^2}{n\lambda_{k+1}^2\rho_k^2} \\
    &= \frac{1}{8c_6 c} \sum_{i > k} \frac{\rho_k^2}{(\rho_k + 1)^2} \\
    &\geq \frac{1}{8c_6 c} \sum_{i > k} \frac{1}{(1 + b)^2} = \frac{1}{8c_6 c} \frac{p-k}{(1 + b)^2} > \frac{1}{8 c_6 c (1 + b)^2}.
\end{align*}
Finally we note that for $b \geq 1$ it is clear that $\min(b^{-2}(1 + b)^{-2}, (1 + b)^{-2}) = b^{-2}(1 + b)^{-2}$.
Therefore, 
\begin{align*}
    \underline{V} : \overline{V} \geq \frac{1}{8 c_6 c} b^{-2} (1 + b)^{-2}.
\end{align*}

By setting $c$ in the upper bound such that $c > 8c_6$, we get
\begin{align*}
    \underline{V} : \overline{V} \geq \frac{1}{c^2} b^{-2} (1 + b)^{-2}.
\end{align*}

\end{proof}

\subsection{Bias proof}
\label{apdx:tightness_bias}
\begin{proof}

We will bound the ratio of the lower and upper bounds by bounding the ratios of the corresponding terms in each sum. Observe that for all $i$, the ratio of the terms is equal to
\begin{align*}
 \frac{(\theta^*_i)^2}{\snorm{\theta^*}^2}\cdot\frac{1}{\left(1 + \frac{\lsource_i}{\lambda_{k+1}\rho_k}\right)}.  
\end{align*}
On $i\leq k,$
\begin{align*}
    & \frac{(\theta^*_i)^2}{\snorm{\theta^*}^2}\cdot\frac{1}{\left(1 + \frac{\lsource_i}{\lambda_{k+1}\rho_k}\right)} \\
    &\geq \min_i \frac{(\theta^*_i)^2}{\snorm{\theta^*}^2}\cdot\frac{1}{\left(1 + \frac{\lsource_1}{\lsource_{k+1}}b^{-1}\right)}.
\end{align*}
On $i>k$, we have $\lsource_i/\lsource_{k+1} \leq 1$, so

\begin{align*}
    & \frac{(\theta^*_i)^2}{\snorm{\theta^*}^2}\cdot\frac{1}{\left(1 + \frac{\lsource_i}{\lambda_{k+1}\rho_k}\right)} \\
    &\geq \min_i \frac{(\theta^*_i)^2}{\snorm{\theta^*}^2}\cdot\frac{1}{\left(1 + b^{-1}\right)}.
\end{align*}
Unfortunately, the looseness in the top $k$ components coming from the gap $\lsource_1 / \lsource_{k+1}$ dominates the tighter ratios in the bottom $p-k$ components which only contain a model-dependent gap, $\min_i \theta_i^2 / \snorm{\theta}^2$.
Future work would seek to resolve this and provide tight upper and lower bounds for the bias terms.

\end{proof} \section{Proof of beneficial and malignant shifts}
\label{apdx:beneficial_malignant_shifts}

\subsection{Trace conditions for simple shifts}
\label{apdx:trace_conditions_simple_shifts}
Let $\covsource$ be any source covariance and define $\covtarget$ as $\ltarget_i = \alpha \lsource_i$ for $i \leq k$ and $\ltarget_i = \beta \lsource_i$ for $i > k$ with $\alpha, \beta \geq 0$.

Then $\tr(\covsource) = \sum_{i=1}^k \lsource_i + \sum_{i>k} \lsource_i$ and $\tr(\covtarget) = \alpha(\sum_{i=1}^k \ltarget_i) + \beta(\sum_{i > k} \ltarget_i)$.

For $\alpha > 1, \beta < 1$, if
\begin{align*}
    \frac{\sum_{i > k} \lsource_i}{\sum_{i=1}^k \lsource_i} < \frac{\alpha - 1}{1 - \beta}
\end{align*}
then we have that $\tr(\covsource) < \tr(\covtarget)$ and if the inequality is flipped then we obtain $\tr(\covsource) > \tr(\covtarget)$.

For $\alpha < 1, \beta > 1$, if
\begin{align*}
    \frac{\sum_{i =1}^k \lsource_i}{\sum_{i>k} \lsource_i} < \frac{\beta - 1}{1 - \alpha}
\end{align*}
then we have that $\tr(\covsource) < \tr(\covtarget)$ and if the inequality is flipped then we obtain $\tr(\covsource) > \tr(\covtarget)$.

\subsection{Proof of beneficial and malignant shifts for simple shifts}
\label{apdx:proof_beneficial_malignant_simple_shifts}

We restate the theorem for ease.
\beneficialmalignantshifts*

\begin{proof}
From Theorem \ref{thm:main_results_variance_bounds} and Theorem \ref{thm:tightness}, we have that for a universal constant $b > 1$ if $\rho_k \geq b$ we get the following upper and lower bounds on the out-of-distribution variance for some constants $c_1$, $c_2$,
\begin{align*}
    V_{ood} &\leq c_1 \l( \frac{1}{n}\sum_{i=1}^k \frac{\ltarget_i}{\lsource_i} + n \frac{\sum_{i > k} \ltarget_i \lsource_i}{(\sum_{i > k} \lsource_i)^2} \r) \\
    V_{ood} &\geq c_2 \l( \frac{1}{n}\sum_{i=1}^k \frac{\ltarget_i}{\lsource_i} + n \frac{\sum_{i > k} \ltarget_i \lsource_i}{(\sum_{i > k} \lsource_i)^2} \r).
\end{align*}

Analogously, the in-distribution variance is upper and lower bounded by,
\begin{align*}
    V_{id} &\leq c_1 \l( \frac{k}{n} + \frac{n}{R_k} \r) \\
    V_{id} &\geq c_2 \l( \frac{k}{n} + \frac{n}{R_k} \r)
\end{align*}
where $R_k = (\sum_{i > k} \lsource_i)^2 / \sum_{i > k} \lsource_i^2$.

Let $\covsource$ be any source covariance model that satisfies benign source conditions.
Define $\covtarget$ by,
\begin{align*}
    \ltarget_i =
    \begin{cases}
        \alpha \lsource_i, \quad i \leq k \\
        \beta \lsource_i, \quad i > k
    \end{cases}
\end{align*}
for $\alpha, \beta \geq 0$.

\paragraph{Beneficial shifts.} We use the upper bound to specify requirements for the beneficial shifts.
\begin{align*}
    V_{ood} &\leq c_1\l( \alpha \frac{k}{n} + \beta \frac{n}{R_k} \r) \\
    &= V_{id} + c_1 \l( \frac{k}{n}(\alpha - 1) + \frac{n}{R_k}(\beta - 1) \r).
\end{align*}
Let $\alpha > 1, \beta < 1$. To obtain a beneficial shift in this setting we need,
\begin{align*}
    \frac{n}{R_k}(1 - \beta) &> \frac{k}{n}(\alpha - 1)\\
    \Rightarrow \frac{n}{R_k} &> \frac{k}{n} \l( \frac{\alpha - 1}{1 - \beta} \r).
\end{align*}

In the case $\alpha < 1, \beta > 1$, to obtain a beneficial shift we need,
\begin{align*}
    \frac{n}{R_k}(\beta - 1) &< \frac{k}{n}(1 - \alpha) \\
    \Rightarrow \frac{n}{R_k} &< \frac{k}{n}\l(\frac{1 - \alpha}{\beta - 1}\r).
\end{align*}

In the case where $\alpha = 1$ then any $\beta < 1$ leads to beneficial shifts. Similarly when $\beta = 1$, any $\alpha < 1$ leads to beneficial shifts.

\paragraph{Malignant shifts.}
We use the lower bound to specify requirements for the malignant shift.
\begin{align*}
    V_{ood} &\geq c_2 \l( \alpha \frac{k}{n} + \beta \frac{n}{R_k} \r) \\
    &= V_{id} + c_2 \l( \frac{k}{n}(\alpha - 1) + \frac{n}{R_k}(\beta - 1) \r).
\end{align*}

Let $\alpha < 1$ and $\beta > 1$.
To obtain a malignant shift in this setting we need,
\begin{align*}
    \frac{n}{R_k} > \frac{k}{n}\l( \frac{1 - \alpha}{\beta - 1}\r).
\end{align*}

In the case of $\alpha > 1, \beta < 1$, to obtain a malignant shift we need,
\begin{align*}
    \frac{n}{R_k} < \frac{k}{n}\l( \frac{\alpha - 1}{1 - \beta}\r).
\end{align*}

In the case where $\alpha = 1$ then any $\beta > 1$ leads to malignant shifts. Similarly when $\beta = 1$, any $\alpha > 1$ leads to malignant shifts.

\paragraph{Mild and severe overparameterization.}
We see that the four cases separate into settings in which we are mildly overparameterized, meaning
\begin{align*}
    \frac{n}{R_k} > \frac{k}{n}\l| \frac{\alpha - 1}{1 - \beta} \r|,
\end{align*}
and settings in which we are severely overparamterized, meaning
\begin{align*}
    \frac{n}{R_k} < \frac{k}{n}\l| \frac{\alpha - 1}{1 - \beta} \r|.
\end{align*}

In each of these regimes of overparameterization, the above proof has delineated whether we achieve beneficial or malignant shifts in all settings of $\alpha, \beta$.
\end{proof}

\subsection{Generalized (necessary) conditions for beneficial and malignant shifts}
\label{apdx:sufficient_arbitrary_shifts}

Let $\covsource$ be any source covariance matrix that satisfies benign source conditions and define $\covtarget$ as
\begin{align*}
    \ltarget_i = 
    \begin{cases}
        \alpha_i \lambda_i \qquad i \leq k,\\
        \beta_i \lambda_i \qquad i > k
    \end{cases}
\end{align*}
with $\alpha_i, \beta_i \geq 0$ for all $i$.

Then the OOD variance upper bound is given by,
\begin{align*}
    V_{ood} &\leq c_1 \l( \frac{1}{n}\sum_{i=1}^k \alpha_i + n \frac{\sum_{i > k} \beta_i \lsource_i^2}{(\sum_{i > k} \lsource_i)^2} \r) \\
    &= V_{id} + c_1 \l( \frac{(\sum_{i=1}^k \alpha_i) - k}{n} + n \frac{\sum_{i > k} \lsource_i^2 (\beta_i - 1)}{(\sum_{i > k} \lsource_i)^2} \r) \\
    &= V_{id} + c_1 \l( \frac{k}{n}\l( \frac{\sum_{i =1}^k \alpha_i}{k} - 1 \r) + \frac{n}{R_k} \l( \frac{\sum_{i > k} \beta_i \lsource_i^2 }{\sum_{i > k} \lsource_i^2}  - 1 \r)  \r),
\end{align*}
and the OOD variance lower bound is given by,
\begin{align*}
    V_{ood} &\geq c_2 \l( \frac{1}{n}\sum_{i=1}^k \alpha_i + n \frac{\sum_{i > k} \beta_i \lsource_i^2}{(\sum_{i > k} \lsource_i)^2} \r) \\
    &= V_{id} + c_2 \l( \frac{(\sum_{i=1}^k \alpha_i) - k}{n} + n \frac{\sum_{i > k} \lsource_i^2(\beta_i - 1)}{(\sum_{i > k} \lsource_i)^2} \r) \\
    &= V_{id} + c_2 \l( \frac{k}{n}\l( \frac{\sum_{i =1}^k \alpha_i}{k} - 1 \r) + \frac{n}{R_k} \l( \frac{\sum_{i > k} \beta_i \lsource_i^2 }{\sum_{i > k} \lsource_i^2}  - 1 \r) \r),
\end{align*}
where $V_{id}$ is the ID variance bound.

Again, we use the upper bounds to prove conditions for beneficial shifts and the lower bounds to prove conditions for malignant shifts.

\paragraph{Beneficial shifts.}
From the upper bound we consider two separate cases for non-trivial beneficial shifts: 
\begin{enumerate}
    \item $\sum_{i=1}^k \alpha_i < k$ and $\sum_{i > k} \beta_i \lsource_i^2 > \sum_{i > k} \lsource_i^2$,
    \item $\sum_{i=1}^k \alpha_i > k$ and $\sum_{i > k} \beta_i \lsource_i^2 < \sum_{i > k} \lsource_i^2$.
\end{enumerate}

We start with the case of $\sum_{i=1}^k \alpha_i < k$ and $\sum_{i > k} \beta_i \lsource_i^2 > \sum_{i > k} \lsource_i^2$.
If this is satisfied, the only way to achieve a beneficial shift is if
\begin{align}
    \label{eqn:general_beneficial_severe}
    \frac{n}{R_k} \l( \frac{\sum_{i > k} \beta_i \lsource_i^2 }{\sum_{i > k} \lsource_i^2}  - 1 \r) < \frac{k}{n}\l( 1 - \frac{\sum_{i =1}^k \alpha_i}{k} \r).
\end{align}
We also have in this setting that,
\begin{align*}
    0 < 1 - \frac{\sum_{i=1}^k \alpha_i}{k} \leq 1.
\end{align*}

In Equation \ref{eqn:general_beneficial_severe} we see a notion of severe overparameterization that leads to beneficial shifts.
For instance as $R_k \to \infty$ we see the left-hand-side (LHS) of Equation \ref{eqn:general_beneficial_severe} $\to 0$.
So as $R_k \to \infty$ we have that finite $n$ always leads to a beneficial shift in this setting.
We note that equivalently if $\beta_i = 1$ for all $i$ then we also have the LHS $\to 0$, just as in the case of severe overparameterization.
We will return to the definitions of mild and severe overparameterization for arbitrary shifts after showing the remaining conditions for beneficial and malignant shifts.

Now consider the case of $\sum_{i=1}^k \alpha_i > k$ and $\sum_{i > k} \beta_i \lsource_i^2 < \sum_{i > k} \lsource_i^2$.
If this is satisfied, the only way to achieve a beneficial shift is if
\begin{align}
    \label{eqn:general_beneficial_mild}
    \frac{n}{R_k} \l( 1 - \frac{\sum_{i > k} \beta_i \lsource_i^2 }{\sum_{i > k} \lsource_i^2} \r) > \frac{k}{n}\l( \frac{\sum_{i =1}^k \alpha_i}{k} - 1 \r).
\end{align}
In this setting it is clear that
\begin{align*}
    0 < 1 - \frac{\sum_{i > k} \beta_i \lsource_i^2 }{\sum_{i > k} \lsource_i^2} \leq 1.
\end{align*}

In Equation \ref{eqn:general_beneficial_mild}, it is clear that we have a notion of mild overparameterization that leads to beneficial shifts.
As above if $\alpha_i = 1$ for all $i$ then we always obtain a beneficial shift in this setting.
Otherwise if $R_k$ does not grow too quickly (as in the case with mild overparameterization) then this is a necessary condition to achieve beneficial shifts when $\sum_{i > k} \beta_i \lsource_i^2 < \sum_{i > k} \lsource_i^2$.

\paragraph{Malignant shifts.}

From the lower bound we once again consider two separate cases for non-trivial malignant shifts:
\begin{enumerate}
    \item $\sum_{i=1}^k \alpha_i > k$ and $\sum_{i > k} \beta_i \lsource_i^2 < \sum_{i > k} \lsource_i^2$,
    \item $\sum_{i=1}^k \alpha_i < k$ and $\sum_{i > k} \beta_i \lsource_i^2 > \sum_{i > k} \lsource_i^2$.
\end{enumerate}

We start with the case of $\sum_{i=1}^k \alpha_i > k$ and $\sum_{i > k} \beta_i \lsource_i^2 < \sum_{i > k} \lsource_i^2$.
If this is satisfied then the only way to achieve a malignant shift is if,
\begin{align}
    \label{eqn:general_malignant_severe}
    \frac{n}{R_k} \l( 1 - \frac{\sum_{i > k} \beta_i \lsource_i^2 }{\sum_{i > k} \lsource_i^2} \r) < \frac{k}{n}\l( \frac{\sum_{i =1}^k \alpha_i}{k} - 1 \r).
\end{align}

In the case of $\sum_{i=1}^k \alpha_i < k$ and $\sum_{i > k} \beta_i \lsource_i^2 > \sum_{i > k} \lsource_i^2$ the only way to achieve a malignant shift is if,
\begin{align}
    \label{eqn:general_malignant_mild}
    \frac{n}{R_k} \l( \frac{\sum_{i > k} \beta_i \lsource_i^2 }{\sum_{i > k} \lsource_i^2} - 1 \r) > \frac{k}{n}\l( 1 - \frac{\sum_{i =1}^k \alpha_i}{k}  \r).
\end{align}

We now are ready to define mild and severe overparameterization for arbitrary multiplicative shifts.
\begin{theorem}(Mild and severe overparameterization for arbitrary multiplicative shifts)
    Let $\covsource$ be any source covariance matrix that satisfies benign source conditions, meaning $\exists \ k$ such that $\rho_k \geq b$ for a universal constant $b > 1$.
    Furthermore, let $\covtarget$ be defined by $\ltarget_i = \alpha_i \lsource_i$ for $i \leq k$ and $\ltarget_i = \beta_i \lsource_i$ for $i > k$.

    We will define
    \begin{align*}
        C := \l| \l(\frac{\sum_{i =1}^k \alpha_i}{k} - 1 \r)  \l(1 - \frac{\sum_{i > k} \beta_i \lsource_i^2 }{\sum_{i > k} \lsource_i^2} \r)^{-1} \r|.
    \end{align*}

    Then we are \textit{mildly overparameterized} if
    \begin{align*}
        \frac{n}{R_k} = \omega \l( C \frac{k}{n} \r)
    \end{align*}
    and we are \textit{severely overparameterized} if
    \begin{align*}
        \frac{n}{R_k} = o \l( C \frac{k}{n} \r).
    \end{align*}
\end{theorem}

We now state our taxonomy of covariate shifts for arbitrary multiplicative shifts.
\begin{theorem}(Beneficial and Malignant (Arbitrary) Multiplicative Shifts on Variance)
    Let $\covsource$ be any source covariance matrix that satisfies benign source conditions, meaning $\exists \ k$ such that $\rho_k \geq b$ for a universal constant $b > 1$.
    Furthermore, let $\covtarget$ be defined by $\ltarget_i = \alpha_i \lsource_i$ for $i \leq k$ and $\ltarget_i = \beta_i \lsource_i$ for $i > k$.

    \begin{enumerate}
        \item If $\sum_{i=1}^k \alpha_i \leq k$ and $\sum_{i > k} \beta_i \lsource_i^2 < \sum_{i > k} \lsource_i^2$ then we obtain a beneficial shift.
        \item If $\sum_{i=1}^k \alpha_i < k$ and $\sum_{i > k} \beta_i \lsource_i^2 \leq \sum_{i > k} \lsource_i^2$ then we obtain a beneficial shift.
        \item If $\sum_{i=1}^k \alpha_i \geq k$ and $\sum_{i > k} \beta_i \lsource_i^2 > \sum_{i > k} \lsource_i^2$ then we obtain a malignant shift.
        \item If $\sum_{i=1}^k \alpha_i > k$ and $\sum_{i > k} \beta_i \lsource_i^2 \geq \sum_{i > k} \lsource_i^2$ then we obtain a malignant shift.
        \item If we are in the mildly overparameterized regime:
        \begin{itemize}
            \item $\sum_{i=1}^k \alpha_i > k$ and $\sum_{i > k} \beta_i \lsource_i^2 < \sum_{i > k} \lsource_i^2$ leads to beneficial shifts,
            \item $\sum_{i=1}^k \alpha_i < k$ and $\sum_{i > k} \beta_i \lsource_i^2 > \sum_{i > k} \lsource_i^2$ leads to malignant shifts.
        \end{itemize}
        \item If we are in the severely overparameterized regime:
        \begin{itemize}
            \item $\sum_{i=1}^k \alpha_i < k$ and  $\sum_{i > k} \beta_i \lsource_i^2 > \sum_{i > k} \lsource_i^2$ leads to beneficial shifts,
            \item $\sum_{i=1}^k \alpha_i > k$ and $\sum_{i > k} \beta_i \lsource_i^2 < \sum_{i > k} \lsource_i^2$ leads to malignant shifts.
        \end{itemize}
    \end{enumerate}

\end{theorem}

 \section{Experiment details}
\label{apdx:experiment_details}

\subsection{Synthetic data experiments}
Our synthetic data experiments use source data generated from random Gaussians with covariance structures that are known to exhibit benign overfitting.
These structures include the $(k, \delta, \epsilon)$ spiked covariance models and eigendecay rates given by \citet{bartlett2020benignpnas} such as $\lsource_i = i^{-\alpha}\ln^{-\beta}(i+1)$ for $\alpha = 1, \beta > 1$.
Target data is generated from random Gaussians with covariances that lead to beneficial and malignant shifts based on our theories and modifications of the aforementioned source covariance structures.

All ground truth models are sampled uniformly on the $p$-dimensional hypersphere, as $\tsource \sim \calS^{p-1}$.
Label noise is sampled as $\varepsilon \sim \calN(0, 1)$, unless otherwise specified.
For a data matrix $X \in \R^{n \times p}$, training labels are obtained as $y = X\tsource + \varepsilon$.
Excess risk is computed for unseen testing data from source and target distributions of interest using clean labels.

In Figure \ref{fig:k_eps_beneficial_malignant} we take the source to be the $(k, \delta, \epsilon)$ spiked model with parameters given by $k=70$, $\delta=1$, and $\epsilon=0.005$. The beneficial shift scales the first $k$ eigenvalues by $\alpha=1.125$ and the last $p-k$ eigenvalues by $\beta=0.65$. For the malignant shift we use $\alpha=0.875$ and $\beta=1.35$. The minimum-norm linear interpolator is fit to 500 data points sampled from a centered multivariate Gaussian with unit variance and dimension $p=4900$. The model vector is sampled from a centered Gaussian and scaled to unit norm. The x-axis represents the amount of additive label noise in training. All evaluation is done on clean data. Each point is the average of 40 runs.

In Figure \ref{fig:k_eps_extreme_overparam}, we take the source to be the $(k, \delta, \epsilon)$ spiked model with source parameters as $k = 10, \delta = 1.0, \epsilon = 1e^{-6}$ and target parameters $\Tilde{k} = 10, \Tilde{\delta} = 1.35, \Tilde{\epsilon} = 6.5e^{-7}$. We use $n = 50$ training data points, $10k$ held-out testing data points in each OOD test set, and vary $p$ from 75 to 1000 dimensions.
We solve OLS using the closed-form MNI solution on the source data.
Each experiment is averaged over 100 independent runs.

In Figure \ref{fig:mlp_synth_regr} we train fully-connected neural networks with ReLU activation functions.
Data is sampled as above from the covariance structures given by $\lsource_i = i^{-\alpha}\ln^{-\beta}(i+1)$ with varying $\beta$ to obtain beneficial and malignant shifts.
The network architecture is $3$ hidden layers, with hidden widths $512$ and $2048$.
Networks are trained with stochastic gradient descent with momentum $0.9$ until the training MSE has reached $< 5e^{-6}$.
We start with a learning rate of 0.01 and decay by a stepped cosine schedule for 1,500 epochs.
We take batch size of 64 and train without weight decay.
Each experiment is averaged over 20 independent runs.
We train in PyTorch with a single A100 NVIDIA GPU.
In these experiments we take $n=200$ and compare $p=20$ with $p=2000$.
Label noise is sampled as $\calN(0, \sigma^2)$ and we vary $\sigma^2$ to show the behaviors at varying train label noise.

In Figure \ref{fig:mlp_log_plaw_k10_synthetic}  we train full-connected neural networks with ReLU activation functions.
Source data is sampled from a mean-centered Gaussian with diagonal covariance matrix with eigenvalues $\lambda_i = i^{-1}\ln^{-1.5}(i+1)$. 
Target covariate shifts are implemented in the style of Theorem \ref{thm:benficial_malignant_shifts} where the top $k$ source eigenvalues are multiplied by $\alpha$ and the bottom $p-k$ source eigenvalues are multiplied by $\beta$.
In this experiment, we take $k=10, \alpha=2, \beta=0.1$ and experiment with $n=400$ source data samples for $p=200$ and $p=4,000$.
The network architecture is 3 hidden layers with hidden width $2,048$. 
Our training setup is the same as given above for prior MLP experiments.

\subsection{CIFAR-10 and CIFAR-10C experiments}

In Figures \ref{fig:cf10c_mni} and \ref{fig:cf10c_mni_overparam} we use a binary variant of CIFAR-10 and CIFAR-10C.
For details on the CIFAR-10C dataset, see \citet{HendrycksDietterich2019}.
The binary problem is constructed by selecting only the dog and truck classes.
To stay overparameterized, we subsample $n=500, 1000, 2000$ points in a class-balanced manner.
Images are flattened into $p=3072$ dimensional vectors.
We fit our model using the OLS solution for the MNI against $\{0, 1\}$ class labels.
We test on the same two classes from CIFAR-10 and CIFAR-10C Gaussian blur and Gaussian noise corruptions.
Recall that these two corruptions were selected for their eigenspectra's similarity to beneficial and malignant shifts, respectively.
Label noise is injected by flipping class labels with a given probability.

In Figure \ref{fig:cf10_resnets}, we train ResNet18 models on the entire CIFAR-10 dataset and evaluate on the CIFAR-10C test sets for the Gaussian blur and Gaussian noise corruptions. 
The setting is not high-dimensional because we train on 50000 images with 3072 dimensions. 
However, the ResNet18 architecture has around 11.7 million parameters, so the level of overparameterization is very high. 
The training procedure is similar to that used for our MLP experiments. 
Networks are trained with stochastic gradient descent with a learning rate of 0.1 and stepped cosine decay schedule for 60 epochs.
Each point in the plot is an  average over 30 independent runs.
As before, we train in PyTorch with a single A100 NVIDIA GPU.
Label noise takes the form of random label flips with probabilities 0.1 to 0.9. \section{Additional experiments}

We present a number of additional supporting experiments that show: (1) more cases of the behavior of the MNI and MLPs under covariate shift on synthetic datasets; (2) underparameterized and overparameterized regimes for linear regression under covariate shift for more realistic eigendecay rates outside of $(k, \delta, \epsilon)$ spiked covariance models; (3) cases in which the MNI is overfit in a \textit{tempered} or \textit{catastrophic} manner and evaluated on OOD datasets constructed based on our results in Theorem \ref{thm:benficial_malignant_shifts}, indicating that our insights hold up for the MNI even when benign source conditions are not satisfied; (4) the value of overparameterization for the MNI trained on CIFAR-10 and evaluated on CIFAR-10C; (5) experiments training ResNet-18 models to interpolation on the full CIFAR-10 dataset and evaluated on CIFAR-10C blur and noise corruptions.

\subsection{MNI on Synthetic Data}
\label{apdx:additional_exps_synthetic_mni}

\begin{figure}[H]
    \centering
    \includegraphics[width=0.7\linewidth]{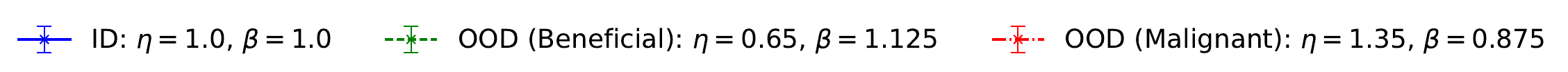}
    \includegraphics[width=0.32\linewidth]{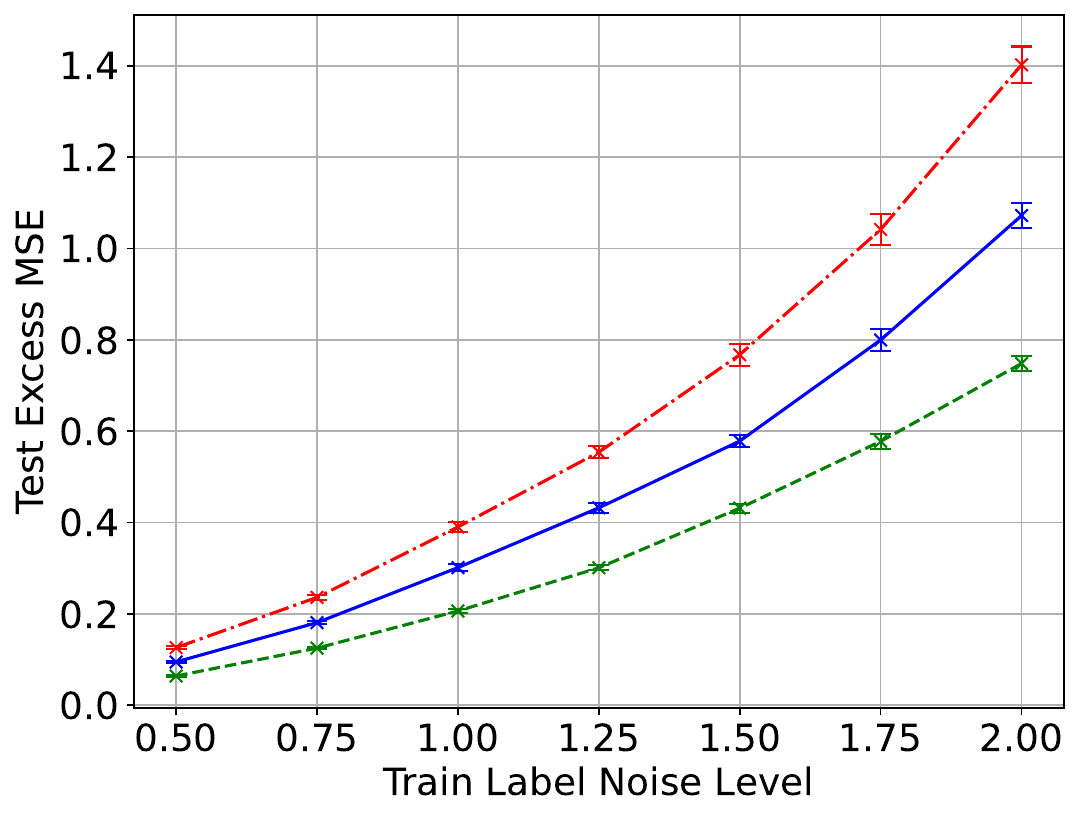}
    \includegraphics[width=0.32\linewidth]{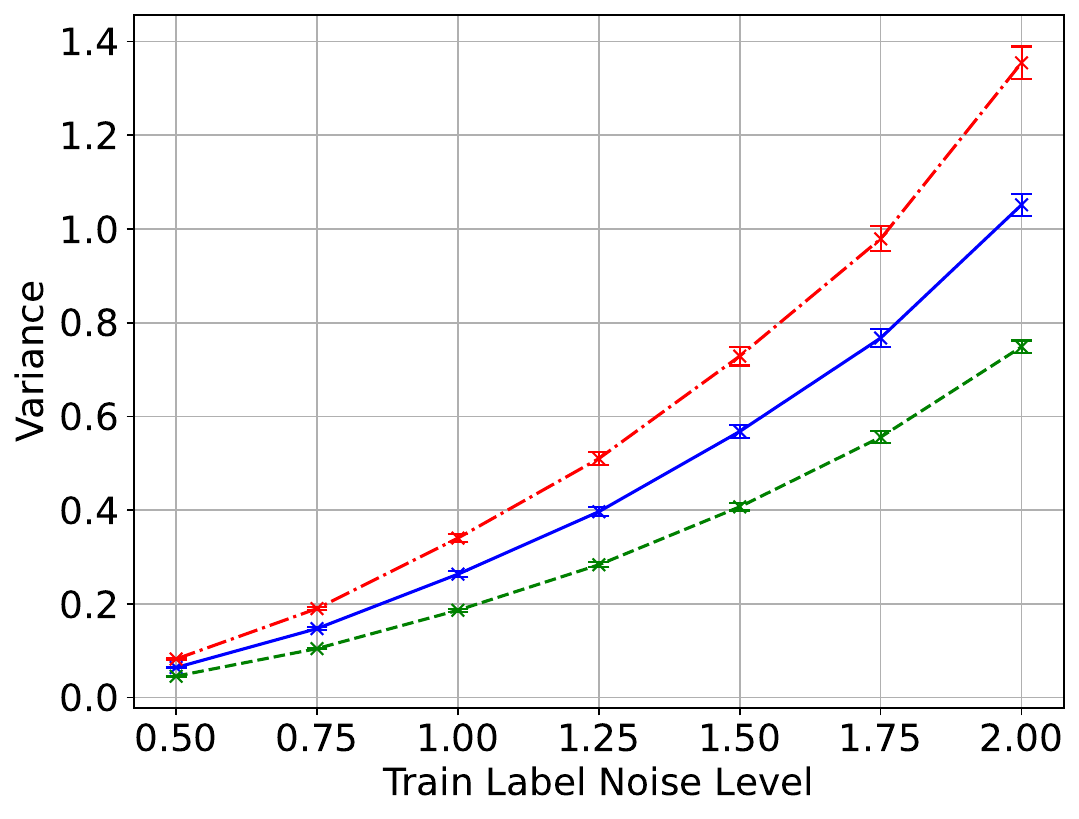}
    \includegraphics[width=0.32\linewidth]{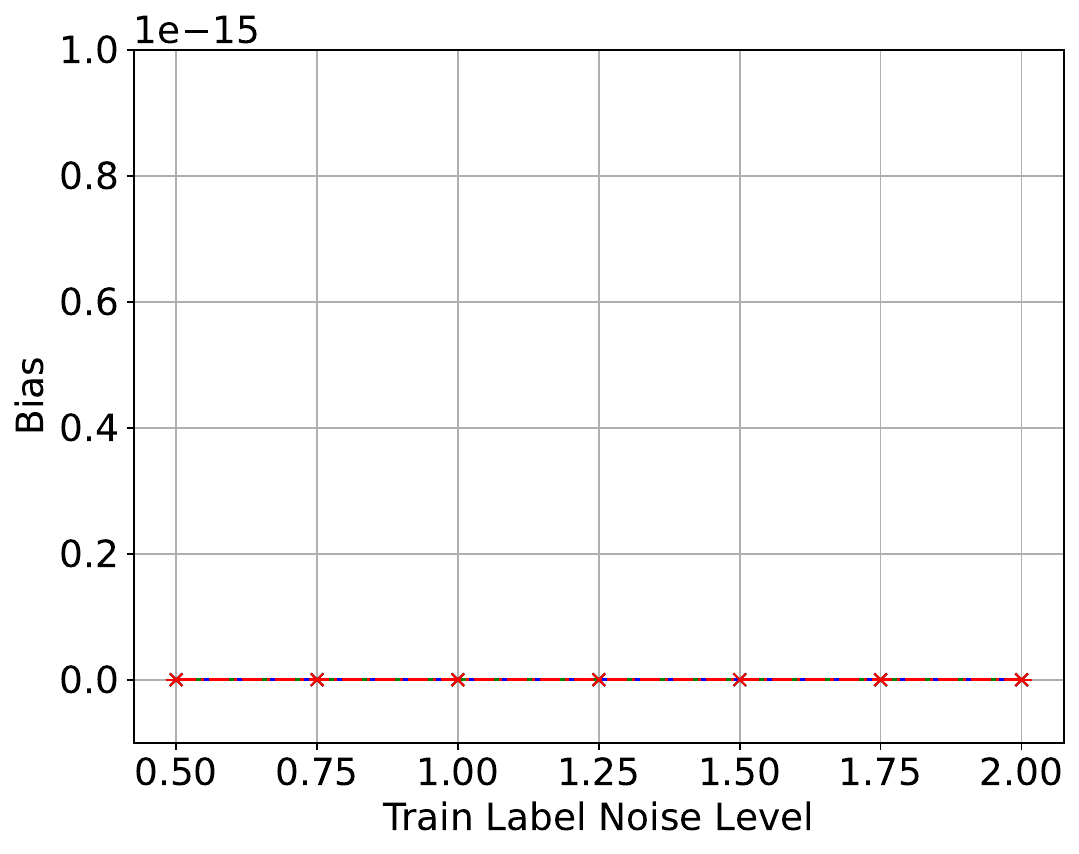}
    \caption{We fit interpolating linear models to random Gaussian data sampled from spiked covariance models with parameters $k, \delta, \epsilon$. In this setting, $k=70, n=500$, $p=4900$, $\delta=1$ and $\epsilon=0.005.$ To illustrate a beneficial shift, we scale the first $k$ eigenvalues by $\alpha=1.125$ and the last $p-k$ eigenvalues by $\beta=0.65$. Similarly, for the malignant shift we use $\alpha=0.875$ and $\beta=1.35$. All experiments are averaged over 25 independent runs with standard error bars displayed. Note that the bias is consistently below $10^{-16}$.}
    \label{fig:k_eps_beneficial_malignant}
\end{figure}

In Figure \ref{fig:k_eps_beneficial_malignant}, we experiment with interpolating linear models where $\covsource, \covtarget$ are given by $(k, \delta, \epsilon)$-spike covariances with $k=70$, $n=500$, and $p=4900$.
We design problem parameters to show settings in which $\tr(\covtarget) > \tr(\covsource)$ and we get a beneficial shift, and $\tr(\covtarget) < \tr(\covsource)$ and we get a malignant shift.
To do this, the source covariance matrix is constructed using $\delta=1$ and $\epsilon=0.005.$ To illustrate a beneficial shift, we scale the first $k$ eigenvalues by $\alpha=1.125$ and the last $p-k$ eigenvalues by $\beta=0.65$. 
Similarly, for the malignant shift we use $\alpha=0.875$ and $\beta=1.35$. 
The resulting plots are significant because they highlight the distinct effects that the first $k$ and last $p-k$ components have on the excess risk.

As illustrated by our main theorems, increasing the energy of an eigenvalue has a negative impact on the risk. 
Nonetheless, these plots show that where the increase happens plays an important role on how the shift affects generalization. 
We are able to improve performance by decreasing the energy on the tail and increasing the energy on the head in such a way that the total energy is increased. 
In short, this setting is a direct connection to our theory and shows clearly that our constructions for beneficial and malignant shifts, when mildly overparameterized, hold up in low and high train label noise regimes, with higher noise exacerbating the effects of the shifts. In addition, Figure \ref{fig:k_eps_beneficial_malignant} demonstrates that the variance generally contributes much more significantly to the overall risk.

\begin{figure}[H]
\centering
\subfigure{\includegraphics[width=.35\textwidth]{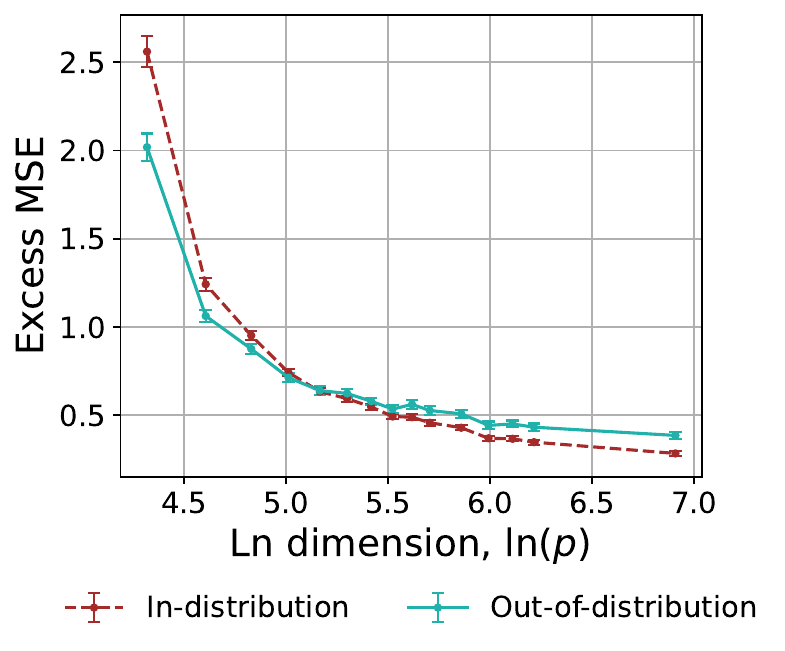}}

\caption{We experiment with the $(k, \delta, \epsilon)$ spiked covariance models and examine conditions for beneficial and malignant shifts as given in Theorem \ref{thm:benficial_malignant_shifts}. We take $n=50, k=10, \delta=1.0, \epsilon=1e^{-6}, \Tilde{\delta}=1.5, \Tilde{\epsilon} = 5e^{-7}$, and vary $p$. In all cases, $\tr(\covtarget) > \tr(\covsource)$, showing that beneficial shifts of this form can occur. As we increase $p$ while keeping other problem parameters fixed we observe the transition from mild to severe overparameterization and see the cross-over point between the shift going from beneficial to malignant. For both ID and OOD excess risk, we observe that excess risk is a decreasing function of input dimension. Curves are averaged over 100 independent runs.}
\label{fig:k_eps_extreme_overparam}
\end{figure}

Figure \ref{fig:k_eps_extreme_overparam} shows another example of the transition from mild overparameterization to severe overparameterization in the case of $(k, \delta, \epsilon)$ spiked covariance models. In this example we take $k = 10, n = 50, \delta=1.0, \epsilon = 1e^{-6}$.
Using our shifts defined in Theorem \ref{thm:benficial_malignant_shifts} we set $\alpha = 1.5$ and $\beta = 0.5$.
We plot excess MSE on both ID and OOD test sets vs. the input data dimension, while holding all other problem parameters fixed and clearly observe the transition from beneficial to malignant shifts in keeping with our theorem.

Next, we experimentally show that while our theory is built for benign source covariance structures it holds for non-benign covariances.
In particular, we examine eigendecay rates that are known to lead to \textit{tempered} overfitting and \textit{catastrophic} overfitting \citep{Mallinar2022}. \citet{bartlett2020benignpnas} identify the covariance structure given by $\lambda_i = i^{-1}\textrm{ln}^{-2}(i+1)$ as sufficient for benign overfitting.
The rate of $i^{-\alpha}$ for $\alpha > 1$ is akin to a ridgeless Laplace kernel and corresponds to tempered overfitting. Finally, the rate of $i^{-\ln(i)}$ is akin to a ridgeless Gaussian kernel and corresponds to catastrophic overfitting. This relative ordering is determined by how high-dimensional the tail eigenvalues are, in decreasing order.

\begin{figure}[H]
\centering
\includegraphics[width=0.7\linewidth]{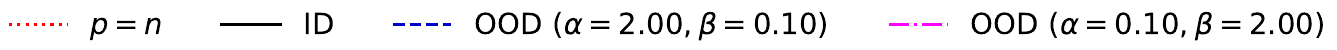}
\subfigure[Ridgeless, $\lambda_i = i^{-1}\ln^{-2}(i+1)$]{\includegraphics[width=.32\textwidth]{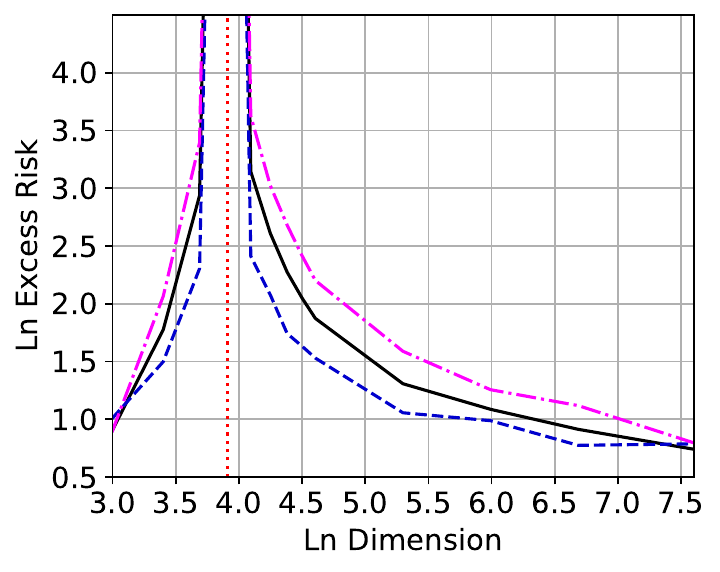}}\hfill
\subfigure[Ridgeless, $\lambda_i = i^{-2}$]{\includegraphics[width=.32\textwidth]{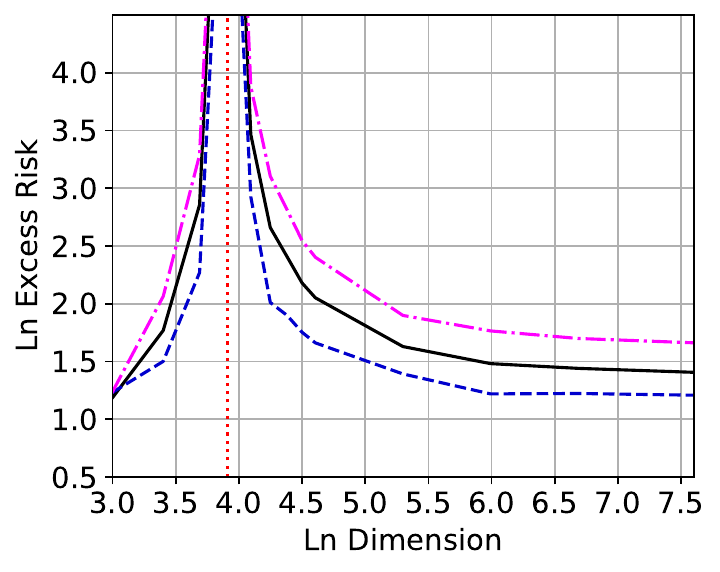}}\hfill
\subfigure[Ridgeless, $\lambda_i = i^{-\ln(i)}$]{\includegraphics[width=.32\textwidth]{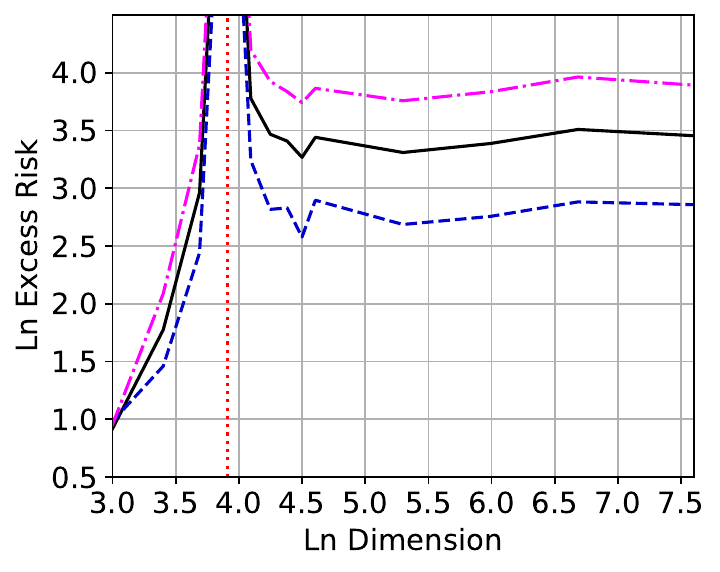}} 

\caption{Comparing covariate shift in underparameterized vs. overparameterized linear regression for three different eigendecay rates. In the $p > n$ setting: (a) leads to benign overfitting, (b) leads to tempered overfitting, and (c) leads to catastrophic overfitting. We implement simple multiplicative shifts with $\alpha, \beta$ as defined in Section \ref{sec:taxonomy_shifts} where we take $n=50, k=10$. Ground truth models are sampled uniformly from $\calS^{p-1}$ and training label noise is sampled from $\calN(0, 2)$. Every curve is averaged over 50 independent runs.}
\label{fig:under_and_over}
\end{figure}
It is clear that even though Theorem \ref{thm:benficial_malignant_shifts} is for the case in which $\covsource$ satisfies benign source conditions, the style of beneficial and malignant shift we identify holds for the MNI even when overfit in a non-benign manner.
That is, when $\covsource$ has eigendecay rates that are \textit{tempered} or \textit{catastrophic} we can still obtain non-trivial beneficial and malignant shifts by changing the energy on the signal and noise components in a heterogeneous way.

We also notice in Figure \ref{fig:under_and_over} that even when varying the dimension up to $p=2000$ at $n=50, k=10$ we do not quite observe the cross-over from beneficial to malignant shifts in the overparameterized regime.
However, we observe that in Figure \ref{fig:under_and_over}(a) that the two OOD curves begin to cross-over.
Given compute budget, we run a variant of Figure \ref{fig:under_and_over} where we extend up to $p=5000$ and take smaller $n$, e.g. $n=20, 30, 40$, in order to closer examine the different regimes of overfitting.
In addition, we experimentally show results for $p=5, 10$ which we liken to the classical linear regression regime in which $k = p < n$, meaning all of the signal is captured in the $p$ components.
In this setting, $\alpha$ shifts are all that influence the distribution shift behavior.
We show these behaviors in Figure \ref{fig:under_and_over_crossover_vary_n}.

\begin{figure}[H]
\centering
\subfigure[$n=20$]{\includegraphics[width=.32\textwidth]{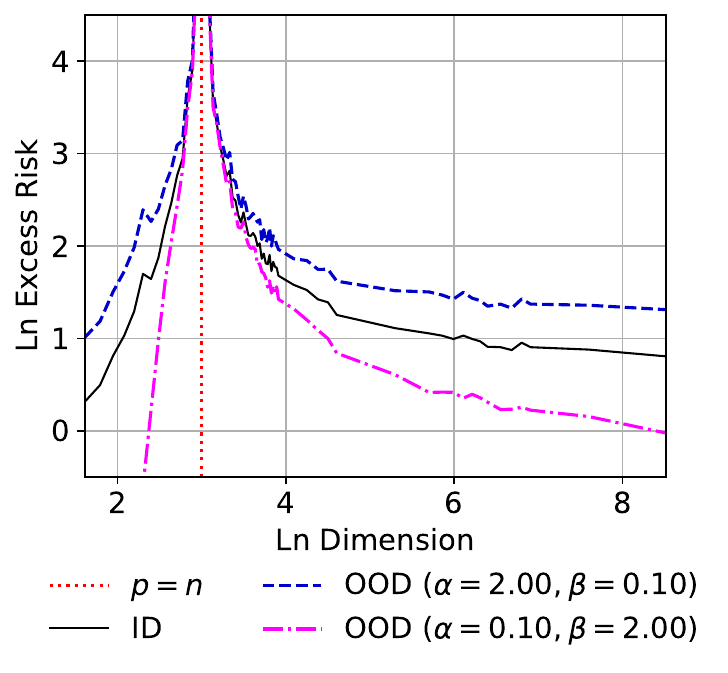}}\hfill
\subfigure[$n=30$]{\includegraphics[width=.32\textwidth]{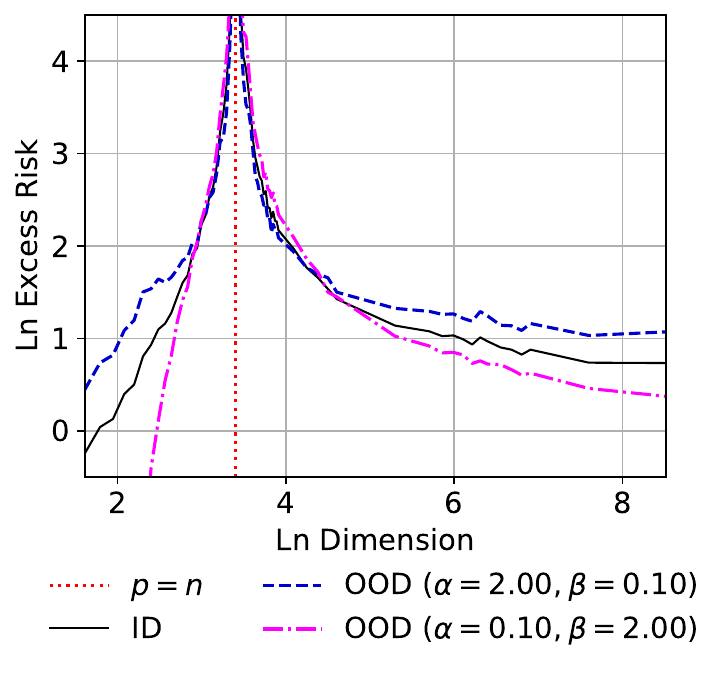}}\hfill
\subfigure[$n=40$]{\includegraphics[width=.32\textwidth]{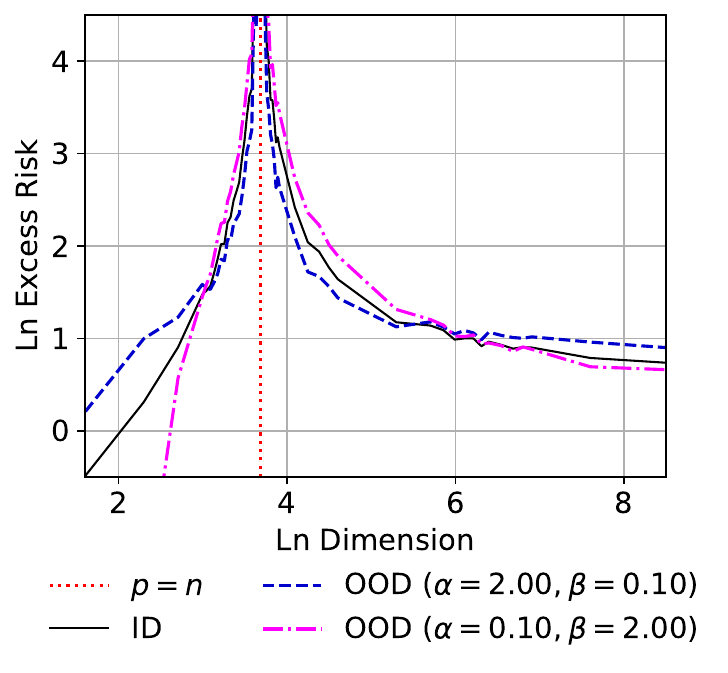}}

\caption{Comparing covariate shift in underparameterized vs. overparameterized linear regression for when $\lambda_i = i^{-1}\ln^{-2}(i+1)$. We implement simple multiplicative shifts with $\alpha, \beta$ as defined in Section \ref{sec:taxonomy_shifts} where we take $k=10$ and vary $n$. Ground truth models are sampled uniformly from $\calS^{p-1}$ and training label noise is sampled from $\calN(0, 2)$. Every curve is averaged over 100 independent runs.}
\label{fig:under_and_over_crossover_vary_n}
\end{figure}

\subsection{MLP on Synthetic Data}

We now show additional results for MLPs trained to interpolation on synthetic datasets.
This experiment is analogous to that of Figure \ref{fig:mlp_synth_regr} except that we implement shifts in the style of Theorem \ref{thm:benficial_malignant_shifts}.

\begin{figure}[H]
\centering
\includegraphics[width=0.5\linewidth]{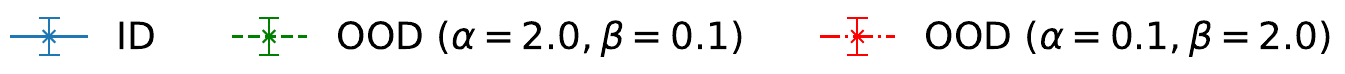}\\
\subfigure[$n=400, p=200, h=2048$]{\includegraphics[width=.28\textwidth]{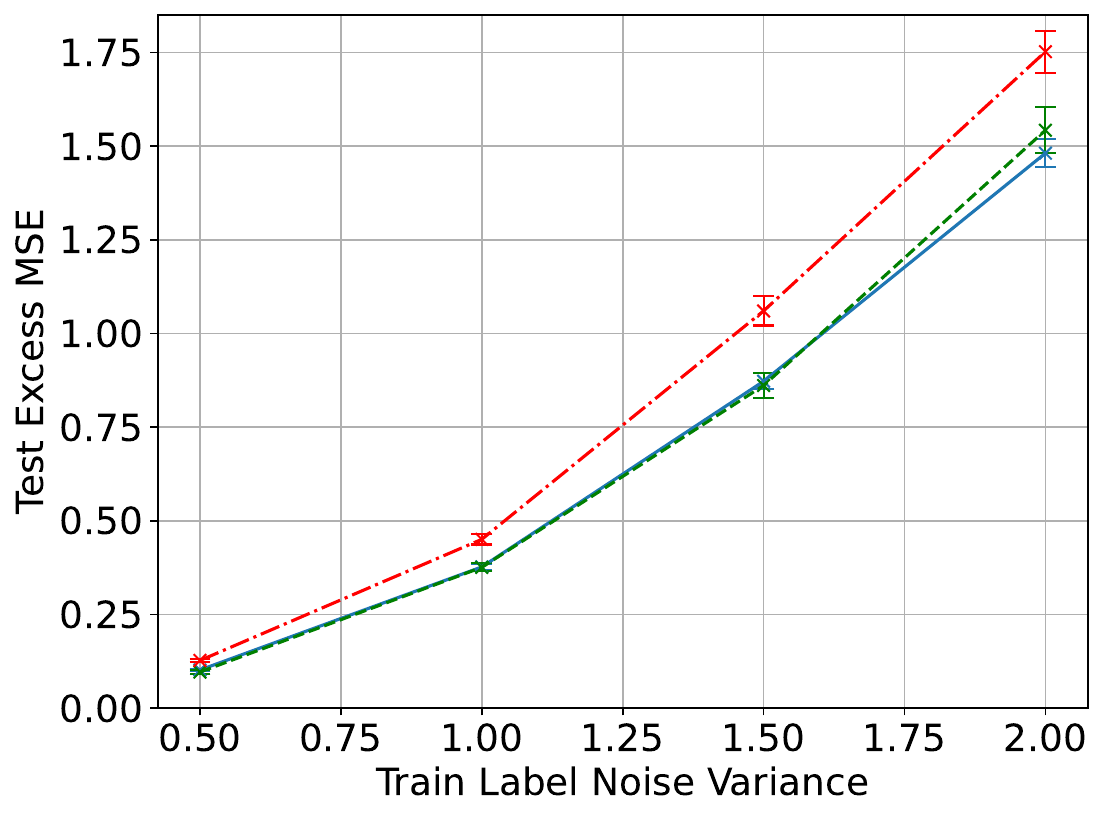}}\qquad \qquad
\subfigure[$n=400, p=4k, h=2048$]{\includegraphics[width=.28\textwidth]{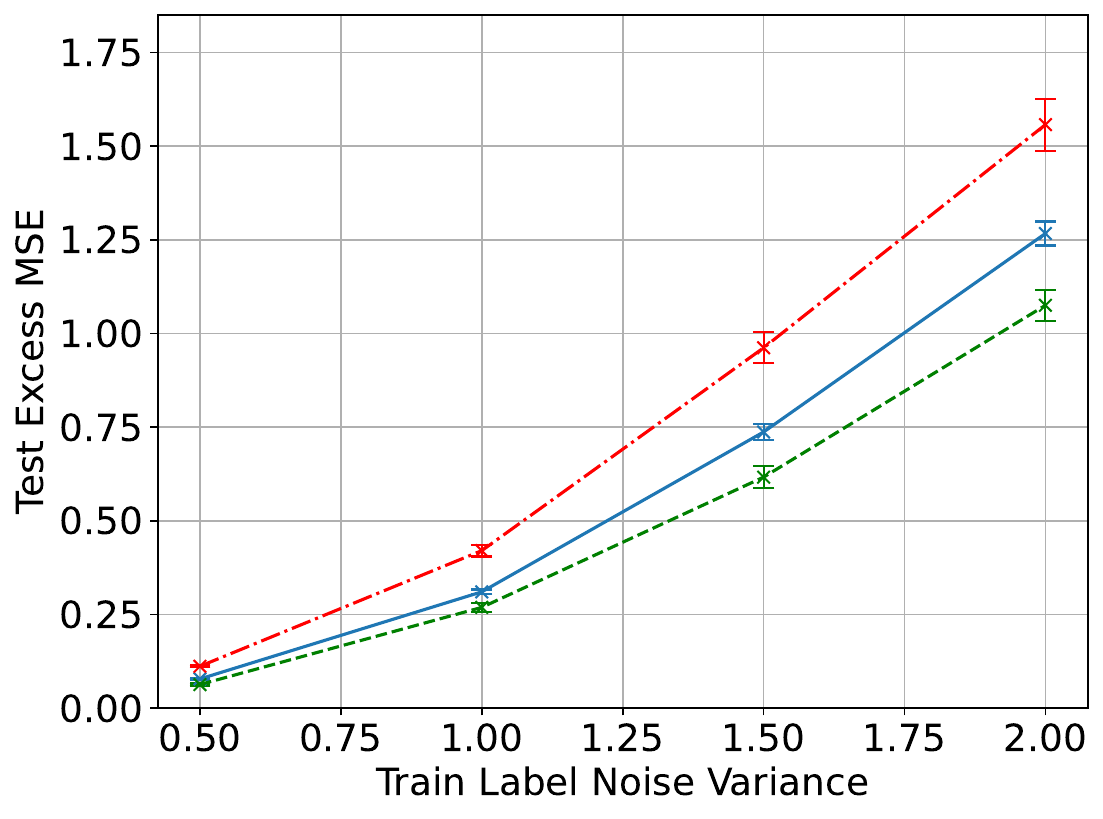}}

\caption{We implement multiplicative shifts for interpolating 3-layer ReLU MLPs in the style of Theorem \ref{thm:benficial_malignant_shifts}. Source data, $\xsource$, is sampled from a mean-centered Gaussian with diagonal covariance given by $\lambda_i = i^{-1}\ln^{-1.5}(i+1)$. Ground truth models are sampled as $\tsource \sim \calS^{p-1}$ and training label noise is samples as $\varepsilon = \calN(0, \sigma_x^2)$. Noisy training labels are obtained as $y = \xsource \tsource + \varepsilon$. The target covariances are obtained by multiplying the top $k=10$ source eigenvalues by $\alpha$ and the bottom $p-k$ source eigenvalues by $\beta$. From our theory, we expect that $\alpha=2, \beta=0.1$ leads to beneficial shifts while $\alpha=0.1, \beta=2$ leads to malignant shifts. We see this holds up when $p > n$, and that $h > n$ does not change this relationship. All curves are averaged over 20 independent runs and each training run reaches MSE loss $\leq 5e^{-6}$.}
\label{fig:mlp_log_plaw_k10_synthetic}
\end{figure}

In Figure \ref{fig:mlp_synth_regr} we sampled the ID dataset from a mean-centered Gaussian with diagonal covariance that has eigenvalues $\lsource_i = i^{-1} \ln^{-3}(i+1)$ and we examined the behavior for OOD datasets under covariate shift where the eigenvalues of the OOD covariance are given by $\lsource_i = i^{-1}\ln^{-2}(i+1)$ and $\lsource_i = i^{-1}\ln^{-4}(i+1)$.
In Figure \ref{fig:mlp_log_plaw_k10_synthetic} we take the ID data to be sampled from a mean-centered Gaussian with diagonal covariance that has eigenvalues $\lsource_i = i^{-1}\ln^{-1.5}(i+1)$.
For the covariance of the OOD datasets, we shift the top $k=10$ eigenvalues by a factor of $\alpha$ and the bottom $p-k$ eigenvalues by a factor of $\beta$, as in the setting of Theorem \ref{thm:benficial_malignant_shifts}.
We experiment here with $\alpha = 2, \beta = 0.1$ and $\alpha=0.1, \beta = 2$.
Each model achieves training MSE $\leq 5e^{-6}$.
We see the same trends as in Figure \ref{fig:mlp_synth_regr} with respect to $p > n$ versus $h > n$.
In the $p < n$ case, even though $h > n$ we do not clearly observe a beneficial shift as predicted by our high-dimensional linear theory.
However, when $p > n$ we do observe beneficial shifts for $\alpha=2, \beta=0.1$, as suggested by our theorem for the mildly overparameterized case.

\subsection{MNI on CIFAR-10C Experiments}

\subsubsection{Additional blur and noise filter experiments}

\begin{figure*}[h]
\centering
\includegraphics[width=0.8\linewidth]{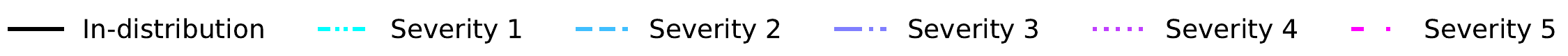}\\
\subfigure[Blur Covariance]{\includegraphics[width=.24\textwidth]{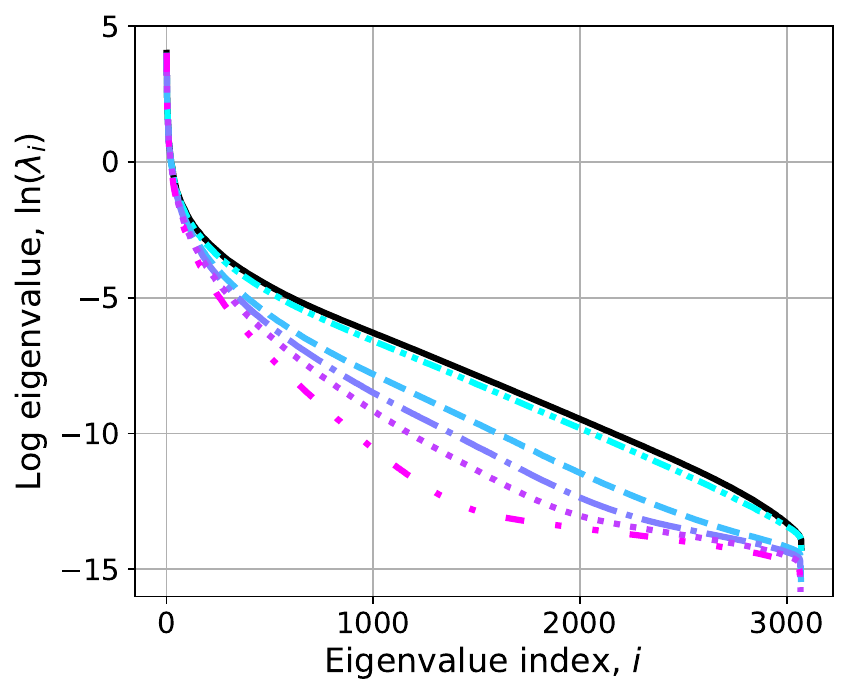}}\hfill
\subfigure[Noise Covariance]{\includegraphics[width=.24\textwidth]{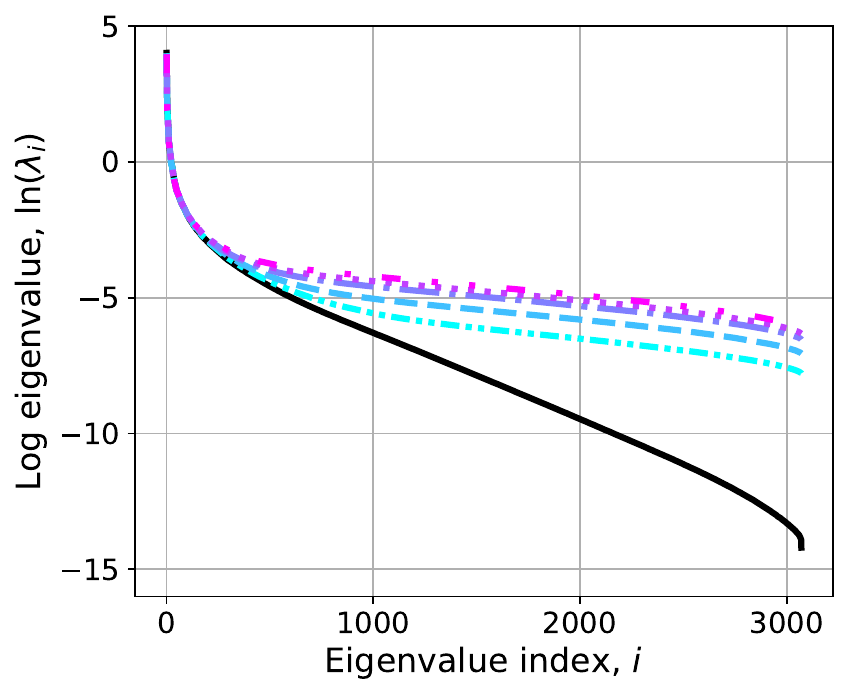}}\hfill
\subfigure[MNI tested on Blur]{\includegraphics[width=.24\textwidth]{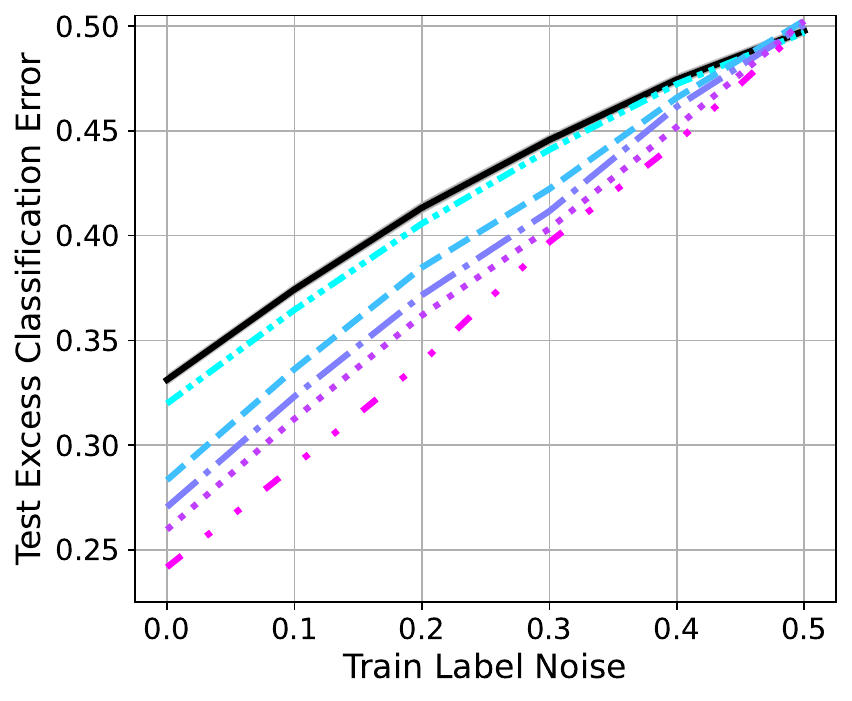}}\hfill
\subfigure[MNI tested on Noise]{\includegraphics[width=.24\textwidth]{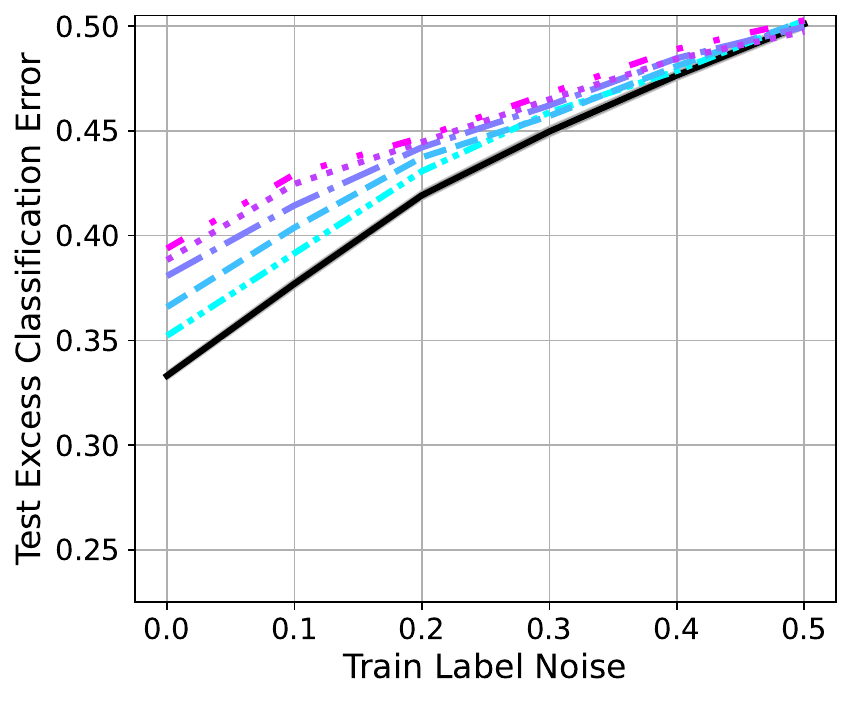}}
\caption{We fit the MNI to binary CIFAR-10 (dog vs. truck) and test on binary CIFAR-10C under Gaussian blur and noise corruptions. In (a), (b) we plot the eigenvalues of the covariance matrices for ID test data and on test sets for each severity. To ensure $p > n$ we subsample the training set to $n=1k$ and average curves over 50 independent runs. We evaluate the MNI against all 5 corruption severities and plot excess classification error vs training label noise, which is class label flip probability. We see that the eigenspectra of the OOD datasets is directly correlated to the OOD performance of the MNI.}
\label{fig:cf10c_mni}
\end{figure*}

\begin{figure*}[h]
    \centering
    \includegraphics[width=0.8\linewidth]{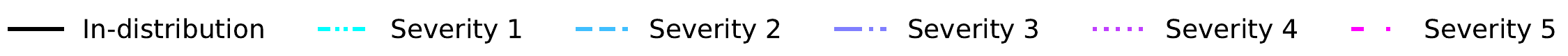}\\
    \subfigure[Blur and Noise Covariance]{\includegraphics[width=0.24\linewidth]{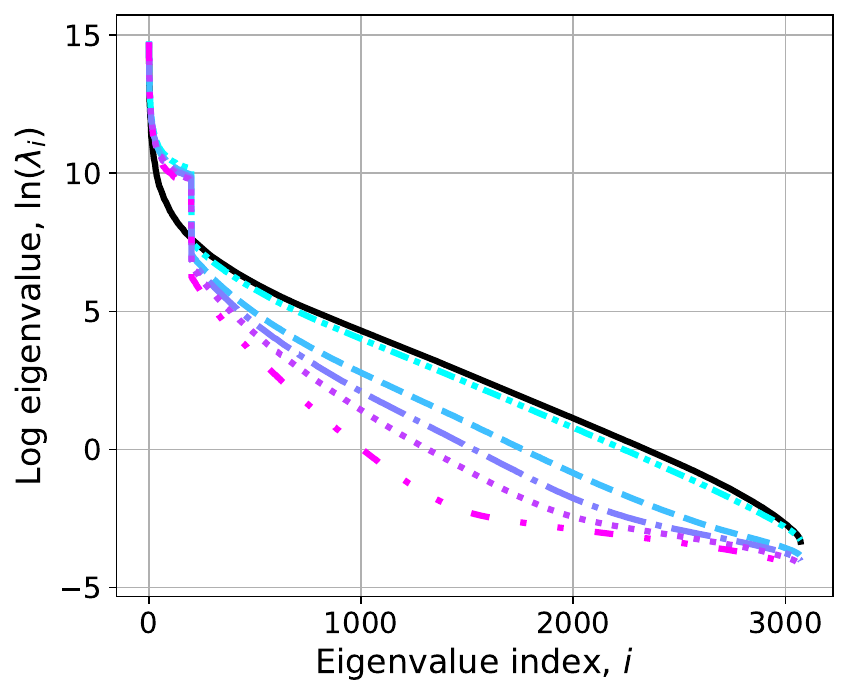}}\qquad
    \subfigure[$n = 500$]{\includegraphics[width=0.24\linewidth]{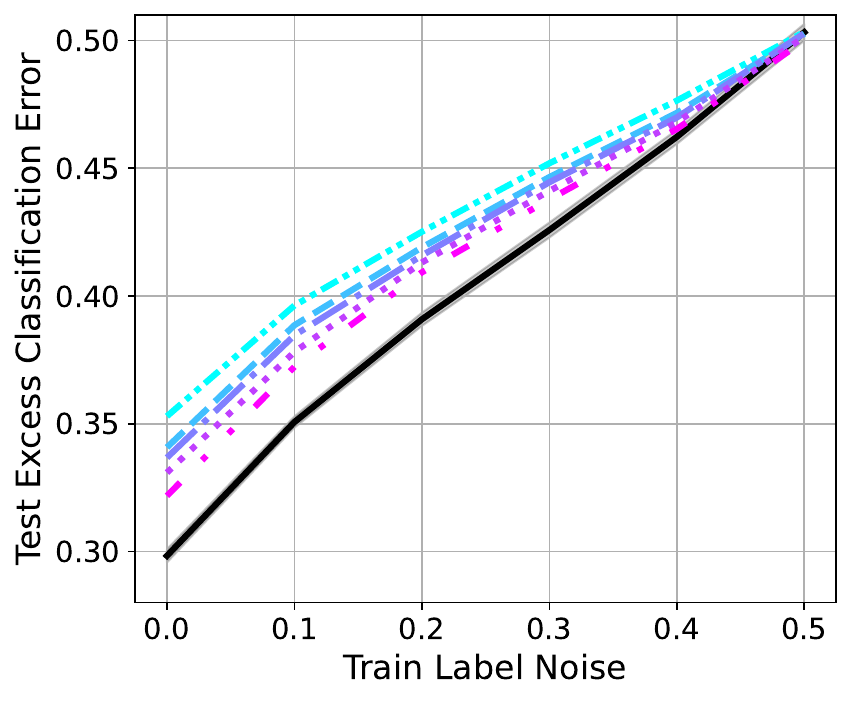}}\qquad
     \subfigure[$n=2000$]{\includegraphics[width=0.24\linewidth]{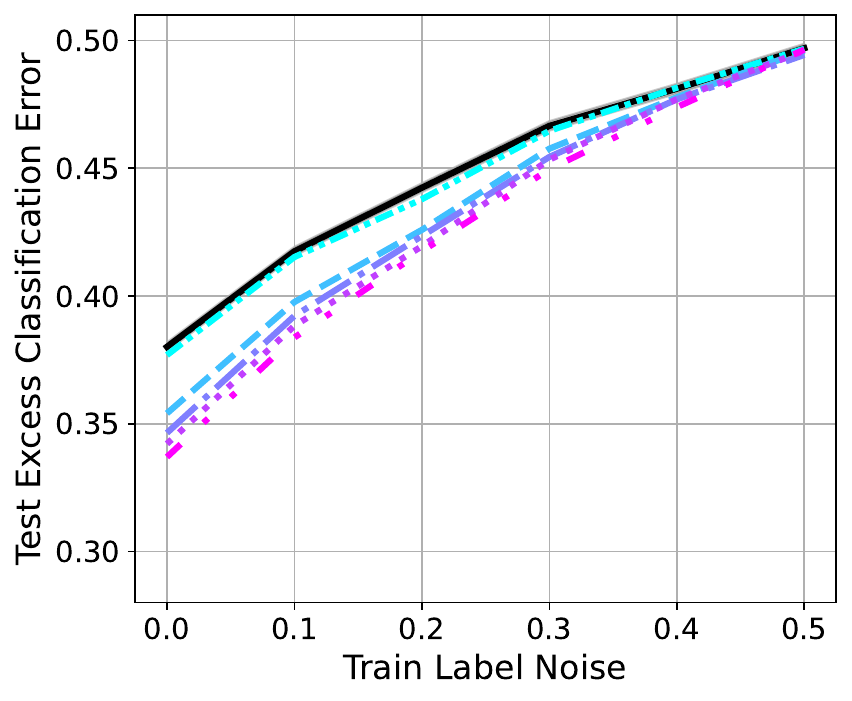}}
    \caption{We consider an experiment using a custom variant of the CIFAR-10C out-of-distribution (OOD) test sets while continuing to train on the original CIFAR-10 dataset with $n$ training samples at varying amounts of training label noise.
    Our constructions injects Gaussian noise at varying severity levels into the top 200 high-variance directions of the Gaussian blur test sets at each severity level.
    In (a) we plot the log of the spectrum of the covariance matrices of each test set.
    This results in a covariance spectrum in which the top eigenvalues of the OOD data are larger than the top eigenvalues of the in-distribution (ID) eigenvalues, and the bottom eigenvalues of the OOD data are smaller than the bottom eigenvalues of the ID data.
    This corresponds to the $\alpha > 1, \beta < 1$ setting in our taxonomy.
    In (b) and (c) we plot test excess classification error vs. train label noise.
    In (b) we show the \textit{severely overparameterized} setting which results in malignant shifts, and in (c) we show the \textit{mildly overparameterized} setting which results in benficial shifts, in keeping with the intuitions from our taxonomy.
    Furthermore, the trace of the OOD covariances are larger than the ID covariance and yet in (c) we observe improved OOD performance, in contrast to intuitions from prior work.
    }
    \label{fig:cf10c_artificial_blur_and_noise}
\end{figure*}

\subsubsection{Varying levels of overparameterization}

\begin{figure}[H]
\centering
\subfigure[$p>n$, In-distribution]{\includegraphics[width=.27\textwidth]{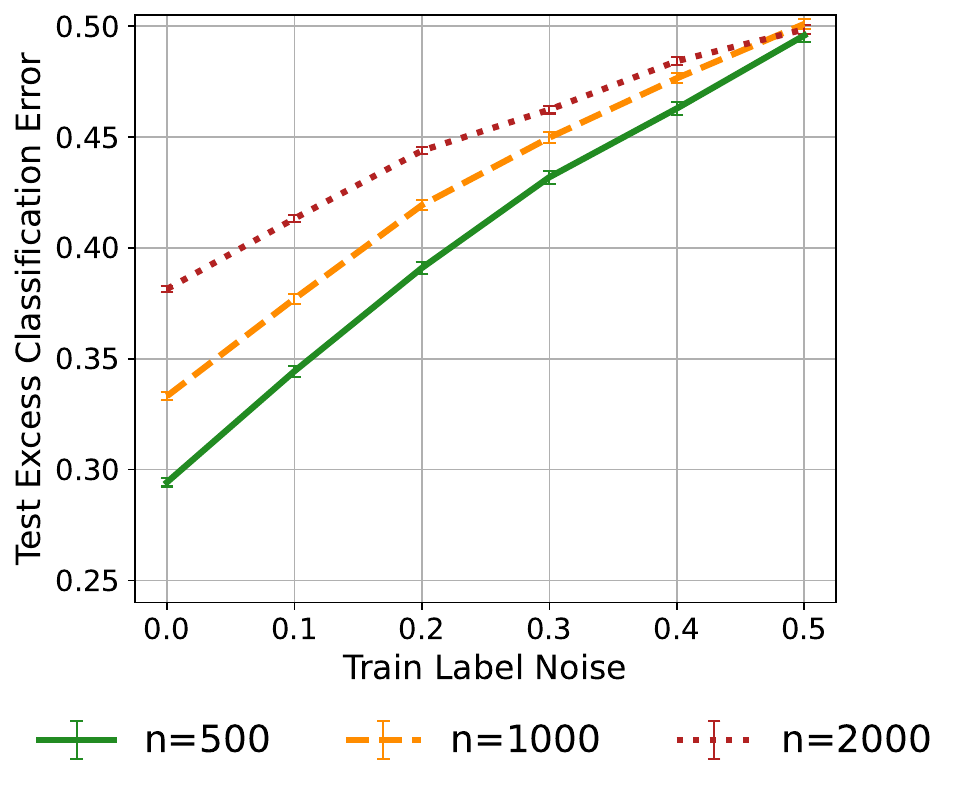}}\qquad
\subfigure[$p>n$, Blur]{\includegraphics[width=.27\textwidth]{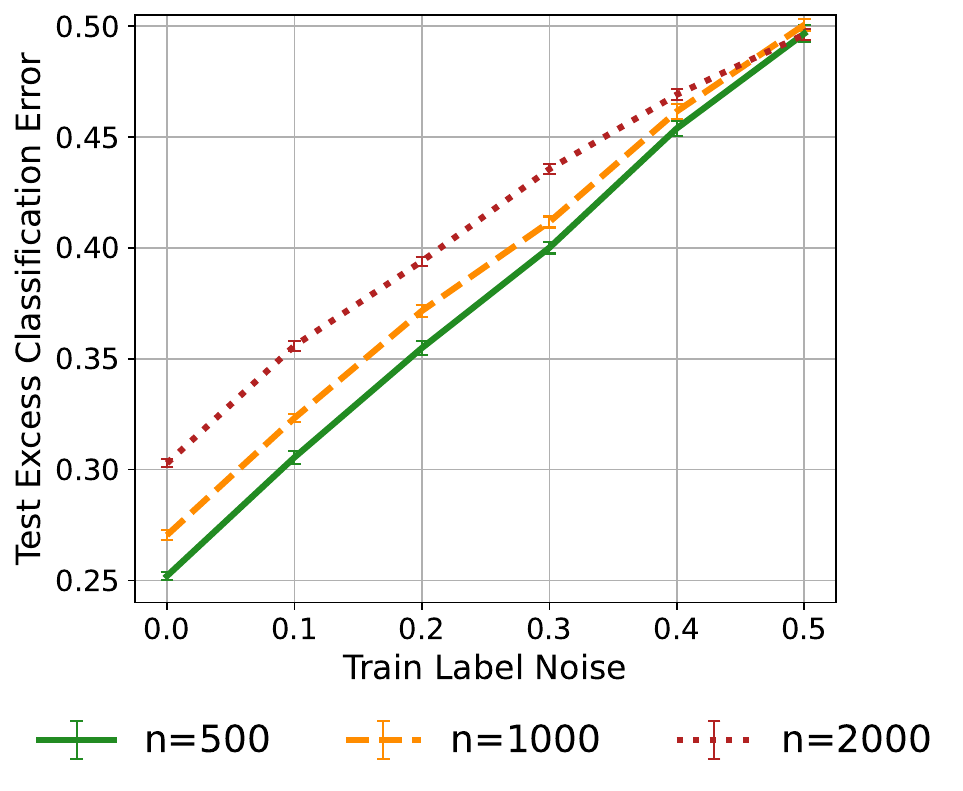}}\qquad
\subfigure[$p>n$, Noise]{\includegraphics[width=.27\textwidth]{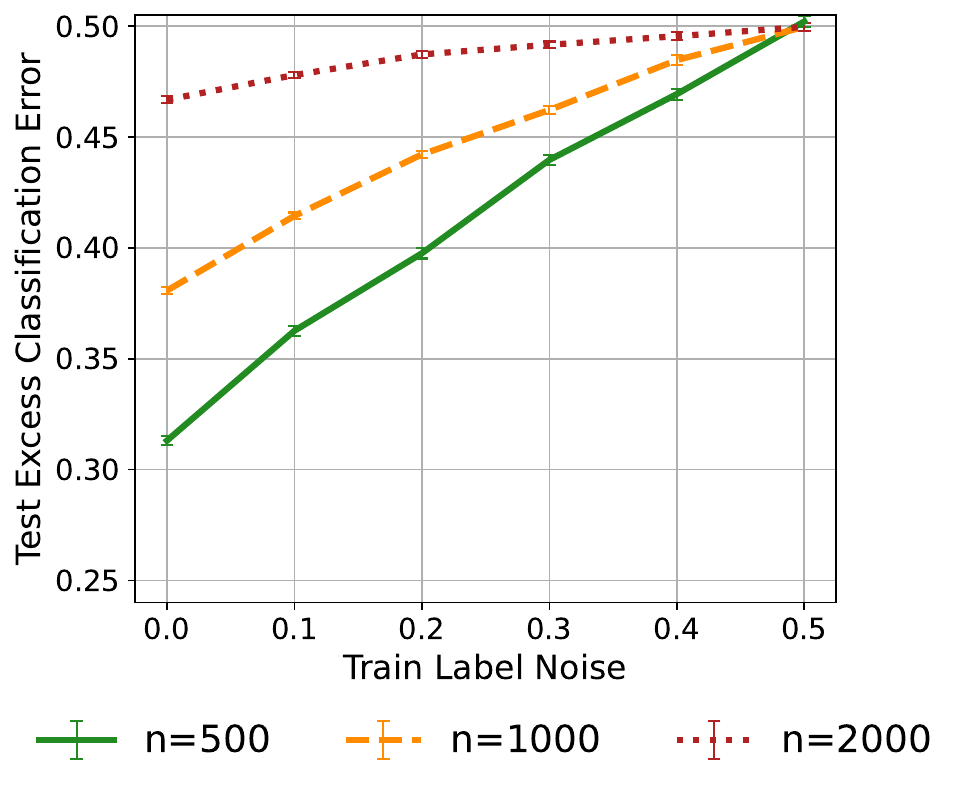}}\\
\subfigure[$p<n$, In-distribution]{\includegraphics[width=.27\textwidth]{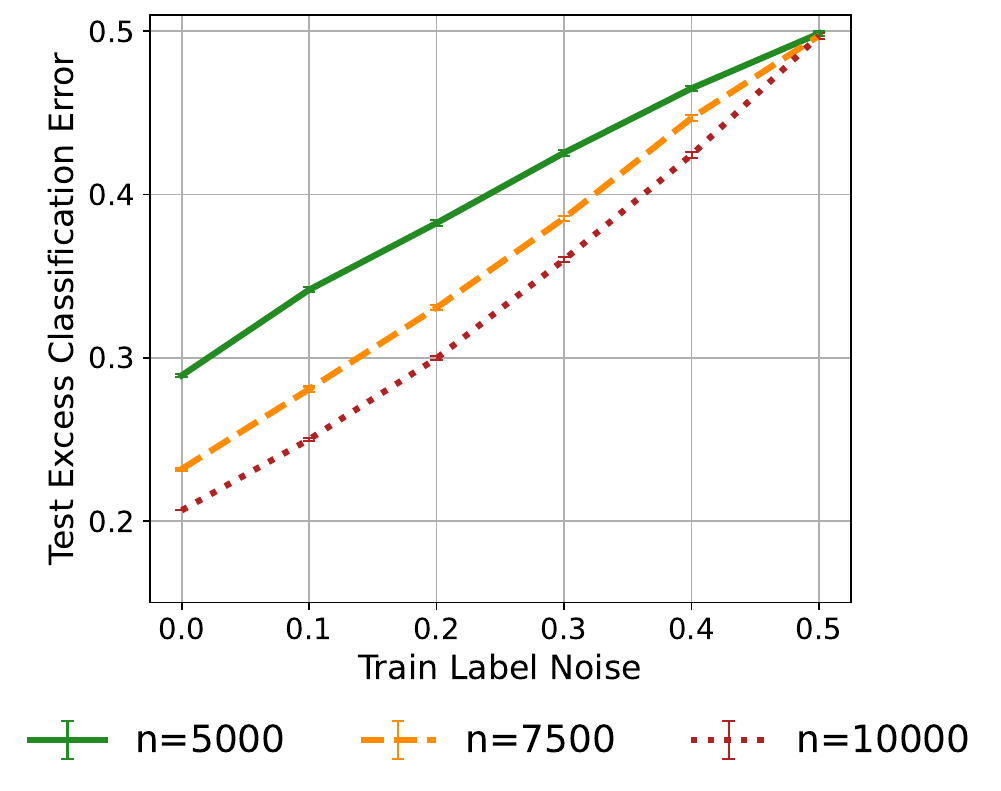}}\qquad
\subfigure[$p<n$, Blur]{\includegraphics[width=.27\textwidth]{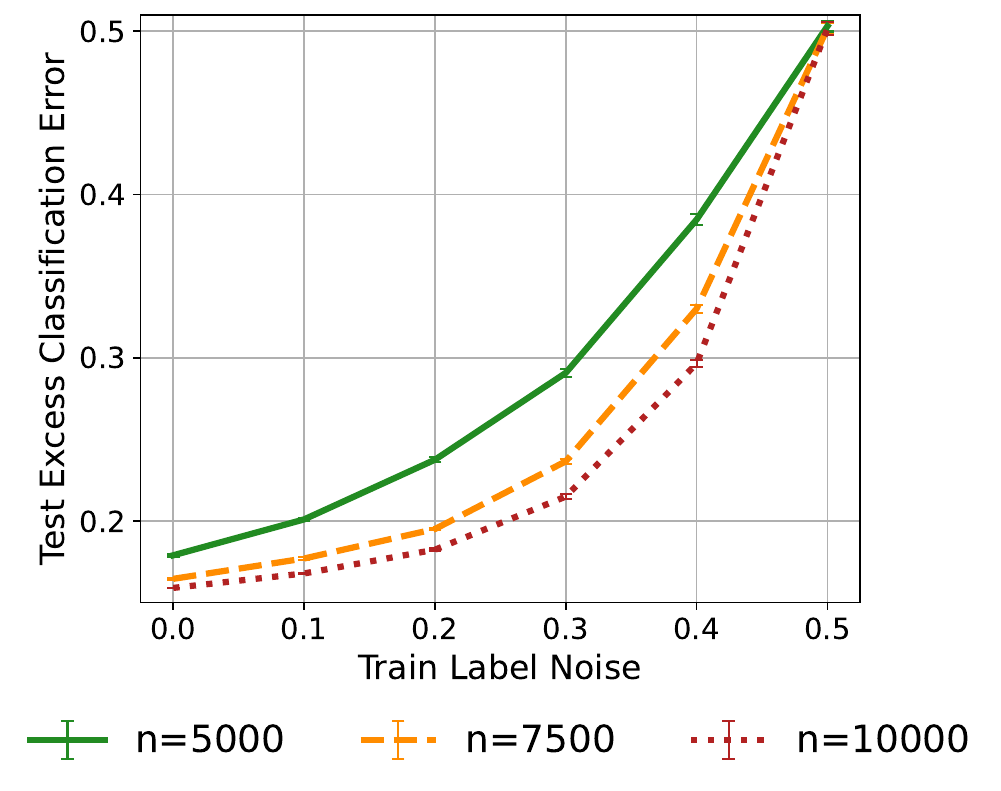}}\qquad
\subfigure[$p<n$, Noise]{\includegraphics[width=.27\textwidth]{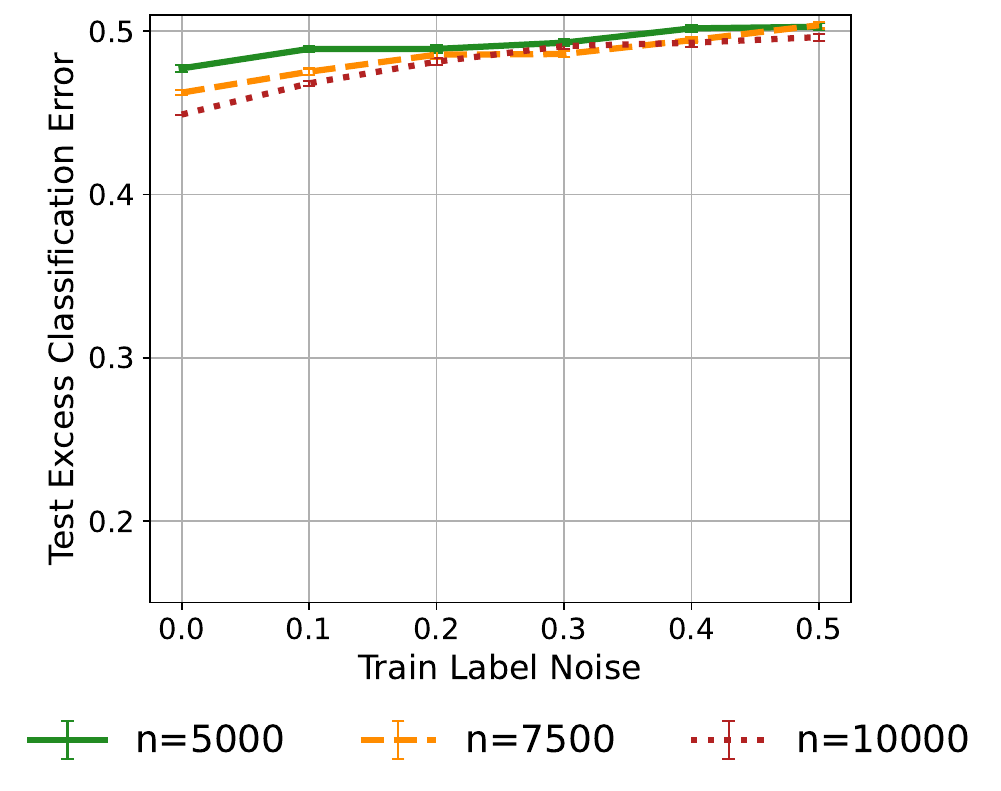}}
\caption{We fit the ridgeless OLS solution to binary CIFAR-10 (dog vs. truck) and test on binary CIFAR-10C under Gaussian blur and noise corruptions. In the top row, we vary the level of overparameterization as $n=500, 1k, 2k$ and average each curve over 50 independent runs. In this $p>n$ setting the ridgelss OLS solution results in the MNI. In the bottom row we obtain a non-interpolating, ridgeless linear solution. Evaluations are done on severity 3 of CIFAR-10C, however the results hold up across all severities.
We plot excess classification error vs training label noise, which is class label flip probability. We see that overparameterization improves robustness of the MNI at all noise levels.}
\label{fig:cf10c_mni_overparam}
\end{figure}

In Figure \ref{fig:cf10c_mni_overparam} (a-c) we show that overparameterization improves OOD excess classification error for the MNI fit to binary CIFAR-10 and evaluated on binary CIFAR-10C under Gaussian blur and noise corruptions.
The details of these datasets and setups are given in Appendix \ref{apdx:experiment_details}.
We note that all of the curves in the top row of this figure are in the overparameterized regime, meaning they are on the right side of the double descent curve.
Flattened CIFAR images have $p=3,072$ and so we vary the number of training subsample sizes over $n=500, 1000, 2000$ in order to remain in an overparameterized setting.
We find that when we are overparameterized, as we reduce $n$ we obtain improved performance. 
We average over 50 independent runs in each setting and provide standard error bars to show that this observation is not due to specific random samples.
We also see that at higher levels of overparameterization, the relative difference in excess classification error between ID, blur, and noise test sets lessens.
For example, at $0.0$ label noise and $n=2k$ the average excess error varies from $0.3028$ on the blur set to $0.4668$ on the noise set for an absolute difference of $0.164$, whereas at $n=500$ the average excess error only varies from $0.252$ on the blur set to $0.3131$ on the noise set for an absolute difference of $0.0611$.

For completeness, in Figure \ref{fig:cf10c_mni_overparam} (d - e) we show the above setting in the  underparameterized regime where we obtain the linear solution via the ridgeless OLS solution.
As these plots are on the left side of the double descent peak, we see that adding more data improves OOD excess classification error.
While these models are not interpolating, we observe that noise corruptions lead to nearly \textit{catastrophic} performance, meaning random guessing, on the OOD test sets, whereas blur corruptions lead to more \textit{benign} performance.
Finally the ID performance appears to be \textit{tempered}, in showing a nearly linear relationship between train label noise and test excess classification error.

\subsection{ResNet on CIFAR-10C Experiments}
\begin{figure}[H]
\centering
\includegraphics[width=0.8\linewidth]{figs/CIFAR_10_C/replotted_jan9/cf10c_legend.pdf}\\
\subfigure[Interpolating ResNet tested on Blur]{\includegraphics[width=.3\textwidth]{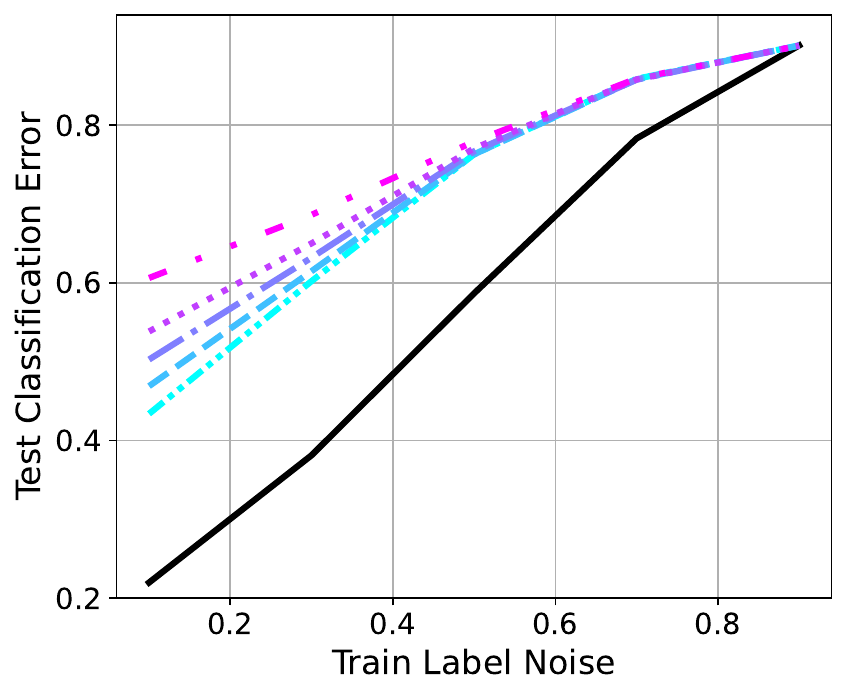}}\qquad \qquad \qquad
\subfigure[Interpolating ResNet tested on Noise]{\includegraphics[width=.3\textwidth]{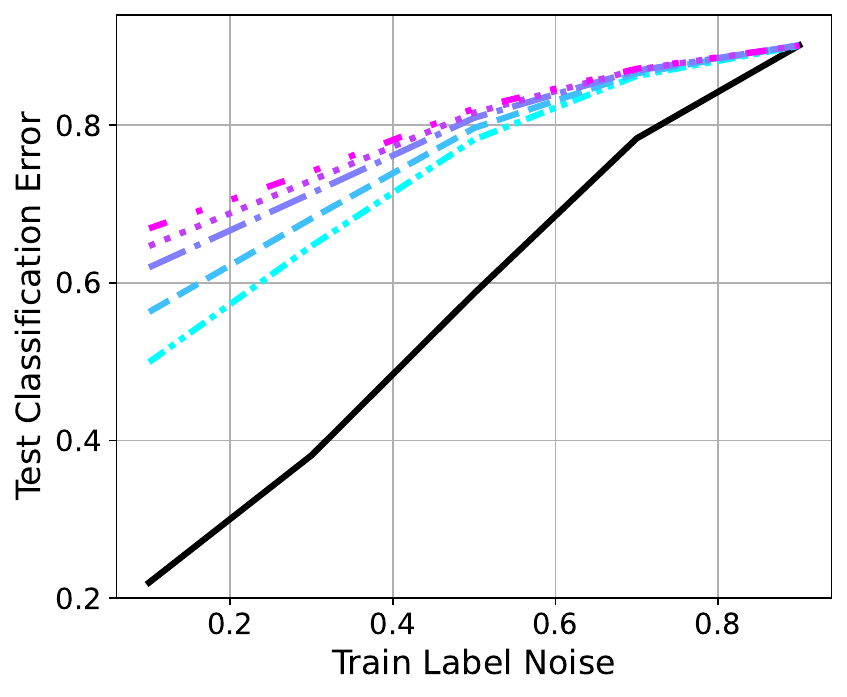}}
\caption{We train ResNet18 on clean CIFAR-10 and evaluate on test sets that has been corrupted by Gaussian blur and Gaussian noise, which correspond to beneficial and malignant shifts, respectively. Labels are flipped with probability 0.1 through 0.9, seen on the x-axis. The setting is not high-dimensional because the training data contains 50000 images, each of which are 3072-dimensional. ResNet18 contains around 11.7 million parameters, so the setting is very overparameterized. We observe that while both shifts negatively affect generalization, the beneficial shift isn't as bad as the malignant shift. This result is similar to those seen in subfigures (a) and (b) in Figure \ref{fig:mlp_synth_regr}, where the data is not high-dimensional but the MLP is overparameterized.}
\label{fig:cf10_resnets}
\end{figure}

Figure \ref{fig:cf10_resnets} shows the behavior of interpolating ResNets trained on the full CIFAR-10 dataset and evaluated on CIFAR-10C blur and noise corruptions.
While these numbers are suboptimal with respect to CNNs on CIFAR-10 we note that they are justified in our setting as our goal is to study interpolating models.
At 90\% label noise it takes a lot of compute to interpolate the entire CIFAR-10 dataset, especially if using data augmentations, weight decay, or other regularizations.
As such, we turn off weight decay and data augmentations for these models to be able to tractably interpolate CIFAR-10 at high noise levels. 

\end{document}